\documentclass{article}

\usepackage{microtype}
\usepackage{graphicx}
\usepackage{subcaption}
\usepackage{booktabs} 

\usepackage{hyperref}
\usepackage{xcolor}
\usepackage{etoc}

\usepackage[most]{tcolorbox}
\usepackage{listings}
\usepackage{xcolor}

\definecolor{EBFT}{HTML}{1B9E77}   
\definecolor{SFT}{HTML}{7570B3}    
\definecolor{RLVR}{HTML}{D95F02}   
\definecolor{Base}{HTML}{3D3D3D}   
\definecolor{CardRule}{HTML}{222222}
\definecolor{SpecBg}{HTML}{F3F4F6}
\definecolor{FailBg}{HTML}{FFF1F2}
\definecolor{FailRule}{HTML}{E11D48}

\newcommand{\methodtag}[2]{%
  \textbf{\textsc{\textcolor{#1}{#2}}}%
}

\newcommand{\failuretag}[1]{%
  \vspace{1pt}\par
  \begin{tcolorbox}[
    enhanced,
    colback=FailBg,
    colframe=FailRule,
    boxrule=0.4pt,
    arc=1.2pt,
    left=4pt,right=4pt,top=2pt,bottom=2pt,
    boxsep=1pt
  ]
  {\footnotesize 
  \textbf{\textcolor{FailRule}{Failure mode:}}~#1
  }
  \end{tcolorbox}%
}

\newcommand{\spechead}[1]{%
  \begin{tcolorbox}[
    enhanced,
    colback=SpecBg,
    colframe=CardRule,
    boxrule=0.35pt,
    arc=1.2pt,
    left=4pt,right=4pt,
    top=1pt,bottom=1pt,  
    boxsep=0.5pt,        
  ]
  \centering\itshape~#1
  \end{tcolorbox}%
}

\tcbset{
  qualcard/.style={
    enhanced,
    colback=white,
    colframe=black,
    boxrule=0.5pt,
    arc=2pt,

    left=5pt,right=5pt,top=2pt,bottom=4pt,
    boxsep=1pt,

    fonttitle=\bfseries,
    colbacktitle=white,
    coltitle=black,

    toptitle=1pt,
    bottomtitle=1pt,
    lefttitle=0pt,
    righttitle=0pt,
    titlerule=0pt,

    before skip=6pt,
    after skip=0pt
  }
}

\NewDocumentCommand{\samy}{O{0.8\linewidth} +m}{%
  {\color{red}[Samy:\ \parbox[t]{#1}{#2}]}%
}

\usepackage[preprint]{icml2026}

\usepackage{amsmath}
\usepackage{amssymb}
\usepackage{mathtools}
\usepackage{amsthm}

\usepackage{algorithm}
\usepackage{algorithmic}

\usepackage{enumitem}

\usepackage{accents}

\usepackage[capitalize,noabbrev]{cleveref}

\theoremstyle{plain}
\newtheorem{theorem}{Theorem}[section]

\newtheorem{lemma}[theorem]{Lemma}

\theoremstyle{definition}

\theoremstyle{remark}

\usepackage{xcolor}

\usepackage[disable,textsize=tiny]{todonotes}

\icmltitlerunning{Energy-Based Fine-Tuning}

\begin{document}

\twocolumn[
  \icmltitle{Matching Features, Not Tokens: Energy-Based Fine-Tuning of Language Models}



  \icmlsetsymbol{equal}{\textbf{*}}
  
  \begin{icmlauthorlist}
    \icmlauthor{Samy Jelassi}{equal,harvard}
    \icmlauthor{Mujin Kwun}{equal,harvard}
    \icmlauthor{Rosie Zhao}{equal,harvard}
    \icmlauthor{Yuanzhi Li}{tbd} 
    \icmlauthor{Nicolo Fusi}{msr}
    \icmlauthor{Yilun Du}{harvard}
    \icmlauthor{Sham M. Kakade}{harvard}
    \icmlauthor{Carles Domingo-Enrich}{equal,msr}
  \end{icmlauthorlist}

  \icmlaffiliation{harvard}{Harvard University, Kempner Institute}
  \icmlaffiliation{tbd}{MBZUAI}
  \icmlaffiliation{msr}{Microsoft Research New England}
  \icmlcorrespondingauthor{Samy Jelassi}{jelassisamy@gmail.com}
  \icmlcorrespondingauthor{Mujin Kwun}{mujin\_kwun@harvard.edu}
  \icmlcorrespondingauthor{Rosie Zhao}{rosiezhao@g.harvard.edu}
  \icmlcorrespondingauthor{Carles Domingo-Enrich}{carlesd@microsoft.com}

  \icmlkeywords{language models, reinforcement learning, energy-based models, sequence-level learning}

  \vskip 0.3in
]



\printAffiliationsAndNotice{\icmlEqualContribution}

\begin{abstract}
Cross-entropy (CE) training provides dense and scalable supervision for language models, but it optimizes next-token prediction under teacher forcing rather than sequence-level behavior under model rollouts.
We introduce a feature-matching objective for language-model fine-tuning that targets sequence-level statistics of the completion distribution, providing dense semantic feedback without requiring a task-specific verifier or preference model.
To optimize this objective efficiently, we propose \emph{energy-based fine-tuning} (EBFT), which uses strided block-parallel sampling to generate multiple rollouts from nested prefixes concurrently, batches feature extraction over these rollouts, and uses the resulting embeddings to perform an on-policy policy-gradient update. 
We present a theoretical perspective connecting EBFT to KL-regularized feature-matching and energy-based modeling.
Empirically, across Q\&A coding, unstructured coding, and translation, EBFT matches RLVR and outperforms SFT on downstream accuracy while achieving a lower validation cross-entropy than both methods\footnote{Our code is available at \url{https://github.com/sjelassi/ebft_openrlhf} and our project page at \url{https://energy-based-fine-tuning.github.io}.}.
\end{abstract}

\section{Introduction}

Cross-entropy (CE) training under teacher forcing is the standard approach for pre-training, continued/mid-training, and supervised fine-tuning (SFT) of large language models.
Its next-token objective provides an extremely \emph{dense} learning signal, is stable under massively parallel optimization, and admits efficient implementations at scale \cite{kaplan2020scaling,brown2020language}.
However, the same teacher-forcing setup introduces a distribution shift: during training, the model conditions on ground-truth prefixes, while at deployment time, it must condition on \emph{its own} generations. Errors early in a generated sequence alter the conditioning context for subsequent predictions, causing later tokens to be sampled from distributions the model was rarely trained on \cite{bengio2015scheduledsampling,lamb2016professorforcing}.

\citet{braverman2019calibration} quantify this distribution shift by measuring the expected conditional entropy of the $k$-th generated token as a function of $k$. For a perfect model, this quantity should remain stable as $k$ grows, since the generated prefixes would be distributionally identical to ground-truth text. In practice, however, the expected entropy increases with $k$, even for models that achieve low training perplexity. This reveals a fundamental limitation of token-level supervision: low perplexity measures one-step prediction accuracy on ground-truth prefixes but does not guarantee well-calibrated behavior over longer generations. The model may match the data distribution locally while diverging from it at the sequence level.

\begin{figure}[h]
    \centering
    \includegraphics[width=0.99\linewidth]{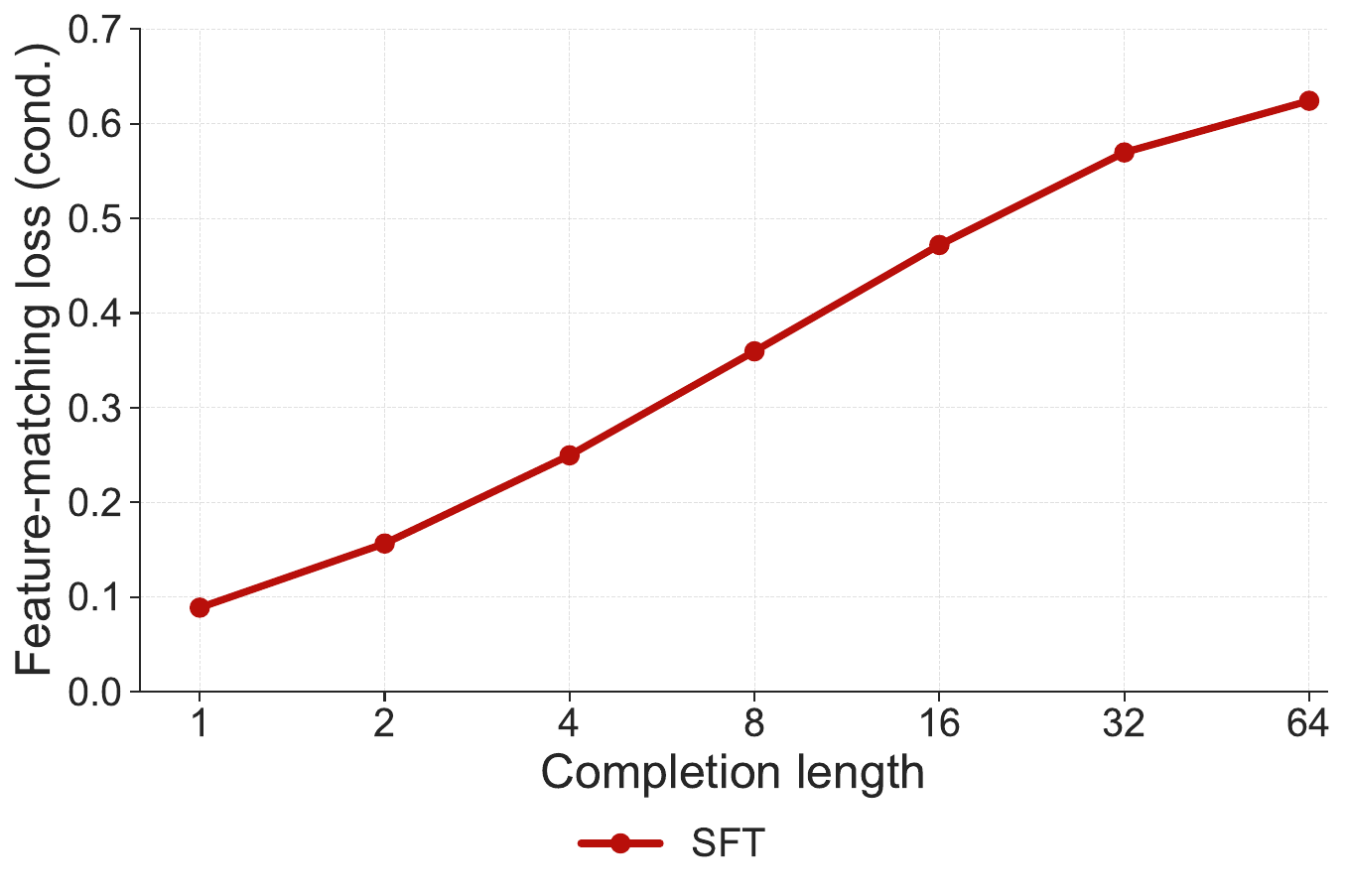}
    \caption{\textbf{Feature-matching loss grows with completion length.} Conditional feature-matching loss (lower is better) as a function of completion length for Qwen2.5-1.5B fine-tuned with SFT on OpenCodeInstruct~\citep{ahmad2025opencodeinstruct}. Although this increase is expected even under a perfect model due to growing feature variance, part of the degradation reflects SFT's inability to calibrate the model's rollout distribution over long horizons.}
    \label{fig:gen_len_sft_fig}
\end{figure}

A natural way to measure this sequence-level divergence is to 
compare statistics of model-generated completions against those of 
ground-truth completions in a feature space. We formalize this 
through a \emph{conditional feature-matching loss} that, for each 
context, measures the squared error between the mean feature 
embedding of model rollouts and that of the ground-truth 
completion (formal definition in \Cref{subsec:fm_loss}). We say a 
model is \emph{(feature-) calibrated} when this loss is zero, meaning 
its expected feature embeddings match those of the data for all 
contexts. \Cref{fig:gen_len_sft_fig} plots this loss as a 
function of completion length for a Qwen2.5-1.5B model fine-tuned 
with SFT. The loss increases with completion length, which is 
partly expected even for a perfect model due to growing feature 
variance over longer generations. However, part of this 
degradation reflects SFT's failure to calibrate the model's 
rollout distribution. A fine-tuning method that directly targets 
this feature-matching loss could reduce this gap.

RL fine-tuning~\cite{ouyang2022training,schulman2017ppo} addresses this mismatch by optimizing sequence-level rewards under the model's own rollouts, enabling direct behavioral control. However, its effectiveness depends on access to a reliable reward function or verifier, as in reinforcement learning with verifiable rewards (RLVR;~\cite{lightman2023lets,shao2024deepseekmath}). These reward signals may be unavailable, noisy, or misaligned with desired behavior in many open-ended tasks. Even when a reliable reward exists, RL optimizes a scalar signal and does not directly target distributional calibration of the rollout distribution. In our experiments, we observe this tradeoff concretely: RLVR improves downstream performance at the cost of worsening both the validation cross-entropy and the feature-matching loss introduced above.

 Beyond RLVR, a related class of methods incorporates \emph{partial rollouts} at training time~\cite{zelikman2024quietstarlanguagemodelsteach, hatamizadeh2025rlp, dong2025rpt}. Since these rollouts are typically too short or incomplete to be scored by a verifier, these methods introduce surrogate rewards -- commonly the model's own log-probabilities or token-overlap similarity between a sampled continuation and a reference. While useful in practice, neither type of surrogate provides calibration guarantees: self-likelihood rewards reinforce already high-probability samples without necessarily improving coverage, and similarity-based measures can improve the chosen metric without improving likelihood or calibration.

We propose an alternative that replaces these heuristic surrogate rewards with a principled objective: we directly optimize the feature-matching loss, using it not just as a diagnostic but as the training signal itself. A frozen feature network $\phi$, initialized from the pre-trained model, embeds concatenated prompt--completion sequences, and the generator $p_{\theta}$ is fine-tuned to match the resulting feature moments using a REINFORCE-style gradient estimator on partial rollouts (see \Cref{fig:ebm}). The resulting training signal is dense, operates at the sequence level, requires no task-specific reward or verifier, and — unlike the surrogate-reward methods above — optimizes a proper scoring rule under sufficiently rich features. We call this approach \emph{Energy-Based Fine-Tuning} (EBFT): under a KL-regularized view, the feature-matching objective implicitly defines an energy function over sequences, with the optimal policy taking the form of an exponential tilt of the base model (\Cref{sec:KL_regularization}).

\begin{figure}[t]
    \centering
    \includegraphics[width=0.99\linewidth]{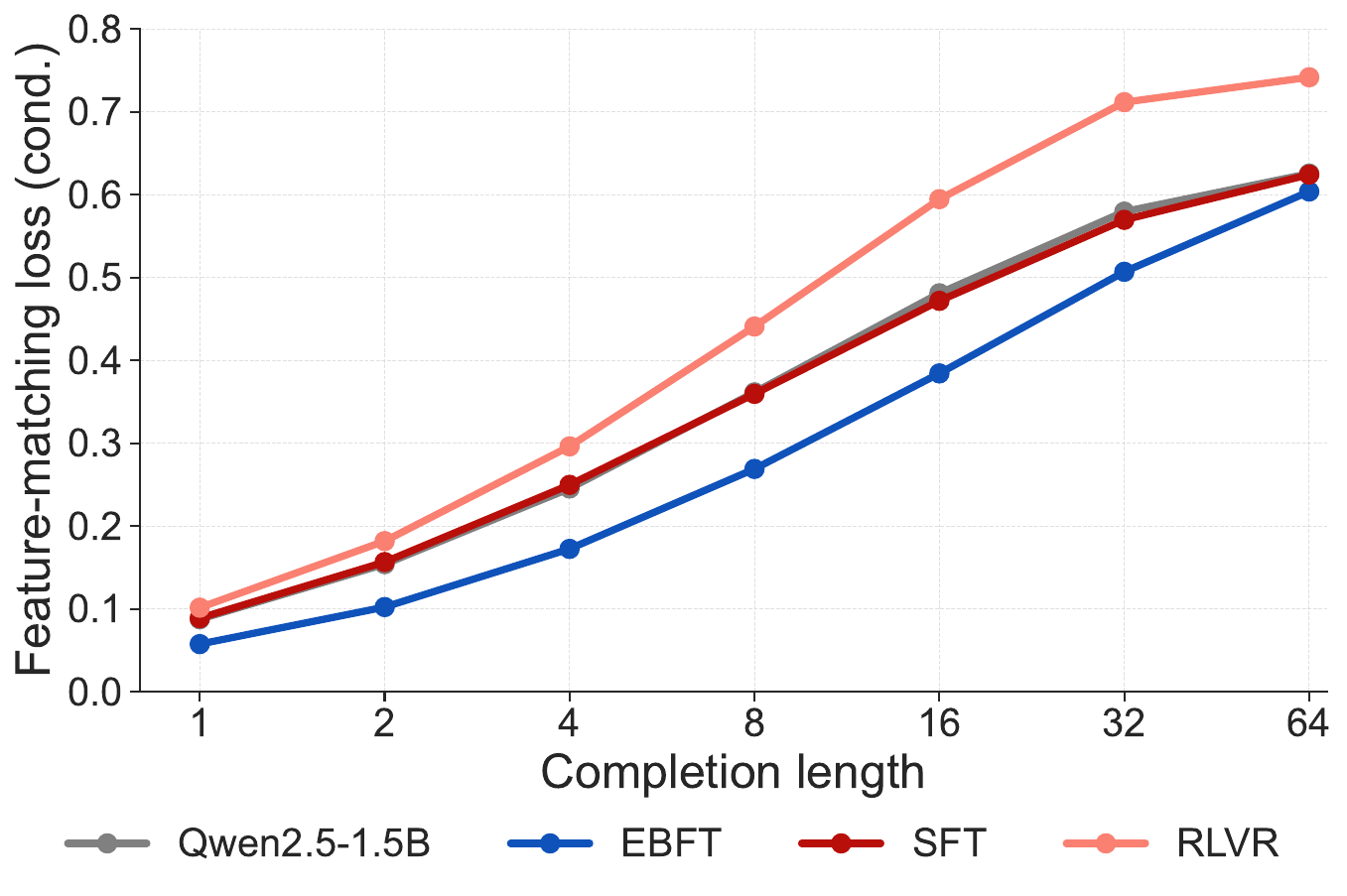}
    \caption{\textbf{EBFT achieves the lowest feature-matching loss across all completion lengths.}  Despite training with a rollout horizon of only 8 tokens, EBFT's gains persist and grow at longer completions. RLVR worsens this loss relative to the base model.}
    
    \label{fig:gen_len_fig}
\end{figure}

\begin{figure*}
    \centering
    \includegraphics[width=\linewidth]{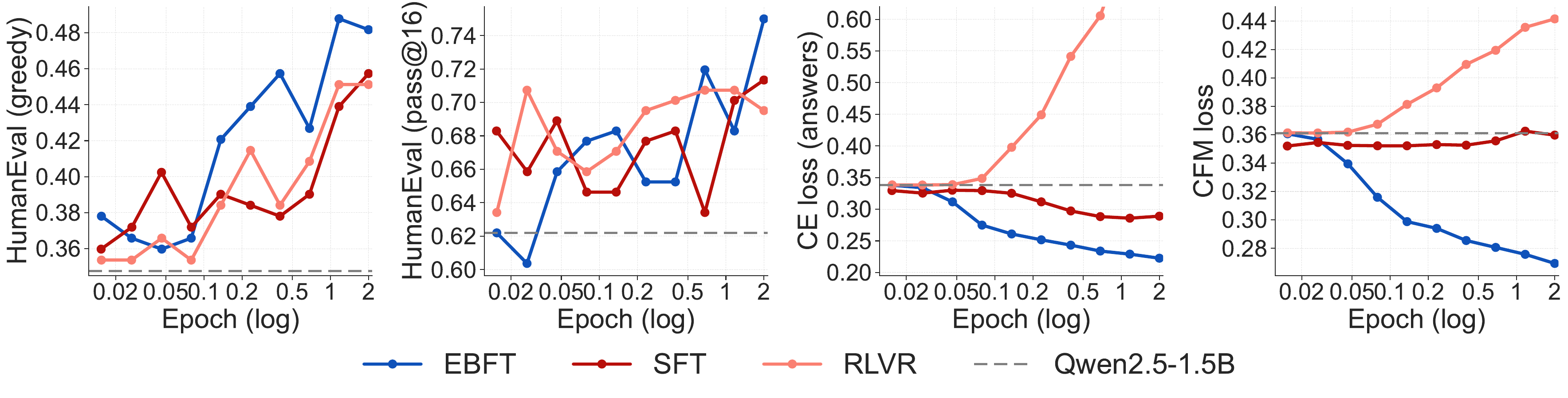}
    \caption{\textbf{EBFT improves downstream performance without sacrificing distributional calibration.} From left to right, we plot HumanEval accuracy (greedy and pass@16), validation cross-entropy (CE), and conditional feature-matching (CFM) loss over training for Qwen2.5-1.5B fine-tuned on OpenCodeInstruct~\citep{ahmad2025opencodeinstruct}. SFT improves cross-entropy and CFM loss but lags on downstream accuracy. RLVR improves downstream accuracy but substantially degrades both calibration metrics relative to the base model (dashed line). EBFT achieves the best results across all four metrics, avoiding this tradeoff. CE and CFM losses are computed on a 1k-samples held-out subset of OpenCodeInstruct.} 
    \label{fig:front_figure}
\end{figure*}

\paragraph{Contributions.}
We introduce a feature-matching loss for language model fine-tuning that targets sequence-level statistics of the rollout distribution, and propose Energy-Based Fine-Tuning (EBFT) as a practical method to optimize it. We provide a theoretical perspective connecting EBFT to KL-regularized energy-based models (\Cref{sec:KL_regularization}). Empirically, we observe the following across Q\&A coding, unstructured coding, and translation datasets:
\begin{enumerate}[leftmargin=*, itemsep=1pt, topsep=2pt]
    \item EBFT achieves the lowest feature-matching loss across all completion lengths (\Cref{fig:gen_len_fig}). Despite training with a rollout horizon of only 8 tokens, its gains persist and grow at longer generations, indicating genuine distributional calibration. In contrast, RLVR worsens this loss relative to the base model.
    \item On downstream performance, EBFT consistently outperforms SFT and is competitive with RLVR, despite requiring no task-specific reward or verifier.
    \item On validation cross-entropy, EBFT improves over SFT across all tasks, even though SFT explicitly optimizes this objective (\Cref{fig:front_figure}). RLVR, by contrast, substantially degrades validation perplexity.
    \item EBFT can be applied in non-verifiable settings where RLVR is inapplicable. For instance, when training on raw code scraped from GitHub, EBFT yields substantial gains over SFT.
\end{enumerate}

\section{Language Modeling with Feature Matching}
\label{sec:loss}

\subsection{The feature-matching loss}
\label{subsec:fm_loss}

\begin{figure*}[t]
  \centering
  \includegraphics[width=\textwidth]{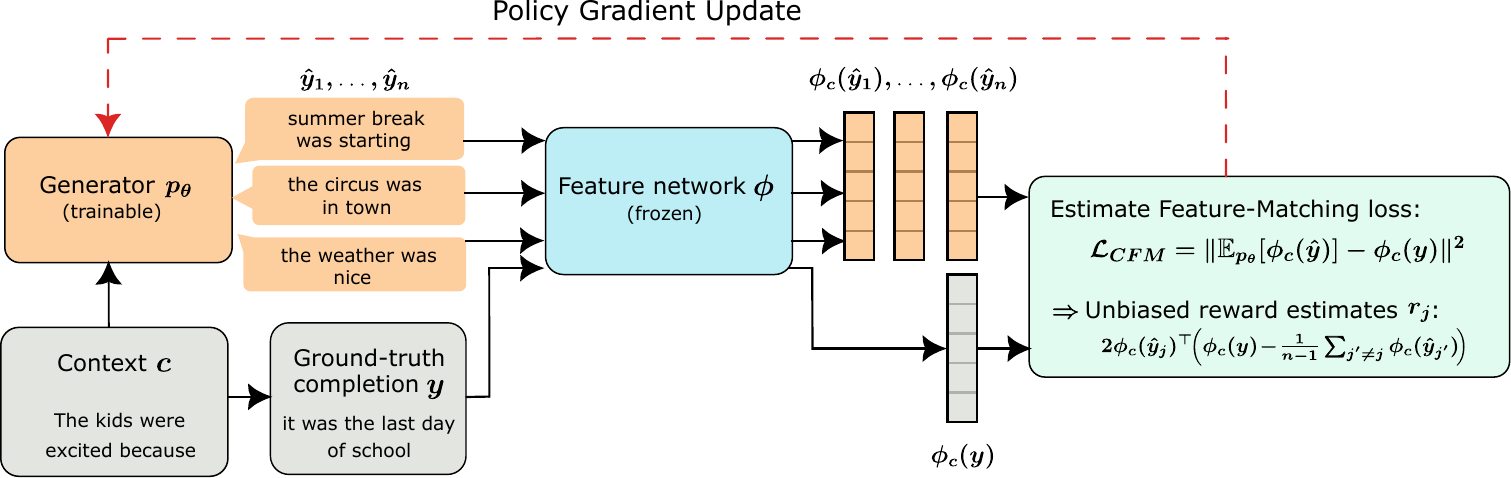}
  \caption{\textbf{Overview of Energy-Based Fine-Tuning (EBFT).} For each context $c$, the generator $p_\theta$ samples $n$ completions. A frozen feature network $\phi$ embeds each prompt--completion pair, producing features $\phi(c\!:\!\hat{y}_j)$ for the sampled completions and $\phi(c\!:\!y)$ for the ground truth. Each completion receives a feature-matching reward measuring alignment with the ground-truth feature moment, and the generator is updated via REINFORCE with an RLOO baseline.}
  \label{fig:ebm}
\end{figure*}

Given vocabulary $\mathcal{V}$ and  ground truth distribution $p$ over contexts $c \in \mathcal{V}^{*}$ and completions $y \in \mathcal{V}^{G}$ of length $G$, and a language model $p_{\theta}$, the feature-matching loss function is:
\begin{equation}
\begin{aligned} \label{eq:L_FM}
    &\mathcal{L}_{\mathrm{FM}}(\theta)\\
    \! :=& \mathbb{E}_{c \sim p}\Big[
      \big\| \mathbb{E}_{\hat{y} \sim p_{\theta}(\cdot|c)}[\phi(c\!:\!\hat{y})]
      \! - \! \mathbb{E}_{y \sim p(\cdot|c)}[\phi(c\!:\!y)] \big\|^2
    \Big],
\end{aligned}
\end{equation}
where $c\!:\!y$ denotes concatenation and $\phi : \mathcal{V}^{*} \to \mathbb{R}^d$ is the feature map. We use the short-hand $\phi_c(y) := \phi(c\!:\!y)$.
Since $\mathcal{L}_{\mathrm{FM}}$ depends on the unknown data moment
$\mathbb{E}_{y \sim p(\cdot|c)}[\phi_c(y)]$, it cannot be directly estimated from
ground-truth pairs $(c,y)$.
A bias-variance decomposition lets us write the feature-matching loss in terms of the \emph{conditional feature-matching loss}:
\begin{align}
\begin{split} \label{eq:L_CFM}
&\mathcal{L}_{\mathrm{FM}}(\theta) = \mathcal{L}_{\mathrm{CFM}}(\theta)
    - \mathbb{E}_{c \sim p}\Big[ \mathrm{Var}[\phi_c(y)|c] \Big], 
\end{split}
\end{align}
where $\mathcal{L}_{\mathrm{CFM}}(\theta):= \mathbb{E}_{c \sim p}\Big[\big\| \mathbb{E}_{\hat{y} \sim p_{\theta}(\cdot|c)}[\phi_c(\hat{y})] \! - \! \phi_c(y) \big\|^2\Big]$, $\mathrm{Var}[\phi_c(y)|c] = \mathbb{E}_{y\sim p(\cdot|c)}\Big[\big\|\phi_c(y) - \mathbb{E}_{y'\sim p(\cdot|c)}[\phi_c(y')] \big\| \Big].$

The offset $\mathbb{E}_{c \sim p}[\mathrm{Var}[\phi_c(y)|c]]$ is independent of $\theta$
and captures the per-context variability of the ground-truth features.
Since $\mathcal{L}_{\mathrm{FM}}$ has optimal value zero, $\mathcal{L}_{\mathrm{CFM}}$
equals this offset at optimality.
$\mathcal{L}_{\mathrm{CFM}}$ admits an unbiased estimator from ground-truth pairs $(c,y)$;
we plot these estimates in Figures~\ref{fig:gen_len_sft_fig} and~\ref{fig:gen_len_fig}.
Note that $\mathcal{L}_{\mathrm{FM}}$ and $\mathcal{L}_{\mathrm{CFM}}$ mirror the
(conditional) flow-matching loss functions for continuous generative modeling
\citep{lipman2022flow}.
\paragraph{Why minimize $\mathcal{L}_{\mathrm{FM}}$?} We say that $p_{\theta}$ is \emph{(feature-) calibrated} if $\mathcal{L}_{\mathrm{FM}}(\theta) = 0$, meaning its expected feature embeddings match those of the data for all contexts $c$. If the feature map $\phi$ is rich enough that matching feature moments implies matching distributions, then feature-calibration is equivalent to $p_{\theta} = p$ and $\mathcal{L}_{\mathrm{FM}}$ is a \emph{strictly proper scoring rule} \citep{gneiting2007strictly}. In other words, under a sufficiently expressive feature map, minimizing $\mathcal{L}_{\mathrm{FM}}$ is guaranteed to recover the true conditional distribution: it cannot be fooled by a model that matches some statistics while diverging elsewhere.

\paragraph{Relationship between feature-matching and cross-entropy.}
In practice, we optimize a mixed objective that combines feature matching with standard next-token cross-entropy (CE):
\begin{align*}
    \mathcal{L}(\theta) = \mathcal{L}_{\mathrm{FM}}(\theta) + \gamma\, \mathcal{L}_{\mathrm{CE}}(\theta), \qquad \gamma \ge 0.
\end{align*}

To build intuition, consider the special case of completions of length $G=1$ with one-hot features $\phi_c(y)=e_y \in \{0,1\}^{|\mathcal{V}|}$. Feature matching then reduces to an $\ell_2$ moment-matching loss on the next-token distribution:
\begin{align} \label{eq:FM_one_hot}
    \mathcal{L}_{\mathrm{FM}}(\theta) = \mathbb{E}_{c \sim p}\Big[ \sum_{y \in \mathcal{V}} \big(p_{\theta}(y|c) - p(y|c) \big)^2 \Big],
\end{align}
while CE is
\begin{align} \label{eq:CE_loss_sec_2}
    \mathcal{L}_{\mathrm{CE}}(\theta) = - \mathbb{E}_{c \sim p}\Big[ \sum_{y \in \mathcal{V}} p(y|c)\, \log p_{\theta}(y|c) \Big].
\end{align}
Both losses \emph{share} the same unique minimizer--- the ground-truth distribution $p_{\theta} = p$--- and so does their combination $\mathcal{L}$. However, their landscapes differ: $\mathcal{L}_{\mathrm{FM}}$ penalizes deviations from $p(y|c)$ symmetrically, while $\mathcal{L}_{\mathrm{CE}}$ penalizes underestimation more heavily than overestimation. More importantly, using longer completions (e.g., $G\in\{4,8,16\}$) with richer features lets $\mathcal{L}_{\mathrm{FM}}$ target sequence-level statistics that the token-level CE loss is blind to.  

\paragraph{Constructing the feature map.}
We instantiate $\phi$ as a frozen \emph{feature network} obtained by copying $p_\theta$ at initialization.
Given a concatenated sequence $c\!:\!y$, we take the concatenation of intermediate activations at different depths of the feature network, normalize each block to unit $L^2$ norm, and concatenate them to form $\phi(c\!:\!y)$.
In all experiments, we use layers at depths 25\%, 50\%, and 75\%; the intuition is that earlier layers capture low-level information, final layers are biased toward next-token prediction, and middle layers carry semantic and structural information. We hypothesize that such high-dimensional feature maps are close to satisfying the richness condition above.

\subsection{Feature-matching rewards and on-policy training}
\label{subsec:fm_reinforce}

This subsection derives an unbiased REINFORCE estimator for $\nabla_\theta \mathcal{L}_{\mathrm{FM}}(\theta)$ and describes the practical training recipe summarized in \Cref{alg:outer}.
 
\paragraph{Gradient estimation via REINFORCE.}
Since $\mathcal{L}_{\mathrm{FM}}$ and $\mathcal{L}_{\mathrm{CFM}}$ differ by a
constant independent of $\theta$ per \eqref{eq:L_CFM}, their gradients coincide:
$\nabla_{\theta} \mathcal{L}_{\mathrm{FM}}(\theta) =
\nabla_{\theta} \mathcal{L}_{\mathrm{CFM}}(\theta)$.
Hence, it suffices to estimate the gradient of the per-example loss
\begin{align}
    \mathcal{L}_{\mathrm{CFM}}(\theta;c,y)
    = \big\| \mathbb{E}_{\hat{y} \sim p_{\theta}(\cdot|c)}[\phi_c(\hat{y})]
      \! - \! \phi_c(y) \big\|^2,
    \label{eq:L_FM_surrogate}
\end{align}
which satisfies $\nabla_{\theta} \mathcal{L}_{\mathrm{CFM}}(\theta) =
\mathbb{E}_{(c,y) \sim p}[\nabla_{\theta} \mathcal{L}_{\mathrm{CFM}}(\theta;c,y)]$.
Using the product rule of differentiation and the widely used identity $\nabla_{\theta}\mathbb{E}_{\hat{y}\sim p_\theta}[g(\hat{y})] = \mathbb{E}_{\hat{y}\sim p_\theta}[g(\hat{y})\,\nabla_{\theta}\log p_\theta(\hat{y}\mid c)]$ yields a REINFORCE gradient
\begin{align*}
    \nabla_{\theta} \mathcal{L}_{\mathrm{CFM}}(\theta;c,y)
    = - \mathbb{E}_{\hat{y} \sim p_{\theta}(\cdot|c)} \big[ \nabla_{\theta} \log p_{\theta}(\hat{y}|c)\, r(\hat{y},c)\big],
\end{align*}
where the reward is
\begin{align} \label{eq:r_haty_c}
   \hspace*{-.3cm} r(\hat{y},c) = \underbrace{2 \phi_c(\hat{y})^{\top} \phi_c(y)}_{\text{alignment term}} - \underbrace{2 \phi_c(\hat{y})^{\top} \mathbb{E}_{\tilde{y} \sim p_{\theta}(\cdot|c)}\big[ \phi_c(\tilde{y}) \big]}_{\text{diversity term}}.
\end{align}
We obtain an unbiased estimator of this gradient by sampling $n > 1$ completions $(\hat{y}_j)_{j=1}^{n}$ from $p_{\theta}(\cdot|c)$ and computing
\begin{align}
    \begin{split} \label{eq:REINFORCE_practical}
        &\frac{1}{n} \sum_{j=1}^{n} \nabla_{\theta} \log p_{\theta}(\hat{y}_j|c) r_j, \quad \text{where} \\
        &r_j = 2 \phi_c(\hat{y}_j)^{\top} \phi_c(y) - \frac{2}{n - 1} \sum_{j' \neq j} \phi_c(\hat{y}_j)^{\top} \phi_c(\hat{y}_{j'}).
    \end{split}
\end{align}
As in REINFORCE, it is possible to reduce the variance of this gradient by subtracting from $r_j$ a baseline which is independent from $\hat{y}_j$.
We use REINFORCE leave-one-out (RLOO), but must account for the fact that $r_j$ already depends on the other completions; see \Cref{sec:RLOO_baseline} for the derivation of the REINFORCE gradient and RLOO baseline.

\begin{algorithm}[t]
\caption{One EBFT training iteration.}
\label{alg:outer}
\begin{algorithmic}[1]
\STATE \textbf{Input:} input prompt $c$; ground-truth completion $y$; $n$: samples per prompt; generation length $G$; 
\emph{generator} $p_\theta$; \emph{feature network} $\phi$
\STATE \textbf{Generation:} sample $n$ rollouts of length $G$ from the actor: $(\hat{y}_j)_{j=1}^{n} \sim p_\theta(\cdot \mid c)$
\STATE \textbf{Feature network embeddings:} compute the ground truth feature vector $\phi_c(y)$
and the rollout feature vectors $\big(\phi_c(\hat{y}_j)\big)_{j=1}^{n}$. Whiten the features as in \eqref{eq:whitened_features} if needed.
\STATE \textbf{Feature-matching reward:} For $j=1:n$, compute
\begin{align*}
r_j \! = \! 2 \phi_c(\hat{y}_j)^{\top} \phi_c(y) \! - \! \frac{2}{n \!- \! 1} \! \sum_{j'=1, j' \neq j}^{n} \! \phi_c(\hat{y}_j)^{\top} \phi_c(\hat{y}_{j'}),
\end{align*}
and the RLOO baseline $b^j$ as in \eqref{eq:RLOO_baseline} in \Cref{sec:RLOO_baseline}.
\STATE \textbf{Actor update:} update $p_\theta$ with an RLOO update across $j=1,\dots,n$.
\end{algorithmic}
\end{algorithm}

\paragraph{Energy-Based Fine-Tuning (EBFT) training recipe.}
\label{sec:method}
\Cref{alg:outer} summarizes one EBFT iteration. EBFT uses two models: a \emph{generator} $p_\theta$, which is the model that we want to fine-tune, and a \emph{feature network} $\phi$.
In what follows, we only train the generator; we keep the feature network \emph{frozen}. Given a pair $(c,y)$ of ground truth context and completion of length $G$, we generate $n$ completions $(\hat{y}_j)_{j=1}^{n}$ of the same length, and we feed the concatenated sequences $c\!:\!y$ and $(c\!:\!\hat{y}_j)_{j=1}^{n}$ through the feature network to obtain the feature vectors $\phi_c(y)$ and $(\phi_c(\hat{y}_j))_{j=1}^{n}$, which we use to get the rewards $(r_j)_{j=1}^{n}$ and the RL gradient following equation \eqref{eq:REINFORCE_practical}.

In practice, we introduce additional implementation details which affect how the method is instantiated: an optional \textit{alignment-biased reward} that adjusts the fidelity--diversity trade-off, an efficient \textit{strided block-parallel rollout scheme} for collecting many on-policy samples, and \textit{feature whitening} to improve the conditioning of the feature space. We describe the first two below, and discuss whitening in the following subsection.

\paragraph{Adding an alignment bias.} In some settings, one may prefer higher-fidelity samples at the cost of reduced diversity. This can be achieved by scaling the diversity term of the reward in \eqref{eq:r_haty_c} by a factor $\alpha \in (0,1)$, which biases the objective toward closer alignment with the ground-truth features. We describe this variant in \Cref{sec:alignment_bias} and include experiments in \Cref{app:alpha_gamma_sweep}.

\paragraph{Strided block-parallel rollouts.}
To obtain many on-policy rollouts per training sequence efficiently, we use a strided block-parallel decoding scheme implemented with a custom attention mask (introduced by Quiet-STaR \cite{zelikman2024quietstarlanguagemodelsteach}).
At a high level, this amortizes prefix computation and enables batched feature-network evaluation across many anchored prompts extracted from the same sequence.
We give details and an example in \Cref{sec:strided_parallel}.

\subsection{Feature matching with whitening}
\label{subsec:fm_variants}
When the feature map $\phi$ has correlated or anisotropic 
directions, some dimensions can dominate the feature-matching 
loss. We address this with a whitened variant. For each context 
$c$ and sampled completions $(\hat{y}_j)_{j=1}^{n}$, we estimate 
the second-moment matrix 
$\hat{\Sigma}_{c} = \frac{1}{n}\sum_{j=1}^{n} 
\phi_c(\hat{y}_j)\phi_c(\hat{y}_j)^{\top}$ and define whitened 
features
\begin{align}
    \tilde{\phi}_c(z) = 
(\hat{\Sigma}_{c}^{\dagger})^{1/2}\phi_c(z),
    \label{eq:whitened_features}
\end{align}
where $\dagger$ denotes the Moore--Penrose pseudoinverse. The 
whitened feature-matching loss corresponds to a relaxation of the 
local $\chi^2$ divergence between $p(\cdot\mid c)$ and 
$p_{\theta}(\cdot\mid c)$. Since 
$D_{\mathrm{KL}}(P\|Q) \approx \tfrac{1}{2}D_{\chi^2}(P\|Q)$ 
when $P$ and $Q$ are close, whitened feature matching approximates 
a sequence-level cross-entropy in the fine-tuning regime 
$p_{\theta} \approx p$ (see \Cref{sec:whitening}).
However, whitening with the low-rank estimate 
$\hat{\Sigma}_c$ instead of the true second-moment matrix 
$\Sigma_c$ systematically reduces the norm of the whitened 
ground-truth feature $\tilde{\phi}_c(y)$, weakening the alignment 
signal. We find that normalizing the whitened features \emph{only 
in the alignment term} corrects this, yielding the reward used in 
all whitening experiments in this paper:
\begin{align}
    r_j =
    \underbrace{
    2\,
    \frac{\tilde{\phi}_c(\hat{y}_j)^{\top}\tilde{\phi}_c(y)}
    {\|\tilde{\phi}_c(\hat{y}_j)\|\,\|\tilde{\phi}_c(y)\|}
    }_{\text{normalized alignment term}}
    -
    \underbrace{
    \frac{2}{n-1}\sum_{j' \neq j}
    \tilde{\phi}_c(\hat{y}_j)^{\top}\tilde{\phi}_c(\hat{y}_{j'})
    }_{\text{whitened diversity term}}.
    \label{eq:r_j_whitened_main}
\end{align}
The diversity term is left unnormalized to retain the full 
whitened geometry. Additional variants are studied in 
\Cref{sec:whitening}.

\subsection{Connections with energy-based models and calibration}
\label{subsec:fm_connections}
 
To conclude this section, we further motivate EBFT by showing that, under KL regularization, feature matching admits both a calibration view and an energy-based interpretation.

\paragraph{Energy-based view (via KL regularization).}
As in standard reinforcement learning, one may add a Kullback--Leibler (KL) regularization term to prevent the learned distribution from deviating too far from a reference distribution $q(\cdot| c)$. Consider the KL-regularized objective
\begin{align}
\begin{split} \label{eq:fm_kl}
&\min_{\rho} 
\;\; \mathbb{E}_{c \sim p} \Big[
\big\| \mathbb{E}_{\rho(\cdot| c)}[\phi_c(y)] - \mathbb{E}_{p(\cdot| c)}[\phi_c(y)] \big\|^2
\\ &\qquad\qquad\quad +\; {\textstyle \frac{1}{\beta}} \, 
D_{\mathrm{KL}}
\!\left(\rho(\cdot| c)\,\|\,q(\cdot|c)\right) \Big],
\end{split}
\end{align}
where $\beta>0$ controls the strength of the regularization.
Although we do not include this KL term in our experiments, it provides a useful interpretation of EBFT.
In particular, the solution to \eqref{eq:fm_kl} has the form of an exponential tilt of the base distribution,
\begin{align*}
\rho^{\star}(y|c) \propto q(y|c)\,\exp\!\big(-\chi_c^{\top}\phi_c(y)\big),
\end{align*}
for a context-dependent vector $\chi_c \in \mathbb{R}^d$.
Intuitively, $\chi_c$ is the tilt direction that assigns the most probability to
completions actually observed in the data, subject to a size constraint on
$\|\chi_c\|$; see Theorem~\ref{thm:EBFT_KL}
for the precise statement.
This is precisely the maximum-likelihood problem for an energy-based model with
energy function $E(y,c) = \chi_c^{\top}\phi_c(y)$, motivating the term
\emph{energy-based fine-tuning}.
Importantly, EBFT does not explicitly parameterize or learn $\chi$; instead, it directly optimizes the generator parameters via feature-matching gradients.
We provide a detailed derivation of this connection in \Cref{sec:KL_regularization}.

\begin{table*}[t]
\centering
\normalsize
\setlength{\tabcolsep}{5pt}
\renewcommand{\arraystretch}{1.15}

\begin{tabular}{l cccccc cccccc}
\toprule
& \multicolumn{6}{c}{\textbf{Q\&A Coding}} 
& \multicolumn{6}{c}{\textbf{Unstructured Coding}} \\
\cmidrule(lr){2-7}\cmidrule(lr){8-13}
\textbf{Method} 
& {\footnotesize \textbf{CE}} & {\footnotesize \textbf{FM}} & {\footnotesize \textbf{greedy}} & {\footnotesize \textbf{pass@1}} & {\footnotesize \textbf{pass@4}} & {\footnotesize \textbf{pass@16}}
& {\footnotesize \textbf{CE}} & {\footnotesize \textbf{FM}} & {\footnotesize \textbf{greedy}} & {\footnotesize \textbf{pass@1}} & {\footnotesize \textbf{pass@4}} & {\footnotesize \textbf{pass@16}} \\
\midrule
Base         & 
0.338 & 
0.361 & 0.484 & 0.424 & 0.606 & 0.715 & 0.631 & 
0.369 & 0.473 & 0.419 & 0.596 & 0.702 \\
Warm start    & 
0.301 & 
0.344 & 0.483 & 0.440 & 0.611 & 0.723 & 0.499 & 
0.317 & 0.508 & 0.458 & 0.638 & 0.743 \\
\midrule
SFT          & 
0.289 & 
0.315 & 0.483 & 0.455 & 0.617 & 0.728 & 0.501 & 
0.321 & 0.504 & 0.467 & 0.644 & 0.747 \\
EBFT         & 
0.207 & 
0.258
& \textbf{0.548} & 0.510 & 0.659 & \textbf{0.771} & 0.499 &
0.320 & \textbf{0.548} & \textbf{0.524} & \textbf{0.664} & \textbf{0.769} \\
EBFT (ws.)   & 
\textbf{0.190} & 
\textbf{0.255} & 0.534 & 0.508 & 0.658 & 0.756 & \textbf{0.481} &  
\textbf{0.312} & 0.536 & 0.514 & 0.659 & \textbf{0.769} \\
\midrule
RLVR         & 
0.774 & 
0.442 & 0.535 & 0.510 & 0.660 & 0.752 & -- & -- & -- & -- & -- & -- \\
RLVR (ws.)   & 
0.389 & 
0.402 & 0.524 & \textbf{0.529} & \textbf{0.662} & 0.749 & -- & -- & -- & -- & -- & -- \\
\bottomrule
\end{tabular}

\vspace{2mm}
 
\begin{tabular}{l cc cccc cccc}
\toprule
& \multicolumn{2}{c}{

} 
& \multicolumn{4}{c}{\textbf{Translation (COMET)}} 
& \multicolumn{4}{c}{\textbf{Translation (BLEU)}} \\
\cmidrule(lr){2-3}\cmidrule(lr){4-7}\cmidrule(lr){8-11}
\textbf{Method} 
& {\footnotesize \textbf{CE}} & {\footnotesize \textbf{FM}}
& {\footnotesize \textbf{greedy}} & {\footnotesize \textbf{best-of-1}} & {\footnotesize \textbf{best-of-4}} & {\footnotesize \textbf{best-of-16}}
& {\footnotesize \textbf{greedy}} & {\footnotesize \textbf{best-of-1}} & {\footnotesize \textbf{best-of-4}} & {\footnotesize \textbf{best-of-16}} \\
\midrule
Base         & 1.870 & 0.637 & 0.644 & 0.611 & 0.701 & 0.745 & 0.074 & 0.124 & 0.186 & 0.231 \\
Warm start   & 2.647 & 0.695 & 0.711 & 0.691 & 0.759 & 0.793 & 0.158 & 0.169 & 0.233 & 0.279 \\
\midrule
SFT          & 1.782 & 0.690 & 0.717 & 0.696 & 0.761 & 0.795 & 0.160 & 0.172 & 0.235 & 0.280 \\
EBFT         & \textbf{1.670} & \textbf{0.578} & 0.725 & 0.713 & 0.765 & 0.795 & 0.182 & 0.194 & 0.244 & 0.283 \\
EBFT (ws.)   & 1.671 & 0.580 & \textbf{0.734} & \textbf{0.724} & \textbf{0.772} & \textbf{0.800} & 0.185 & 0.197 & \textbf{0.247} & \textbf{0.286} \\
\midrule
RLVR         & 2.454 & 0.641 & 0.697 & 0.691 & 0.735 & 0.761 & 0.176 & 0.194 & 0.226 & 0.248 \\
RLVR (ws.)   & 2.311 & 0.718 & 0.724 & 0.718 & 0.759 & 0.781 & \textbf{0.195} & \textbf{0.210} & 0.245 & 0.269 \\
\bottomrule
\end{tabular}

\vspace{2mm}
\caption{\textbf{EBFT outperforms SFT and matches or exceeds RLVR on downstream metrics, while achieving the best distributional calibration across all tasks.} Best results per method on Q\&A coding, unstructured coding, and translation. CE: validation cross-entropy; FM: feature-matching loss (both lower is better). ``ws.'': warm-started from an SFT checkpoint. RLVR is inapplicable to unstructured coding where no verifier exists; EBFT still yields substantial gains over SFT in this setting. See \Cref{tab:per_benchmark_results} for per-benchmark results and \Cref{sec:experiments} for full experimental details.} 
\label{tab:best_prefix_aggregate}
\vspace{-10pt}
\end{table*}

\paragraph{Calibration view: KL projection onto moment constraints.}
The KL-regularized objective also has a calibration interpretation.
Suppose we want a distribution that matches a target statistic
$\mathbb{E}_{p(\cdot\mid c)}[f(y,c)] = m$
while staying close to a base distribution $q(\cdot\mid c)$.
The solution to this constrained KL minimization is an exponential tilt,
\begin{align}
p_{\chi}(y\mid c) \propto \exp\!\big(\chi_c^{\top} f(y,c)\big)\, q(y\mid c),
\end{align}
where $\chi_c$ is chosen so that the moment constraint is satisfied.
\citet{braverman2019calibration} use this principle to correct entropy-rate drift
in language model generations, 
applying a \emph{scalar} tilt with
$f(y,c) = -\log p_{\theta}(y \mid c)$ (the negative log-probability of the model).
EBFT performs the same type of correction, but with $f(y,c) = -\phi_c(y)$,
enforcing \emph{high-dimensional} moment constraints in a semantically rich
feature space rather than a single scalar statistic.

\begin{figure*}[!t]
    \centering
    \includegraphics[width=0.99\linewidth]{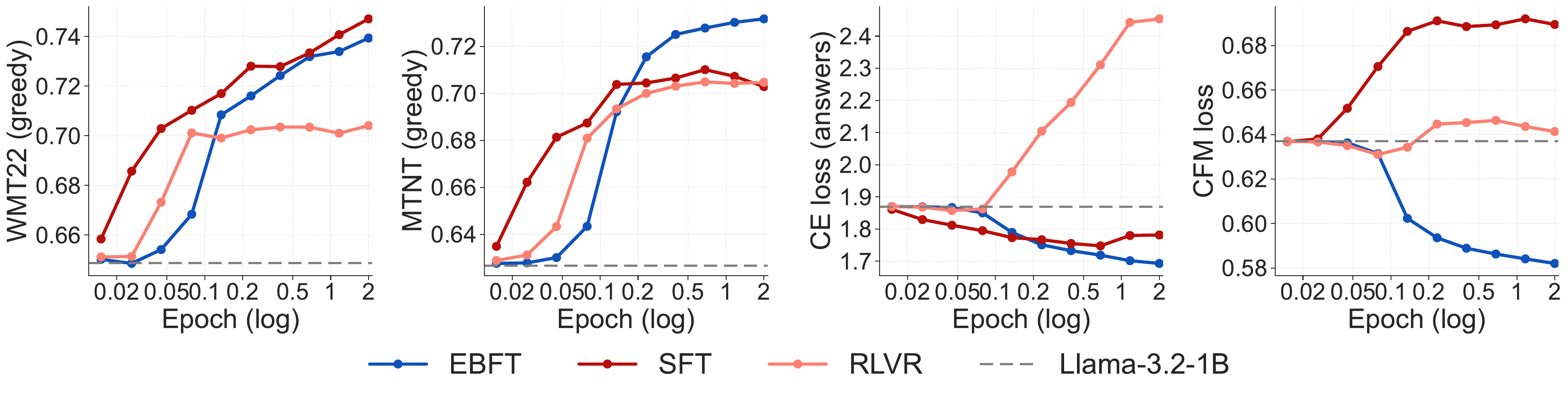}
    \caption{
   \textbf{On translation, EBFT outperforms both SFT and RLVR on downstream accuracy, cross-entropy, and feature-matching loss.} From left to right, we plot COMET scores on WMT22 and MTNT, validation cross-entropy, and CFM loss over training for Llama-3.2-1B fine-tuned on ALMA~\citep{xu2023paradigm}. EBFT achieves the lowest CE and CFM losses and matches SFT on WMT22 while clearly outperforming it on MTNT. RLVR underperforms SFT on all four metrics, with cross-entropy rising well above the base model (dashed line).}
    
    \label{fig:translation_main}
    \vspace{-10pt}
\end{figure*}

\begin{figure*}[!t]
    \centering
    \includegraphics[width=\linewidth]{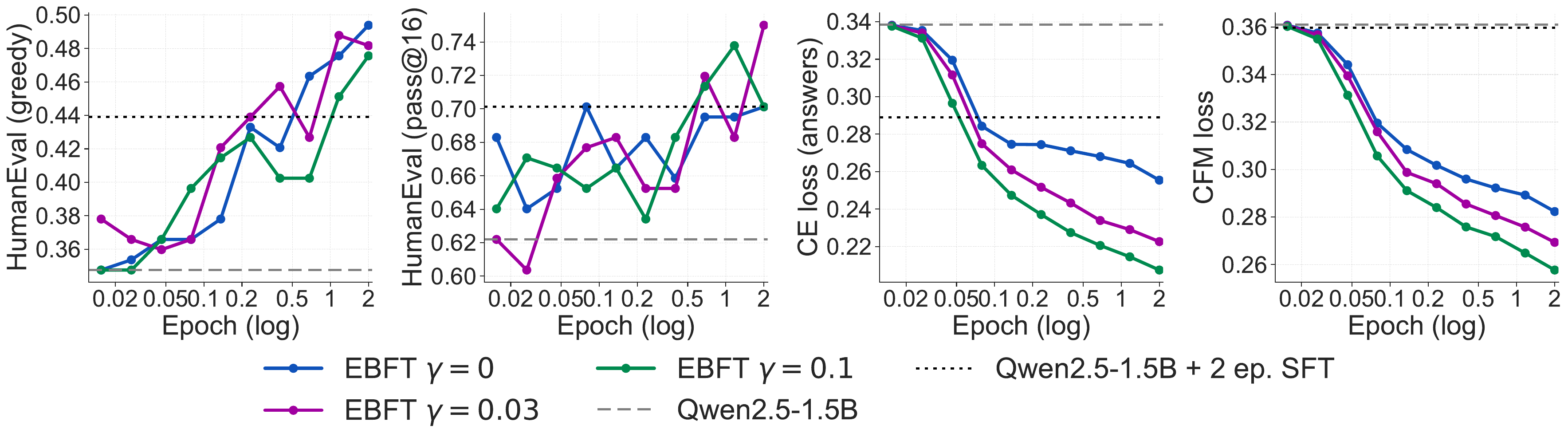}
    \caption{\textbf{The CE regularization weight $\gamma$ controls cross-entropy reduction without affecting downstream performance or feature-matching loss.} We ablate EBFT with $\gamma \in \{0, 0.03, 0.1\}$ on Qwen2.5-1.5B. Dashed and dotted lines indicate the base model and 2-epoch SFT, respectively. Larger $\gamma$ accelerates CE reduction while downstream accuracy and CFM loss remain nearly identical across settings. Even pure feature matching ($\gamma=0$) surpasses SFT on cross-entropy, confirming that the two objectives are aligned rather than in tension.}
\label{fig:sweep_gamma_ce_moment_downstream}
    \vspace{-10pt}
\end{figure*}

\section{Experimental protocol}
\label{sec:experiments}

\subsection{Tasks and metrics}
We evaluate EBFT on tasks spanning both \emph{verifiable} settings, where a correctness signal exists and RLVR can be applied, and \emph{non-verifiable} settings, where no such signal is available and SFT is typically the only option. For all tasks, we train on subsets of the full datasets to enable controlled comparisons across methods under a fixed compute budget.

\paragraph{Coding tasks.}
We consider two complementary training regimes: (a)  Q\&A coding uses a 100k-sample subset of OpenCodeInstruct~\citep{ahmad2025opencodeinstruct}, consisting of natural-language programming prompts paired with reference solutions, and (b)  Unstructured coding uses a 40k-sample subset of SwallowCode~\citep{fujii2025rewritingpretrainingdataboosts}, containing raw Python code without explicit instructions. The former is a verifiable setting (solutions can be checked against unit tests); the latter is not, as there is no correctness signal for raw code continuation, making RLVR inapplicable.

We evaluate on HumanEval~\citep{austin2021program}, MBPP~\citep{chen2021codex}, and MultiPL-E~\citep{cassano2023multipl}, reporting greedy accuracy (temperature 0) as well as \texttt{pass@1}, \texttt{pass@4}, and \texttt{pass@16} accuracy at temperature 0.6. For models trained on Q\&A coding data, HumanEval and MBPP can be considered in-distribution benchmarks, since OpenCodeInstruct contains similar instruction-solution pairs. For models trained on unstructured code, both benchmarks are out-of-distribution, as SwallowCode contains raw Python without explicit prompts or test cases. MultiPL-E translates HumanEval problems into many programming languages; we evaluate on eight of them (C++, JavaScript, TypeScript, Rust, C\#, Go, PHP, and Java). Since all training data is Python-only, MultiPL-E is out-of-distribution for both training regimes and serves primarily as a test of cross-lingual transfer.
\paragraph{Translation.}
We train on a 100k subset of ALMA-Human-Parallel~\citep{xu2023paradigm,xu2024contrastive}, consisting of human-curated parallel sentence pairs. Following \citet{xu2023paradigm}, we use WMT'22 as our primary evaluation benchmark, which covers news and general-domain translation. To test out-of-distribution robustness, we additionally evaluate on two challenging benchmarks. MTNT~\citep{michel2018mtnt} consists of noisy Reddit comments featuring typos, slang, and code-switching, while OpenSubtitles~\citep{lison2016opensubtitles2016} contains short, informal movie and TV dialogue. Both are stylistically far from the clean, formal parallel sentences in ALMA, making them challenging out-of-distribution benchmarks. We report COMET scores in the main text and BLEU in the appendix. For best-of-$k$ evaluation ($k \in \{1, 4, 16\}$, temperature 0.6), we report the per-instance maximum aggregated over the test set.

In addition to downstream task metrics, we track validation cross-entropy and feature-matching loss throughout training on 1k-sample held-out subsets of the respective training datasets (OpenCodeInstruct for coding, ALMA for translation), as these quantities are central to our analysis.

\subsection{Baselines and methods}
We evaluate three methods: (a) standard CE fine-tuning (SFT); (b) RLVR, where the reward is whether the generated code passes all unit tests for Q\&A coding, and BLEU score for translation; and (c) EBFT with a frozen feature network. RLVR is only applicable to Q\&A coding and translation, where verifiable rewards exist. All methods are initialized from the base pre-trained model (Qwen2.5-1.5B~\citep{qwen2025qwen25technicalreport} for coding and Llama3.2-1B~\citep{grattafiori2024llama} for translation). All EBFT runs use whitening as described in \Cref{eq:whitened_features}. We run all methods for 2 epochs. As an additional variant, we report results for EBFT and RLVR initialized from a \textit{warm-start} checkpoint obtained after one epoch of SFT, followed by one epoch of EBFT or RLVR. We include this setting because, as we show in \Cref{sec:results}, RLVR requires a warm-start to achieve competitive downstream performance. Hyperparameter details are provided in \Cref{sec:hyperparams}.

\section{Experimental results}
\label{sec:results}

We evaluate EBFT against SFT and RLVR on Q\&A coding, unstructured coding, and translation. The main finding is that EBFT consistently matches or exceeds RLVR on downstream accuracy while achieving the best cross-entropy and feature-matching losses across all tasks --- avoiding the tradeoff between task performance and distributional quality that characterizes RLVR. \Cref{tab:best_prefix_aggregate} summarizes the best results per method; Figures~\ref{fig:front_figure} and~\ref{fig:translation_main} show training dynamics for representative runs; Figures~\ref{fig:sweep_gamma_ce_moment_downstream}--\ref{fig:humaneval_scale} report ablations. Full results across hyperparameter sweeps are provided in \Cref{sec:add_experimental_results}.

\subsection{Main results}

\paragraph{EBFT matches RLVR and outperforms SFT on downstream accuracy.}
On Q\&A coding (\Cref{tab:best_prefix_aggregate,fig:front_figure}), EBFT outperforms SFT by a wide margin across all decoding strategies (e.g., greedy: 0.548 vs 0.483, pass@16: 0.771 vs 0.728) and matches or exceeds RLVR (greedy: 0.548 vs 0.535, pass@16: 0.771 vs 0.752), despite not using any correctness signal. On unstructured code (\Cref{tab:best_prefix_aggregate}), where RLVR is inapplicable, EBFT similarly outperforms SFT across all metrics (pass@1: 0.524 vs 0.467, pass@16: 0.769 vs 0.747). On translation (\Cref{tab:best_prefix_aggregate}), EBFT outperforms both SFT and RLVR on COMET scores across all decoding strategies (e.g., greedy: 0.725 vs 0.717 for SFT and 0.697 for RLVR).

\paragraph{EBFT achieves lower cross-entropy than SFT, while RLVR degrades it.}
A striking finding is that EBFT reduces the validation cross-entropy more than SFT, even though SFT explicitly optimizes this objective. On Q\&A coding, EBFT achieves a validation CE of 0.207 compared to 0.289 for SFT (\Cref{tab:best_prefix_aggregate}) and \Cref{fig:front_figure} shows that this gap widens steadily over training. On translation, EBFT similarly outperforms SFT on CE (1.670 vs 1.782). On unstructured code, the two methods are comparable (0.499 vs 0.501). We attribute this counterintuitive result to EBFT with whitening approximately optimizing a relaxation of the $\chi^2$ divergence, which is locally equivalent to the KL divergence when the model is close to the data distribution (see \Cref{subsec:fm_variants}). RLVR exhibits the opposite behavior: its validation CE increases throughout training (\Cref{fig:front_figure}), reaching 0.774 on Q\&A coding and 2.454 on translation --- both substantially worse than the base model (0.338 and 1.870, respectively). This confirms that reward-driven optimization can improve downstream accuracy at the cost of severely degrading the model's language modeling quality, a tradeoff that EBFT avoids entirely.

\begin{figure*}[!t]
    \centering
    \includegraphics[width=0.99\linewidth]{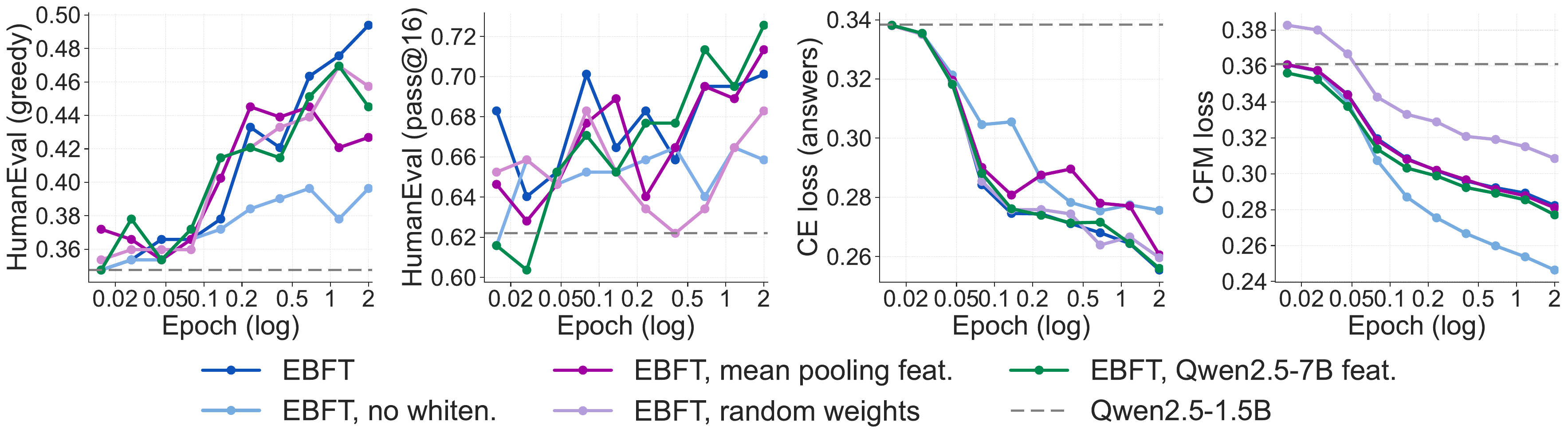}
    \caption{\textbf{Feature network ablations: whitening and last-token pooling matter most; scaling the feature network does not help.} HumanEval accuracy, validation cross-entropy, and CFM loss over training for EBFT ($\gamma=0$) on Qwen2.5-1.5B with different feature network configurations. The default (last-token features with whitening from a frozen 1.5B copy) achieves the best downstream accuracy and CFM loss. Removing whitening and mean pooling cause the largest degradations. Random weights hurt only modestly, and replacing the 1.5B feature network with a frozen Qwen2.5-7B yields similar results, suggesting that pre-trained representations help but that naively scaling the feature network does not.}\label{fig:humaneval_feature_network}
\end{figure*}

\begin{figure*}[!t]
    \centering
    \includegraphics[width=0.99\linewidth]{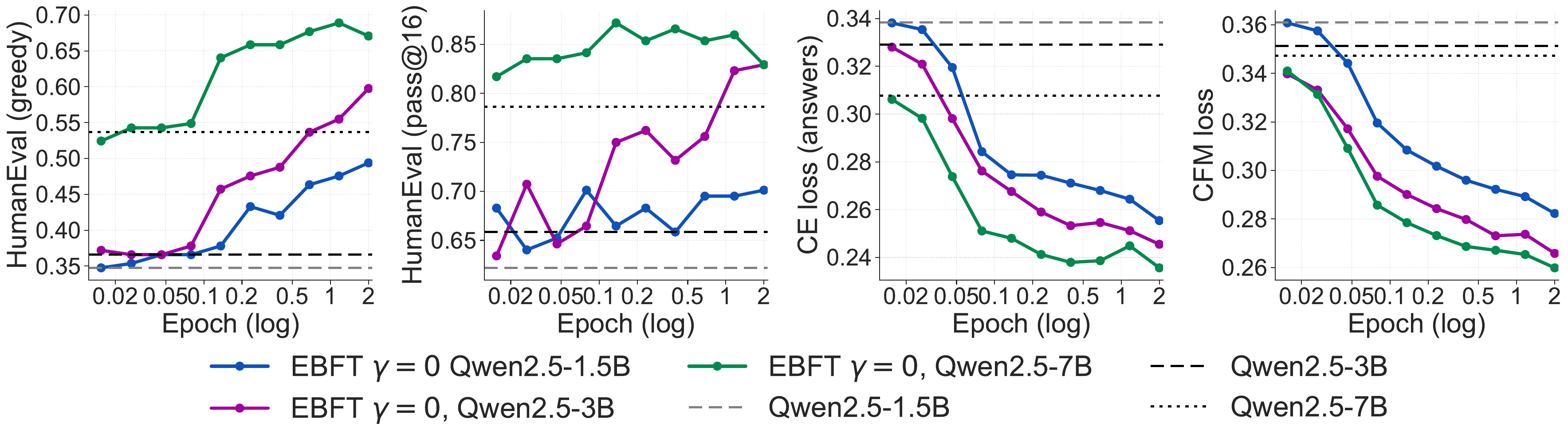}
    \caption{\textbf{EBFT improvements are consistent across model scales.} HumanEval accuracy, validation cross-entropy, and CFM loss over training for EBFT ($\gamma=0$) applied to Qwen2.5-1.5B, 3B, and 7B. Each model uses a frozen copy of itself as the feature network. Dashed lines indicate base model performance. All three scales show substantial and qualitatively similar improvements across all four metrics, with no sign of diminishing returns.}
    \label{fig:humaneval_scale}
\end{figure*}

\paragraph{EBFT achieves the lowest feature-matching loss.}
While it is natural to expect improvements on the feature-matching metric that EBFT directly optimizes, the margins are informative. 
On Q\&A coding, EBFT achieves a feature-matching loss of 0.258, compared to 0.315 for SFT and 0.442 for RLVR (\Cref{tab:best_prefix_aggregate}). RLVR not only fails to improve this metric but actively worsens it relative to the base model (0.361), and \Cref{fig:front_figure} shows that this degradation accelerates over training. On translation, EBFT achieves the largest improvement (0.578 vs 0.690 for SFT and 0.641 for RLVR), while on unstructured code EBFT and SFT are comparable (0.320 vs 0.321). As shown in \Cref{fig:gen_len_fig}, this improvement holds across all completion lengths and extends well beyond the 8-token rollout horizon used during training, suggesting that EBFT improves calibration of the rollout distribution broadly rather than overfitting to the training sequence length.

\paragraph{The CE loss coefficient $\gamma$ improves CE loss without sacrificing downstream performance or feature matching.}
\Cref{fig:sweep_gamma_ce_moment_downstream} compares EBFT with $\gamma \in \{0, 0.03, 0.1\}$ on Q\&A coding. The three settings achieve nearly identical feature-matching loss trajectories and comparable downstream performance on HumanEval, but differ markedly in how fast the validation cross-entropy decreases: larger $\gamma$ drives it down faster, with $\gamma = 0.1$ reaching a CE of approximately 0.21 compared to 0.25 for $\gamma = 0$. All three settings surpass the 2-epoch SFT baseline on cross-entropy, confirming that even pure feature matching ($\gamma = 0$) reduces CE more effectively than directly optimizing it. The absence of any tension between these objectives is expected from a theoretical standpoint: as mentioned in \Cref{subsec:fm_loss}, $\mathcal{L}_{\mathrm{FM}}$ and $\mathcal{L}_{\mathrm{CE}}$ share the same minimizer (the ground-truth distribution $p$), so optimizing feature matching naturally drives the cross-entropy down as well. The role of $\gamma$ is simply to control how aggressively the CE loss is minimized, at no cost to calibration or downstream accuracy.
\paragraph{EBFT generalizes better than SFT to out-of-distribution benchmarks.}
On out-of-distribution coding languages (MultiPL-E benchmark performance in \Cref{tab:per_benchmark_results}), SFT \emph{degrades} performance relative to the base model (greedy: 0.465 vs 0.506), while EBFT yields improvements (0.524). On translation, EBFT outperforms both SFT and RLVR on the noisy MTNT benchmark (greedy COMET: 0.737 vs 0.703 and 0.705), while performing comparably on OpenSubtitles.
\subsection{Ablations}
\paragraph{Feature network ablations: mean pooling and removing whitening hurt most; random weights hurt slightly; a larger feature network has little effect.}
\Cref{fig:humaneval_feature_network} ablates key feature network design choices on Q\&A coding at $\gamma = 0$. The default configuration (last-token features with whitening from a frozen copy of the 1.5B generator) achieves the best downstream performance and lowest feature-matching loss. Mean pooling and removing whitening cause the largest degradations, while random feature network weights hurt only modestly, indicating that pre-trained representations are helpful but not essential. Perhaps surprisingly, replacing the 1.5B feature network with a frozen Qwen2.5-7B produces similar results, suggesting that naively scaling the feature network does not yield additional gains. 

 \paragraph{EBFT improvements scale consistently across model sizes.}
To assess whether EBFT's benefits persist at larger scales, we run EBFT with $\gamma=0$ using Qwen2.5-1.5B, 3B, and 7B as both actor and feature networks, each initialized from the respective base checkpoint. As shown in \Cref{fig:humaneval_scale}, downstream improvements are consistent across model sizes: greedy HumanEval scores increase from approximately 0.49 (1.5B) to 0.60 (3B) to 0.69 (7B), with each model improving substantially over its respective base performance (0.35, 0.37, and 0.55). The same figure shows that both validation cross-entropy and feature-matching losses decrease faster and reach lower absolute values at larger scales, while preserving the same monotonic ordering across runs. These results suggest that EBFT's mechanism, which matches rollout feature statistics to ground-truth statistics, transfers predictably across model scales. 
 
\subsection{Qualitative analysis}
Across both code and translation, EBFT outputs are more \emph{semantically faithful} to the prompt and more \emph{cleanly formatted}. We provide representative generations from HumanEval and MTNT translation in \Cref{app:qual_analysis_code,app:qual_analysis_translation}; here we summarize the main patterns. 
 
Each method exhibits a characteristic failure mode. SFT typically produces structurally reasonable outputs but misses subtle prompt requirements. For instance, when asked to count overlapping substring occurrences, SFT advances by the full substring length and misses overlaps; when asked to return the greatest integer satisfying a condition, SFT returns the first one it finds instead. RLVR often generates plausible logic but fails at the execution level: the generated code calls helper functions like \texttt{is\_prime} without defining them, or includes prose explanations interleaved with code, preventing execution; translations begin with a reasonable output but then drift into multilingual tag lists (e.g., appending "Português:~...", "Spanish:~...") and truncate mid-word. EBFT avoids both failure modes, producing self-contained executable code and clean single-sentence translations that preserve the source meaning. These patterns suggest that the feature-matching objective encourages outputs that are both semantically faithful and cleanly formatted.

\section{Related Work}

Most language model training pipelines remain centered on next-token maximum likelihood (MLE), with reinforcement learning (RL) typically applied as a post-training step. RLHF-style methods optimize sequence-level rewards while regularizing toward a reference policy, often via a KL constraint \cite{christiano2017deep, ouyang2022training}, and preference-optimization approaches such as DPO can be interpreted as reward maximization under a similar regularization \cite{rafailov2023dpo}. Earlier sequence-level training methods likewise augment cross-entropy training with REINFORCE-style updates, but continue to rely on token-level supervision \cite{ranzato2016sequence, edunov2018classical}.

Recent work has explored using RL earlier in training or framing pretraining objectives in RL terms. RLP \cite{hatamizadeh2025rlp}, Reinforcement Pre-Training (RPT) \cite{dong2025rpt}, and RLPT \cite{li2025rlpt} introduce rewards tied to reasoning traces, information gain, or next-segment prediction. However, in all cases the reward signal is ultimately derived from next-token likelihood or correctness on the pretraining stream, rather than from a distinct semantic objective. Similarly, FlowRL proposes matching the full reward distribution to encourage diversity, but still defines rewards through likelihood-based or task-specific signals \cite{zhu2025flowrl}. 

Closely related are methods that derive reward signals from internal model representations rather than external verifiers. Generative Adversarial Post-Training (GAPT) employs a co-evolving discriminator to mitigate reward hacking in interactive generation \cite{wu2025gapt}. RARO \cite{cai2025escaping} uses a relativistic discriminator within an inverse reinforcement learning framework to recover implicit rewards from expert reasoning demonstrations. Concurrently, RLFR \cite{prasad2026features} trains lightweight probes on internal model activations to detect hallucinated claims and uses the probe output as a reward signal for reinforcement learning. While motivated by different applications, all three methods reduce rich representations to learned, task-specific scalar rewards. In contrast, our approach avoids learned reward models entirely, instead optimizing a fixed, vector-valued feature-matching objective that directly aligns rollout and data distributions in a general-purpose feature space.


Alternative generative frameworks aim to move beyond left-to-right likelihood training. Energy-Based Diffusion Language Models (EDLM) and related energy-based approaches operate at the sequence level \cite{xu2024edlm}, but focus on modeling the data distribution itself rather than defining a feature-space alignment objective for an autoregressive policy. Embedding-based similarity has been widely used for evaluation (e.g., BERTScore \cite{zhang2019bertscore}) and occasionally optimized via RL for metric-driven fine-tuning \cite{rennie2017selfcritical}, but not as a general replacement for teacher-forced token prediction.

In contrast, our approach decouples training from surface-form tokens entirely. We define rewards via a learned feature network and optimize an autoregressive policy using REINFORCE to match generated and ground-truth text in embedding space. This provides dense, semantic feedback that does not depend on next-token log loss, enabling sequence-level optimization that directly targets meaning rather than token reconstruction.
 
\section{Conclusion}
We introduced Energy-Based Fine-Tuning (EBFT), a method that 
fine-tunes language models by matching feature statistics of 
on-policy rollouts to those of ground-truth completions. Across 
Q\&A coding, unstructured coding, and translation, EBFT 
consistently outperforms SFT and matches RLVR on downstream 
accuracy, while achieving the best cross-entropy and 
feature-matching losses. Notably, EBFT reduces cross-entropy more 
than SFT despite not directly optimizing it. Unlike RLVR, EBFT 
requires no task-specific reward or verifier, making it applicable 
in non-verifiable settings where RLVR cannot be used. 

EBFT connects classical ideas from moment matching and 
distribution alignment with modern language model training. By 
operating in a feature space rather than over tokens or scalar 
rewards, it provides a flexible mechanism for shaping 
sequence-level behavior. However, EBFT is rollout-based and 
therefore slower per update than standard SFT, making it most 
suitable as a fine-tuning stage applied after cross-entropy 
training. It also requires a frozen feature network and has so far 
been evaluated on models up to 7B parameters with short rollout 
horizons. Scaling both axes, and exploring learned or adaptive 
feature networks, are promising directions for future work. More 
broadly, we view feature matching as a complementary training 
signal that may help bridge likelihood-based training and 
rollout-based optimization.

\bibliography{references}
\bibliographystyle{icml2026}

\clearpage

\appendix
\onecolumn
\part*{}
\etocsettocstyle{\section*{Appendix Contents}}{}
\etocsettocdepth{subsection}
\localtableofcontents
 
\section{The feature-matching loss profile and its optimal behavior}

Figures \ref{fig:gen_len_sft_fig} and \ref{fig:gen_len_fig} show the conditional feature-matching loss defined in \eqref{eq:L_CFM}, plotted against the completion length $G$. We refer to the function that maps $G$ to the corresponding (conditional) feature-matching loss value as the \emph{(conditional) feature-matching loss profile}.
To compute conditional feature-matching loss values, we extract ground-truth pairs $(c,y)$ from long ground-truth token sequences by \emph{selecting a strided set of prefixes of the sequence as the contexts $c$, and the ensuing windows of length $G$ as the completions $y$}. 
The bias-variance decomposition \eqref{eq:L_CFM} directly implies that the minimum value of $\mathcal{L}_{\mathrm{CFM}}$ is $\mathbb{E}_{c \sim p}\big[ \mathrm{Var}[\phi_c(y)|c] \big]$.
The following lemma shows that $\mathbb{E}_{c \sim p}\big[ \mathrm{Var}[\phi_c(y)|c] \big]$ is non-decreasing with the completion length $G$.

\begin{lemma}[The optimal conditional feature-matching profile] \label{lem:FM_profile_properties}
Consider the assumptions
\begin{enumerate}[label=(\alph*)]
  \item $\phi_c(y) := \phi(c\!:\!y)$ depends on $c$ and $y$ only through their concatenation $c\!:\!y$, which follows the construction in \Cref{subsec:fm_loss}.
  \item Context-completion pairs $(c,y)$ are selected/sampled from long ground-truth sequences as described above, such that for a fixed completion length $G$, concatenations $c:y$ are equally distributed to contexts $c$. Neglecting edge effects that stem from the long ground-truth sequences being finite, this holds for example if contexts are sampled as random prefixes of the ground-truth sequence. 
\end{enumerate}
Then, the optimal conditional feature-matching loss profile $\mathbb{E}_{c \sim p}\big[ \mathrm{Var}[\phi_c(y)|c] \big]$ is non-decreasing with $G$ and admits the bound
\begin{align} \label{eq:FM_bounds}
\mathbb{E}_{c \sim p}\big[ \mathrm{Var}[\phi_c(y)|c] \big]
\leq \mathrm{Var}_{c \sim p}[\phi(c)].
\end{align}
\end{lemma}
\begin{proof}
Consider completion lengths $1 \leq G' \leq G$. Let $y, \hat{y}$ denote sequences of length $G$, $y', \hat{y}'$ completions of length $G'$, and $y'', \hat{y}''$ completions of length $G'' = G - G'$, which means that we write the optimal conditional feature-matching loss at completion lengths $G'$ and $G$ as $\mathbb{E}_{c \sim p}\big[ \mathrm{Var}_{y' \sim p(\cdot|c)}[\phi(c\!:\!y')|c] \big]$ and $\mathbb{E}_{c \sim p}\big[ \mathrm{Var}_{y \sim p(\cdot|c)}[\phi(c\!:\!y)|c] \big]$, respectively. Observe that
\begin{align}
\begin{split} \label{eq:optimal_2}
    \mathbb{E}_{c \sim p}\big[ \mathrm{Var}_{y \sim p(\cdot|c)}[\phi_c(y)|c] \big] &= \mathbb{E}_{(c,y) \sim p} \Big[ \big\| \phi_c(y) \big\|^2\Big] - \mathbb{E}_{c \sim p} \Big[ \big\| \mathbb{E}_{y \sim p(\cdot|c)}[\phi_c(y)] \big\|^2\Big] \\ &= \mathbb{E}_{c \sim p} \Big[ \big\| \phi(c) \big\|^2\Big] - \mathbb{E}_{c \sim p} \Big[ \big\| \mathbb{E}_{y'' \sim p(\cdot|c)}\big[ \mathbb{E}_{y' \sim p(\cdot|c:y'')}[\phi(c\!:\!y''\!:\!y')] \big] \big\|^2\Big]
\end{split}    
\end{align}
Here, the second equality holds because $c:y$ is equally distributed to $c$ by Assumption \emph{(b)}, and using the tower property of expectation together with the decomposition $y = y''\!:\!y'$.
By Jensen's inequality, we have that 
\begin{align}
\begin{split}
 &\mathbb{E}_{c \sim p} \Big[ \big\| \mathbb{E}_{y'' \sim p(\cdot|c)}\big[ \mathbb{E}_{y' \sim p(\cdot|c:y'')}[\phi(c\!:\!y''\!:\!y')] \big] \big\|^2\Big] \leq \mathbb{E}_{c \sim p} \Big[ \mathbb{E}_{y'' \sim p(\cdot|c)}\Big[ \big\| \mathbb{E}_{y' \sim p(\cdot|c:y'')}[\phi(c\!:\!y''\!:\!y')] \big\|^2 \Big] \Big] \\ &= \mathbb{E}_{(c,y'') \sim p}\Big[ \big\| \mathbb{E}_{y' \sim p(\cdot|c:y'')}[\phi(c\!:\!y''\!:\!y')] \big\|^2 \Big] = \mathbb{E}_{c \sim p}\Big[ \big\| \mathbb{E}_{y' \sim p(\cdot|c)}[\phi(c\!:\!y')] \big\|^2 \Big] 
\end{split}
\end{align}
Plugging this back into the right-hand side of \eqref{eq:optimal_2} yields
\begin{align}
\begin{split}
    \mathbb{E}_{c \sim p}\big[ \mathrm{Var}_{y \sim p(\cdot|c)}[\phi_c(y)|c] \big] &\geq \mathbb{E}_{c \sim p} \Big[ \big\| \phi(c) \big\|^2\Big] - \mathbb{E}_{c \sim p}\Big[ \big\| \mathbb{E}_{y' \sim p(\cdot|c)}[\phi(c\!:\!y')] \big\|^2 \Big] \\ &= \mathbb{E}_{(c,y) \sim p} \Big[ \big\| \phi_c(y) \big\|^2\Big] - \mathbb{E}_{c \sim p}\Big[ \big\| \mathbb{E}_{y' \sim p(\cdot|c)}[\phi(c\!:\!y')] \big\|^2 \Big] \\ &= \mathbb{E}_{c \sim p}\big[ \mathrm{Var}_{y' \sim p(\cdot|c)}[\phi_c(y')|c] \big],
\end{split}
\end{align}
which concludes the proof that the optimal conditional feature-matching loss is non-decreasing with the completion length $G$. To prove the bound \eqref{eq:FM_bounds}, we apply Jensen's inequality in the opposite direction:
\begin{align}
    \mathbb{E}_{c \sim p} \Big[ \big\| \mathbb{E}_{y \sim p(\cdot|c)}[\phi_c(y)] \big\|^2\Big] \geq \big\| \mathbb{E}_{(c,y) \sim p} [ \phi_c(y) ] \big\|^2,
\end{align}
and plugging this into \eqref{eq:optimal_2} yields
\begin{align}
\begin{split}
    &\mathbb{E}_{c \sim p}\big[ \mathrm{Var}_{y \sim p(\cdot|c)}[\phi_c(y)|c] \big] \leq \mathbb{E}_{(c,y) \sim p} \Big[ \big\| \phi_c(y) \big\|^2\Big] - \big\| \mathbb{E}_{(c,y) \sim p} [ \phi_c(y) ] \big\|^2 \\ &= \mathrm{Var}_{(c,y) \sim p}[\phi(c\!:\!y)] = \mathrm{Var}_{c \sim p}[\phi(c)]
\end{split}
\end{align}
\end{proof}

\section{Feature matching with whitening}
\label{sec:whitening}

This section motivates a \emph{whitened} variant of feature matching by connecting standard cross-entropy training to a local $\chi^2$ objective. Then, we relax a variational formulation of the $\chi^2$ divergence by restricting the function space to a generalized linear model space corresponding to a chosen feature map. The resulting optimization problem admits a closed form and corresponds to the feature-matching loss with whitening, which amounts to premultiplying the feature vectors by the inverse of the matrix of second moments. However, in practice we only have access to a low-rank approximation of this matrix, which means that we can only compute a pseudo-inverse. We describe different empirical loss variants that we tried, including the one that we used to obtain the results in the main paper.
 
\subsection{Relating cross-entropy training to a $\chi^2$ divergence objective}
Fix a context $c$ and consider completions $y\in\mathcal{V}^{G}$ (for some completion length $G$). Let $p(\cdot\mid c)$ denote the ground-truth conditional distribution over completions and $p_{\theta}(\cdot\mid c)$ the model distribution. Cross-entropy training minimizes the conditional KL divergence
\begin{align}
    D_{\mathrm{KL}}\!\big(p(\cdot\mid c)\,\|\,p_{\theta}(\cdot\mid c)\big)
    &= \sum_{y \in \mathcal{V}^{G}} p(y\mid c)\log \frac{p(y\mid c)}{p_{\theta}(y\mid c)}.
    \label{eq:kl_def_whitening}
\end{align}
Rewriting \eqref{eq:kl_def_whitening} as an expectation under $p_{\theta}$ gives
\begin{align}
    D_{\mathrm{KL}}\!\big(p(\cdot\mid c)\,\|\,p_{\theta}(\cdot\mid c)\big)
    &= \sum_{y \in \mathcal{V}^{G}} p_{\theta}(y\mid c)\,\frac{p(y\mid c)}{p_{\theta}(y\mid c)}\log \frac{p(y\mid c)}{p_{\theta}(y\mid c)}.
    \label{eq:kl_rewrite_whitening}
\end{align}
The first-order Taylor expansion of $x \mapsto x\log x$ around $x=1$ is
\begin{align}
    x\log x = (x-1) + \tfrac{1}{2}(x-1)^2 + O\big((x-1)^3\big).
    \label{eq:taylor_xlogx}
\end{align}
Plugging \eqref{eq:taylor_xlogx} into \eqref{eq:kl_rewrite_whitening} yields
\begin{align}
    \begin{split}
    D_{\mathrm{KL}}\!\big(p(\cdot\mid c)\,\|\,p_{\theta}(\cdot\mid c)\big)
    &= \tfrac{1}{2}\sum_{y\in\mathcal{V}^{G}} p_{\theta}(y\mid c)\Big(\tfrac{p(y\mid c)}{p_{\theta}(y\mid c)}-1\Big)^2 
    + \sum_{y\in\mathcal{V}^{G}} p_{\theta}(y\mid c)\,O\!\Big(\Big(\tfrac{p(y\mid c)}{p_{\theta}(y\mid c)}-1\Big)^3\Big) \\
    &= \tfrac{1}{2}\,D_{\chi^2}\!\big(p(\cdot\mid c)\,\|\,p_{\theta}(\cdot\mid c)\big)
    + \sum_{y\in\mathcal{V}^{G}} p_{\theta}(y\mid c)\,O\!\Big(\Big(\tfrac{p(y\mid c)}{p_{\theta}(y\mid c)}-1\Big)^3\Big),
    \end{split}
    \label{eq:kl_chi2_local}
\end{align}
where the $\chi^2$ divergence is
\begin{align}
    D_{\chi^2}\!\big(p(\cdot\mid c)\,\|\,p_{\theta}(\cdot\mid c)\big)
    := \sum_{y\in\mathcal{V}^{G}} \frac{\big(p(y\mid c)-p_{\theta}(y\mid c)\big)^2}{p_{\theta}(y\mid c)}
    = \mathbb{E}_{Y\sim p_{\theta}(\cdot\mid c)}\!\Big[\Big(\tfrac{p(Y\mid c)}{p_{\theta}(Y\mid c)}-1\Big)^2\Big].
    \label{eq:chi2_def}
\end{align}
When $p(\cdot\mid c)\approx p_{\theta}(\cdot\mid c)$ (so the ratio $p/p_{\theta}$ is close to $1$), the remainder term in \eqref{eq:kl_chi2_local} can be neglected, and we obtain the local approximation
\begin{align}
    D_{\mathrm{KL}}\!\big(p(\cdot\mid c)\,\|\,p_{\theta}(\cdot\mid c)\big)\approx \tfrac{1}{2}\,D_{\chi^2}\!\big(p(\cdot\mid c)\,\|\,p_{\theta}(\cdot\mid c)\big).
    \label{eq:kl_approx_chi2}
\end{align}

\subsection{Relaxing a variational formulation of the $\chi^2$ divergence to a linear feature class}

The representation \eqref{eq:chi2_def} makes explicit that $D_{\chi^2}$ corresponds to an $L^2$ discrepancy in the space of one-hot features. Next, we want to express the $\chi^2$ divergence, or rather a relaxation of it, in terms of generic feature maps. For that, we consider a variational representation of the $\chi^2$ divergence.

\begin{lemma}[Variational representation of the chi-squared divergence]
\label{lem:variational_chi_squared_whitening}
Let $P$ and $Q$ be probability measures on a measurable space $\mathcal{Y}$ such that $P\ll Q$, and write their Radon--Nikodym derivative as $r(y)=\frac{dP}{dQ}(y)$. Then
\begin{align} \label{eq:chi2_variational_whitening}
    D_{\chi^2}(P\|Q)
    = \sup_{f\in L^2(Q), \, \mathbb{E}_{Y \sim Q}[f(Y)] = 0}\Big\{\,2\,(\mathbb{E}_{Y\sim P}[f(Y)] - \mathbb{E}_{Y\sim Q}[f(Y)]) - \mathbb{E}_{Y\sim Q}[f(Y)^2]\,\Big\},
\end{align}
Moreover, the supremum is attained at $f^{\star}(y)=r(y)-1$.
\end{lemma}
\begin{proof}
Define $g(y)=\frac{dP}{dQ}(y)$. Note that $\mathbb{E}_{Q}[r(Y)]=1$ and
$\chi^{2}(P\|Q)=\mathbb{E}_{Q}\big[(g(Y)-1)^{2}\big]$.
For any $f\in L^{2}(Q)$,
\begin{align}
    2\,(\mathbb{E}_{P}[f(Y)] - \mathbb{E}_{Q}[f(Y)])-\mathbb{E}_{Q}[f(Y)^{2}]
    &= 2\,\mathbb{E}_{Q}[(g(Y)-1)f(Y)]-\mathbb{E}_{Q}[f(Y)^{2}]
    = \mathbb{E}_{Q}\!\big[2(g(Y) - 1)f(Y)-f(Y)^{2}\big].
    \label{eq:proof_start}
\end{align}
Completing the square pointwise gives 
$2 (g - 1)f - f^{2}=-(f-g+1)^{2}+(r-1)^{2}$, 
hence
\begin{align}
    2\,(\mathbb{E}_{P}[f(Y)] - \mathbb{E}_{Q}[f(Y)])-\mathbb{E}_{Q}[f(Y)^{2}]
    = \mathbb{E}_{Q}[(g(Y)-1)^{2}] - \mathbb{E}_{Q}[(f(Y)-g(Y)+1)^{2}]
    \le \mathbb{E}_{Q}[(g(Y)-1)^{2}],
    \label{eq:proof_bound}
\end{align}
with equality iff  
$f=g-1$. Therefore,
\begin{align}
    \sup_{f\in L^{2}(Q)}\Big\{
    2\,(\mathbb{E}_{P}[f(Y)] - \mathbb{E}_{Q}[f(Y)])-\mathbb{E}_{Q}[f(Y)^{2}]
    \Big\}
    = \mathbb{E}_{Q}[(g(Y) - 1)^{2}]
    = D_{\chi^{2}}(P\|Q).
    \label{eq:proof_conclude}
\end{align}
where the optimum is achieved at $f(x) = g(x) - 1$.
\end{proof}

We particularize Lemma \ref{lem:variational_chi_squared_whitening} in the language model setting. 
Let $\varphi:\mathcal{V}^{G}\to\mathbb{R}^{d}$ be a feature map over completions (in our setting, the natural choice is $\varphi(y)=\phi_c(y)$, the feature-network embedding of the concatenated sequence). Restricting the supremum in \eqref{eq:chi2_variational_whitening} to generalized linear model $f_{w}(y)=w^{\top} \varphi(y)$ yields the relaxation
\begin{align}
    D_{\chi^2}(P\|Q)
    \ge \sup_{w\in\mathbb{R}^{d}}
    \Big\{\,2\,\big(\mathbb{E}_{Y\sim P}\big[w^{\top} \varphi(Y)\big] - \mathbb{E}_{Y\sim Q}\big[w^{\top} \varphi(Y)\big] \big)
    - \mathbb{E}_{Y\sim Q}\big[(w^{\top} \varphi(Y))^2\big]\,\Big\}.
    \label{eq:chi2_variational_linear_whitening}
\end{align}
We can rewrite the expression in the supremum as follows:
\begin{align}
    2\,\big(\mathbb{E}_{Y\sim P}\big[w^{\top} \varphi(Y)\big] - \mathbb{E}_{Y\sim Q}\big[w^{\top} \varphi(Y)\big] \big)
    - \mathbb{E}_{Y\sim Q}\big[\big(w^{\top} \varphi(Y) \big)^2\big] = 2 w^{\top} (\mu_{P}-\mu_{Q}) - w^{\top}\Sigma_{Q}w,
\end{align}
where $\mu_{P}=\mathbb{E}_{Y\sim P}[\varphi(Y)]$, $\mu_{Q}=\mathbb{E}_{Y\sim Q}[\varphi(Y)]$ and $\Sigma_{Q} = \mathbb{E}_{Y\sim Q}\big[\varphi(Y)\varphi(Y)^{\top}\big]$.
When $\Sigma_{Q}$ is invertible, the supremum in \eqref{eq:chi2_variational_linear_whitening} is attained at
\begin{align}
    \hat{w}=\Sigma_{Q}^{-1}(\mu_{P}-\mu_{Q}),
    \label{eq:w_opt_whitening}
\end{align}
and the optimal value is 
\begin{align}
    \sup_{w}\{2 w^{\top} (\mu_{P}-\mu_{Q}) -w^{\top} \Sigma_{Q}w\}
    = (\mu_{P}-\mu_{Q})\,\Sigma_{Q}^{-1}(\mu_{P}-\mu_{Q}).
    \label{eq:mahalanobis_value}
\end{align}
Thus, assuming that $\Sigma_{p_{\theta}(\cdot|c)} = \mathbb{E}_{\hat{y} \sim p_{\theta}(\cdot|c)}\big[\phi_c(\hat{y})\phi_c(\hat{y})^{\top}\big]$ is invertible, we define the whitened feature matching loss as
\begin{align} 
\begin{split} \label{eq:L_WFM}
    \mathcal{L}_{\mathrm{WFM}}(\theta) \! &= \!  \mathbb{E}_{c \sim p}\big[ \big( \mathbb{E}_{\hat{y} \sim p_{\theta}(\cdot|c)}[\phi_c(\hat{y})] \! - \! \mathbb{E}_{y \sim p(\cdot|c)}[\phi_c(y)] \big)^{\top} \\&\qquad\qquad\quad \times \mathbb{E}_{\hat{y} \sim p_{\theta}(\cdot|c)}\big[\phi_c(\hat{y})\phi_c(\hat{y})^{\top}\big]^{-1} \big( \mathbb{E}_{\hat{y} \sim p_{\theta}(\cdot|c)}[\phi_c(\hat{y})] \! - \! \mathbb{E}_{y \sim p(\cdot|c)}[\phi_c(y)] \big) \big],
\end{split}
\end{align}
Observe that the dependence of $\mathcal{L}_{\mathrm{WFM}}(\theta)$ on $\theta$ is through $\mathbb{E}_{\hat{y} \sim p_{\theta}(\cdot|c)}[\phi_c(\hat{y})]$ as in $\mathcal{L}_{\mathrm{FM}}(\theta)$, but also through $\mathbb{E}_{\hat{y} \sim p_{\theta}(\cdot|c)}\big[\phi_c(\hat{y})\phi_c(\hat{y})^{\top}\big]^{-1}$. While applying the REINFORCE argument to estimate the gradient for the former can be done as in \Cref{sec:loss}, the approach breaks down for the latter. We decide to disregard the gradient with respect to $\mathbb{E}_{\hat{y} \sim p_{\theta}(\cdot|c)}\big[\phi_c(\hat{y})\phi_c(\hat{y})^{\top}\big]^{-1}$, which amounts to the following loss:
\begin{align} 
\begin{split} \label{eq:L_WFM_stop_grad}
    \mathcal{L}_{\mathrm{WFM}}(\theta) \! &= \!  \mathbb{E}_{c \sim p}\big[ \big( \mathbb{E}_{\hat{y} \sim p_{\theta}(\cdot|c)}[\phi_c(\hat{y})] \! - \! \mathbb{E}_{y \sim p(\cdot|c)}[\phi_c(y)] \big)^{\top} \\&\qquad\qquad\quad \times \mathrm{stopgrad}\big(\mathbb{E}_{\hat{y} \sim p_{\theta}(\cdot|c)}\big[\phi_c(\hat{y})\phi_c(\hat{y})^{\top}\big] \big)^{-1} \big( \mathbb{E}_{\hat{y} \sim p_{\theta}(\cdot|c)}[\phi_c(\hat{y})] \! - \! \mathbb{E}_{y \sim p(\cdot|c)}[\phi_c(y)] \big) \big],
\end{split}
\end{align}
\subsection{Dealing with non-invertible second moment matrices $\Sigma_{p_{\theta}(\cdot|c)}$: a first approach}
The matrix $\Sigma_{p_{\theta}(\cdot|c)} = \mathbb{E}_{\hat{y} \sim p_{\theta}(\cdot|c)}\big[\phi_c(\hat{y})\phi_c(\hat{y})^{\top}\big]$ and especially its empirical version $\hat{\Sigma}_{p_{\theta}(\cdot|c)} = \frac{1}{n} \! \sum_{j=1}^{n} \! \phi_c(\hat{y}_j) \phi_c(\hat{y}_j)^{\top}$ are often not invertible, in particular when the feature dimension $d$ is high and/or the number of samples per prompt $n$ is low. In particular, the rank of the empirical matrix is upper-bounded by the number of samples per prompt, meaning that it is never invertible when $n < d$, which is usually the case. Thus, we need to solve $\sup_{w}\{2 w^{\top} (\mu_{P}-\mu_{Q}) -w^{\top} \Sigma_{Q}w\}$ for general positive semidefinite $\Sigma_{Q}$. More generally, consider maximizing a functional of the form $f(w)=2\langle w,b\rangle-\langle w,\Sigma w\rangle$. Let $\Sigma=\sum_{i=1}^{r}\lambda_{i}u_{i}u_{i}^{\top}$ be an eigendecomposition, and write $w=\sum_{i}\alpha_{i}u_{i}$ and $b=\sum_{i}\beta_{i}u_{i}$. Then
\begin{align}
    f(w)=2\sum_{i}\alpha_{i}\beta_{i}-\sum_{i}\lambda_{i}\alpha_{i}^{2}.
    \label{eq:quad_decomp}
\end{align}
Splitting into nonzero and zero eigenvalues yields the dichotomy:
\begin{enumerate}[label=(\roman*)]
    \item If $b$ has any component in $\ker(\Sigma)$ (i.e., there exists $i$ with $\lambda_{i}=0$ and $\beta_{i}\neq 0$), then $\sup_{w} f(w)=+\infty$ and there is no maximizer in $\mathbb{R}^{d}$.
    \item If $b\perp\ker(\Sigma)$, i.e.\ $b\in\mathrm{Im}(\Sigma)$, then the supremum is finite and equals
    \begin{align}
        \sup_{w} f(w)=\sum_{\lambda_{i}>0}\frac{\beta_{i}^{2}}{\lambda_{i}}
        = b^{\top}\,\Sigma^{\dagger}b,
        \label{eq:pseudoinverse_value}
    \end{align}
    where $\Sigma^{\dagger}=\sum_{\lambda_{i}>0}\frac{1}{\lambda_{i}}u_{i}u_{i}^{\top}$ is the Moore--Penrose pseudoinverse.
\end{enumerate}
In our case, $b= \mu_P - \mu_Q$ and $\Sigma = \Sigma_Q$. While $\mu_Q \perp \ker(\Sigma_Q)$ by construction, in general $\mu_P$ will have non-zero components in $\ker(\Sigma)$, and in that case $\sup_{w}\{2 w^{\top} (\mu_{P}-\mu_{Q}) -w^{\top} \Sigma_{Q}w\} = +\infty$. This is not surprising given that when the inequality \eqref{eq:chi2_variational_linear_whitening} holds with equality, which is the case for one-hot feature maps, $\sup_{w}\{2 w^{\top} (\mu_{P}-\mu_{Q}) -w^{\top} \Sigma_{Q}w\} = \chi^2(P\|Q)$, and it is easy to see that $\chi^2(P\|Q) = +\infty$ when the support of $P$ is larger than the support of $Q$. To obtain a finite value, it is convenient to replace $\mu_{P}$ by the projection $\mathrm{Pr}_{\mathrm{Im}(\Sigma_Q)} \mu_P$. Then, by equation \eqref{eq:pseudoinverse_value}  we obtain
\begin{align}
    \sup_{w}\{2 w^{\top} (\mathrm{Pr}_{\mathrm{Im}(\Sigma_Q)} \mu_{P}-\mu_{Q}) -w^{\top} \Sigma_{Q}w\} &= (\mathrm{Pr}_{\mathrm{Im}(\Sigma_Q)} \mu_{P}-\mu_{Q})^{\top} \Sigma_Q^{\dagger} (\mathrm{Pr}_{\mathrm{Im}(\Sigma_Q)} \mu_{P}-\mu_{Q}) \\ &= (\mu_{P}-\mu_{Q})^{\top} \Sigma_Q^{\dagger} (\mu_{P}-\mu_{Q}),
\end{align}
where the last equality holds because $\ker(\Sigma_Q^{\dagger}) = \ker(\Sigma_Q)$, and that $\mu_{P} = \mathrm{Pr}_{\mathrm{Im}(\Sigma)} \mu_{P} + \mathrm{Pr}_{\ker(\Sigma)} \mu_{P}$. Hence, the following loss function is numerically robust: 
\begin{align} 
\begin{split} \label{eq:LWFM_2}
    \mathcal{L}_{\mathrm{WFM}}(\theta) \! &= \!  \mathbb{E}_{c \sim p}\big[ \big( \mathbb{E}_{\hat{y} \sim p_{\theta}(\cdot|c)}[\phi_c(\hat{y})] \! - \! \mathbb{E}_{y \sim p(\cdot|c)}[\phi_c(y)] \big)^{\top} \\&\qquad\qquad\quad \times \mathrm{stopgrad}\big(\mathbb{E}_{\hat{y} \sim p_{\theta}(\cdot|c)}\big[\phi_c(\hat{y})\phi_c(\hat{y})^{\top}\big] \big)^{\dagger} \big( \mathbb{E}_{\hat{y} \sim p_{\theta}(\cdot|c)}[\phi_c(\hat{y})] \! - \! \mathbb{E}_{y \sim p(\cdot|c)}[\phi_c(y)] \big) \big],
\end{split}
\end{align}
It is easy to compute the gradient of this loss through the framework of \Cref{sec:loss}; in the computation of the population reward $r(\hat{y},c)$ in \eqref{eq:r_haty_c} we simply replace the features $\phi_c(y)$ by the whitened features:
\begin{align}
    &\tilde{\phi}_c(y) = (\Sigma^{\dagger}_{p_{\theta}(\cdot|c)})^{1/2} \phi_c(y), \qquad \Sigma_{p_{\theta}(\cdot|c)} = \mathbb{E}_{\hat{y} \sim p_{\theta}(\cdot|c)}\big[\phi_c(\hat{y})\phi_c(\hat{y})^{\top}\big],
\end{align}
and in practice, we compute the reward $r_j$ in \eqref{eq:REINFORCE_practical} using $\hat{\Sigma}_{p_{\theta}(\cdot|c)}$ instead of $\Sigma_{p_{\theta}(\cdot|c)}$:
\begin{align}
    &\tilde{\phi}_c(y) = (\hat{\Sigma}^{\dagger}_{p_{\theta}(\cdot|c)})^{1/2} \phi_c(y), \qquad \hat{\Sigma}_{p_{\theta}(\cdot|c)} = \frac{1}{n} \sum_{j=1}^{n} \phi_c(\hat{y}_j)\phi_c(\hat{y}_j)^{\top}.
\end{align}
Above, $(\Sigma^{\dagger})^{1/2}$ denotes the square root of the pseudo-inverse of $\Sigma$, i.e. if $\Sigma$ admits the eigenvalue decomposition $\Sigma = \sum_{i=1}^d \lambda_i u_i u_i^{\top}$, then $(\Sigma^{\dagger})^{1/2} = \sum_{i=1, \lambda_i > 0}^d \lambda_i^{-1/2} u_i u_i^{\top}$. In practice, using a function to compute the singular value decomposition of  is more numerically stable that using a function to compute the eigenvalue decomposition.

\subsection{Variants of the whitened feature-matching loss with better empirical performance}
Let us write the reward $r_j$ explicitly under whitening:
\begin{align} \label{eq:r_j_whitening}
    r_j = \underbrace{2 \phi_c(\hat{y}_j)^{\top} \hat{\Sigma}^{\dagger}_{p_{\theta}(\cdot|c)} \phi_c(y)}_{\text{alignment term $\mathrm{AT}_j$}} - \underbrace{\frac{2}{n - 1} \sum_{j' \neq j} \phi_c(\hat{y}_j)^{\top} \hat{\Sigma}^{\dagger}_{p_{\theta}(\cdot|c)}  \phi_c(\hat{y}_{j'})}_{\text{diversity term $\mathrm{DT}_j$}}.
\end{align} 
Suppose that the features $\big(\phi_c(\hat{y}_j) \big)_{j=1}^{n}$ of the generated completions are ordered such that repeated completions are arranged consecutively, and that there are exactly $K$ different feature vectors among $\big(\phi_c(\hat{y}_j) \big)_{j=1}^{n}$, with multiplicities $(n_k)_{k=1}^{K}$, such that $\sum_{k=1}^{K} n_k = n$. In this section, we make the following assumptions, which hold in practice: 
\begin{enumerate}[label=(\roman*)]
\item The feature dimension $d$ is larger or equal than the number of generated completions $n$. This holds in our experiments, because $d$ is on the order of thousands, while we take $n=4$.
\item The $K$ different feature vectors within $(\phi_c(\hat{y}_j))_{j=1}^{n}$ are linearly independent. This happens with very high probability in our experiments, also as a consequence of $d \gg n$.
\end{enumerate}
For $1 \leq k \leq K$, let $j_k = \sum_{k'=1}^{k-1} n_{k'} + 1$. Hence, the list of feature vectors $\big(\phi_c(\hat{y}_{j_k}) \big)_{k=1}^{K}$ contains each instance in $\big(\phi_c(\hat{y}_j) \big)_{j=1}^{n}$ with multiplicity one. We define the matrices
\begin{itemize}
\item $\Phi \in \mathbb{R}^{d \times n}$ as the matrix whose columns are $(\phi_c(\hat{y}_j))_{j=1}^{n}$, i.e. $\Phi_{\cdot j} = \phi_c(\hat{y}_j)$,
\item $\hat{\Phi} \in \mathbb{R}^{d \times K}$ as the matrix whose columns are $\big( \phi_c(\hat{y}_{j_k}) \big)_{k=1}^{K}$,
\item $\bar{\Phi} \in \mathbb{R}^{d \times K}$ as the matrix whose columns are $\big(\sqrt{n_k} \phi_c(\hat{y}_{j_k}) \big)_{k=1}^{K}$,
\item $\tilde{\Phi} = ((\Phi \Phi^{\top})^{\dagger})^{1/2} \Phi \in \mathbb{R}^{d \times n}$,
\end{itemize}
And the vectors
\begin{itemize}
\item $\psi = \phi_c(y) \in \mathbb{R}^d$, 
\item $\tilde{\psi} = ((\Phi \Phi^{\top})^{\dagger})^{1/2} \psi \in \mathbb{R}^{d}$,
\item $x^{(\psi)} = 
\hat{\Phi}^{\dagger} \psi \in \mathbb{R}^{K}$,
which is the vector of coefficients of the orthogonal projection of $\psi$ onto $\mathrm{span}\big(\big(\phi_c(\hat{y}_{j_k}) \big)_{k=1}^{K} \big)^{\perp}$ with respect to the basis $\big(\phi_c(\hat{y}_{j_k}) \big)_{k=1}^{K}$,
\end{itemize}
We can reexpress the alignment and diversity terms in \eqref{eq:r_j_whitening} with respect to $\tilde{\Phi}$ and $\tilde{\psi}$:
\begin{align} 
    \mathrm{AT}_j &= 2 n \phi_c(\hat{y}_j)^{\top} (\Phi \Phi^{\top})^{\dagger} \phi_c(y) = 2 n \tilde{\Phi}_{\cdot j}^{\top} \tilde{\psi}, \\
    \label{eq:diversity_term_tilde_phi}
    \mathrm{DT}_j &= \frac{2 n}{n - 1} \sum_{j' \neq j} \phi_c(\hat{y}_j)^{\top} (\Phi \Phi^{\top})^{\dagger}  \phi_c(\hat{y}_{j'}) = \frac{2n}{n - 1} \sum_{j' \neq j} \tilde{\Phi}^{\top}_{\cdot j} \tilde{\Phi}_{\cdot j'}. 
\end{align}
The following lemma, proven in \Cref{subsubsec:proof_lem_tilde_phi}, characterizes the inner products $\tilde{\Phi}^{\top}_{\cdot j} \tilde{\Phi}_{\cdot j'}$ and $\tilde{\Phi}_{\cdot j'}^{\top} \tilde{\psi}$, and the norm of $\tilde{\psi}$.
\begin{lemma} \label{lem:tilde_phi}
    Recall that $n_{k_j}$ is the multiplicity of the completion $\hat{y}_j$ within $\big(\phi_c(\hat{y}_j) \big)_{j=1}^{n}$. The inner products between the columns $(\tilde{\Phi}_{\cdot j})_{j=1}^{n}$ are given by
    \begin{align} \label{eq:orthogonal_1}
    \tilde{\Phi}_{\cdot j}^{\top} \tilde{\Phi}_{\cdot j'} &= 1/n_{k_j} \quad &&\text{for all } j' \text{ such that } \phi_c(\hat{y}_j) = \phi_c(\hat{y}_{j'}) \text{ (in particular } \|\tilde{\Phi}_{\cdot j}\| = 1/\sqrt{n_{k_j}}\text{)}, \\
    \tilde{\Phi}_{\cdot j}^{\top} \tilde{\Phi}_{\cdot j'} &= 0 \quad &&\text{for all } j' \text{ such that } \hat{y}_j \neq \hat{y}_{j'}. \label{eq:orthogonal_2}
    \end{align}
    And we have that
    \begin{align}
        \tilde{\Phi}_{\cdot j}^{\top} \tilde{\psi} &= \frac{x^{(\psi)}_j}{n_{k_j}}, \qquad \text{for all $1 \leq j \leq n$, and}, \qquad
       \|\tilde{\psi}\|^2 = \sum_{k=1}^{K} \frac{(x^{(\psi)}_k)^2}{n_k}.
    \end{align}
\end{lemma}

Thus, under the conditions of Lemma \ref{lem:tilde_phi}, 
\begin{align} \label{eq:AT_basic}
    \mathrm{AT}_j &= \frac{2 n x^{(\psi)}_j}{n_{k_j}}, \\
    \mathrm{DT}_j &= \frac{2 n}{n - 1} \sum_{j' \neq j} \mathbf{1}[\phi_c(\hat{y}_j) = \phi_c(\hat{y}_{j'})] \frac{1}{n_{k_j}} = \frac{2n(n_{k_j} - 1)}{(n - 1)n_{k_j}} = \frac{2(1 - 1/n_{k_j})}{1 - 1/n}. \label{eq:DT_basic}
\end{align}
where $x^{(\psi)}_j$ is the $j$-th component of $x^{(\psi)}$. 

Observe that when $\phi_c(y)$ is equal to $\phi_c(\hat{y}_{j_k})$ for some $1 \leq k \leq K$, $x^{(\psi)}$ is the $k$-th vector of the canonical basis of $\mathbb{R}^{K}$, which means that $\mathrm{AT}_j = \frac{2 n}{n_{k_j}}$ for all $j$ such that $\phi_c(y)$ is equal to $\phi_c(\hat{y}_{j})$, and zero otherwise.

When $\phi_c(y)$ is different from all the vectors in $\big(\phi_c(\hat{y}_{j_k}) \big)_{k=1}^{K}$, the norm of the projection of $\phi_c(y)$ onto $\mathrm{span}\big(\big(\phi_c(\hat{y}_{j_k}) \big)_{k=1}^{K}\big)$ is usually significantly smaller than the norm of $\phi_c(y)$, because this subspace is smaller than the ambient dimension as $K \leq n \leq d$, and this is accentuated the smaller $n$ is. Observe that in optimizing the empirical rewards $r_j = \mathrm{AT}_j - \mathrm{DT}_j$, the model balances increasing the alignment term $\mathrm{AT}_j$ and decreasing the diversity term $\mathrm{DT}_j$, and the specific trade-off is determined by the relative sizes of both terms.
Since for small $n$ the alignment terms $(\mathrm{AT}_j)_{j=1}^{n}$ are small because $x^{(\psi)}$ is abnormally small, 
the model focuses on decreasing $\mathrm{DT}_j$ as opposed to strongly improving $\mathrm{AT}_j$. Experimentally, this translates to moderate improvements in downstream performance. To achieve more solid boosts in downstream performance, we tested the following alternative approaches:
\begin{itemize} 
    \item \textit{Variant (i)}: \textbf{Normalizing the whitened features of the generations and the ground truth in the alignment term}. As a result, even when $x^{(\psi)}$ is small, the alignment term is not. Namely, we set the diversity term $\mathrm{DT}_j$ as in \eqref{eq:DT_basic}, and the alignment term $\mathrm{AT}_j$ as follows:
    \begin{align}
        \mathrm{AT}_j &= \frac{2 \phi_c(\hat{y}_j)^{\top} \hat{\Sigma}^{\dagger}_{p_{\theta}(\cdot|c)} \phi_c(y)}{\big\| (\hat{\Sigma}^{\dagger}_{p_{\theta}(\cdot|c)})^{1/2} \phi_c(\hat{y}_j) \big\| \big\| (\hat{\Sigma}^{\dagger}_{p_{\theta}(\cdot|c)})^{1/2} \phi_c(y) \big\|} = \frac{2 \tilde{\Phi}_{\cdot j}^{\top} \tilde{\psi}}{\big\|\tilde{\Phi}_{\cdot j} \big\| \big\|\tilde{\psi} \big\|} 
        = \frac{2 x^{(\psi)}_j}{\sqrt{ n_{k_j} \sum_{k=1}^{K} \frac{(x^{(\psi)}_k)^2}{n_k}}}. \label{eq:AT_normalized}
    \end{align}
    Observe that we can write the vector of alignment terms $(\mathrm{AT}_j)_{j=1}^{n}$ as 
    \begin{align}
    (\mathrm{AT}_j)_{j=1}^{n} = \frac{2 (x^{(\psi)}_j/\sqrt{n_{j_k}})_{j=1}^{n} }{ \|(x^{(\psi)}_j/\sqrt{n_{j_k}})_{j=1}^{n}\|}.
    \end{align}
    \item \textit{Variant (ii)}: \textbf{Normalizing the whitened features of the generations and the ground truth in the alignment and diversity terms}. 
    We set $\mathrm{AT}_j$ as in \eqref{eq:AT_normalized}, and $\mathrm{DT}_j$ as follows:
    \begin{align} 
        \begin{split}
        \mathrm{DT}_j &= \frac{2}{n - 1} \sum_{j' \neq j} \frac{\phi_c(\hat{y}_j)^{\top} \hat{\Sigma}^{\dagger}_{p_{\theta}(\cdot|c)}  \phi_c(\hat{y}_{j'})}{\big\| (\hat{\Sigma}^{\dagger}_{p_{\theta}(\cdot|c)})^{1/2} \phi_c(\hat{y}_{j}) \big\| \big\| (\hat{\Sigma}^{\dagger}_{p_{\theta}(\cdot|c)})^{1/2} \phi_c(\hat{y}_{j'}) \big\|} = \frac{2}{n - 1} \sum_{j' \neq j} \frac{\tilde{\Phi}^{\top}_{\cdot j} \tilde{\Phi}_{\cdot j'}}{\big\|\tilde{\Phi}_{\cdot j} \big\| \big\|\tilde{\Phi}_{\cdot j'} \big\|} \\ &= \frac{2}{n - 1} \sum_{j' \neq j} \mathbf{1}[\phi_c(\hat{y}_j) = \phi_c(\hat{y}_{j'})] \frac{\sqrt{n_{k_j} n_{k_{j'}}}}{n_{k_j}} = \frac{2 (n_{k_j} - 1)}{n - 1}. 
        \end{split}
    \end{align}
    \item \textit{Variant (iii)}: \textbf{Normalizing the whitened features of the ground truth in the alignment term}. We set $\mathrm{DT}_j$ as in \eqref{eq:DT_basic} and $\mathrm{AT}_j$ as follows:
    \begin{align}
        \mathrm{AT}_j &= \frac{2 \phi_c(\hat{y}_j)^{\top} \hat{\Sigma}^{\dagger}_{p_{\theta}(\cdot|c)} \phi_c(y)}{\big\| (\hat{\Sigma}^{\dagger}_{p_{\theta}(\cdot|c)})^{1/2} \phi_c(y) \big\|} = \frac{2 \tilde{\Phi}_{\cdot j'}^{\top} \tilde{\psi}}{ \big\|\tilde{\psi} \big\|} 
        = \frac{2 x^{(\psi)}_j}{n_{k_j}\cdot \frac{1}{\sqrt{n_{k_j}}} \cdot \sqrt{\sum_{k=1}^{K} \frac{(x^{(\psi)}_k)^2}{n_k}}} 
        = \frac{2x^{(\psi)}_j}{n_{k_j} \sqrt{\sum_{k=1}^{K} \frac{(x^{(\psi)}_k)^2}{n_k}}}, 
    \end{align}
\end{itemize}
Experimentally, \textbf{variant (i) offers the best performance, and is the one that we use in all the whitening experiments we report in this paper.} 
 
Observe that when we use whitening (with or without any of these variants), we are not explicitly minimizing a particular loss function on $\theta$, as our REINFORCE-style reward does not take into account that the second-moment matrix $\hat{\Phi}_{p_{\theta}(\cdot|c)}$ and the normalization factors depend on $\theta$. In the figures of \Cref{sec:add_experimental_results}, apart from the non-whitened feature matching loss, we report the proxy quantity:
\begin{align}
    \frac{1}{n} \sum_{j=1}^{n} \big( \mathrm{AT}_j - \frac{1}{2} \mathrm{DT}_j \big)
\end{align}
We refer to this quantity as the "feature-matching loss with whitening". Ignoring the dependence of the second-moment matrix and the normalization constants on $\theta$, we view whitened feature matching as trying to decrease this loss.

\subsubsection{Proof of Lemma \ref{lem:tilde_phi}}
\label{subsubsec:proof_lem_tilde_phi}
And we define the matrix $B \in \mathbb{R}^{K \times n}$ as the matrix whose $k$-th row has value $1/\sqrt{n_k}$ on all positions from $j_k = \sum_{k'=1}^{k-1} n_{k'} + 1$ to $j_{k+1} - 1 = \sum_{k'=1}^{k} n_{k'}$ (both included), and value zero on all remaining positions. 
Observe that the rows of $B$ constitute an orthonormal set, i.e. $B B^{\top} = \mathrm{Id} \in \mathbb{R}^{K \times K}$, and that by construction,
\begin{align}
    \Phi = \bar{\Phi} B.
\end{align}
Thus,
\begin{align} \label{eq:tildePhi_reexpressed}
    \tilde{\Phi} = ((\bar{\Phi} B B^{\top} \bar{\Phi}^{\top})^{\dagger})^{1/2} \bar{\Phi} B = ((\bar{\Phi} \bar{\Phi}^{\top})^{\dagger})^{1/2} \bar{\Phi} B.
\end{align}
Next, we inspect $((\bar{\Phi} \bar{\Phi}^{\top})^{\dagger})^{1/2} \bar{\Phi}$. Observe that by assumptions \emph{(i)} and \emph{(ii)} above, the rank of $\bar{\Phi}$ is $K$. Let $\bar{\Phi} = U \Sigma V^{\top}$ be the thin singular value decomposition of $\bar{\Phi}$, i.e. $U \in \mathbb{R}^{d \times K}$ has orthonormal columns, $V \in \mathbb{R}^{K \times K}$ is an orthogonal matrix, and $\Sigma \in \mathbb{R}^{K \times K}$ is a diagonal matrix with strictly positive numbers on the diagonal (in general the singular values are non-negative, but since $\bar{\Phi}$ has rank $K$, all of them must be positive). Then, by the definitions of the square root and the pseudo-inverse
\begin{align} \label{eq:pseudo_inverse_simplification}
((\bar{\Phi} \bar{\Phi}^{\top})^{\dagger})^{1/2} \bar{\Phi} = 
((U \Sigma^2 U^{\top})^{\dagger})^{1/2} U \Sigma V^{\top} = U \Sigma^{-1} U^{\top} U \Sigma V^{\top} = U V^{\top}
\end{align}
Plugging this into the right-hand side of \eqref{eq:tildePhi_reexpressed} yields
\begin{align} \label{eq:tilde_phi_simplified}
    \tilde{\Phi} = U V^{\top} B.
\end{align}
And 
\begin{align}
    \tilde{\Phi}^{\top} \tilde{\Phi} = B^{\top} V U^{\top} U V^{\top} B = B^{\top} B, 
\end{align}
which means that for $j, j' \in [n]$, $\tilde{\Phi}_{\cdot j}^{\top} \tilde{\Phi}_{\cdot j'}$, which is the $(j,j')$-th component of $\tilde{\Phi}^{\top} \tilde{\Phi}$, is equal to the $(j,j')$-th component of $B^{\top} B$, which is $B_{\cdot j}^{\top} B_{\cdot j'}$. Since $B_{\cdot j}$, (resp. $B_{\cdot j'}$) has a single non-zero entry $1/\sqrt{n_{k_j}}$ in position $k_j$ (resp. $1/\sqrt{n_{k_{j'}}}$ in position $k_{j'}$), equalities \eqref{eq:orthogonal_1} and \eqref{eq:orthogonal_2} follow. 

Let us define the diagonal matrix $\bar{B} \in \mathbb{R}^{K \times K}$ with diagonal values $\big(1/\sqrt{n_k} \big)_{k=1}^{K}$, such that we can express the matrix $\hat{\Phi}$ with columns $\big(\phi_c(\hat{y}_{j_k}) \big)_{k=1}^{K}$ as $\bar{\Phi} \bar{B}$. By construction of $x^{(\psi)}$, we have the following: 
\begin{align}
    \psi = \hat{\Phi} x^{(\psi)} + \psi' = \bar{\Phi} \bar{B} x^{(\psi)} + \psi',
\end{align}
where $\psi'$ is the orthogonal projection of $\psi$ onto $\mathrm{span}\big(\big(\phi_c(\hat{y}_{j_k}) \big)_{k=1}^{K} \big)^{\perp}$.
Hence, using equation \eqref{eq:pseudo_inverse_simplification}, and the fact that $\mathrm{span}( (U_{\cdot k})_{k=1}^{K}) = \mathrm{span}\big(\big(\phi_c(\hat{y}_{j_k}) \big)_{k=1}^{K} \big)$,
\begin{align} \label{eq:psi_product_simplified}
    ((\Phi \Phi^{\top})^{\dagger})^{1/2} \psi = ((\bar{\Phi} \bar{\Phi}^{\top})^{\dagger})^{1/2} \big( \bar{\Phi} \bar{B}  x^{(\psi)} + \psi' \big) =  U V^{\top} \bar{B} x^{(\psi)} + U \Sigma^{-1} U^{\top} \psi' = U V^{\top} \bar{B} x^{(\psi)}.
\end{align}
Hence, using \eqref{eq:tilde_phi_simplified} and \eqref{eq:psi_product_simplified} yields
\begin{align}
    \Phi^{\top} ((\Phi \Phi^{\top})^{\dagger})^{1/2} ((\Phi \Phi^{\top})^{\dagger})^{1/2} \psi = B V U^{\top} U V^{\top} \bar{B} x^{(\psi)} = B^{\top} \bar{B} x^{(\psi)} = \tilde{B} x^{(\psi)}
\end{align}
where we define $\tilde{B} \in \mathbb{R}^{K \times n}$ as the matrix whose $k$-th row has value $1/n_k$ on all positions from $j_k = \sum_{k'=1}^{k-1} n_{k'} + 1$ to $j_{k+1} - 1 = \sum_{k'=1}^{k} n_{k'}$ (both included), and value zero on all remaining positions.
And
\begin{align}
    \psi^{\top} ((\Phi \Phi^{\top})^{\dagger})^{1/2} ((\Phi \Phi^{\top})^{\dagger})^{1/2} \psi = (x^{(\psi)})^{\top} \bar{B}^{\top} V U^{\top} U V^{\top} \bar{B} x^{(\psi)} = \|\bar{B} x^{(\psi)}\|^2 = \sum_{k=1}^{K} \frac{(x_k^{(\psi)})^2}{n_k}. 
\end{align}

\section{Feature matching with alignment bias}
\label{sec:alignment_bias}
In some applications we prefer \emph{more reference-aligned} samples, even at the cost of reduced diversity.
We capture this tradeoff by scaling the target moment: for $\alpha\in(0,1]$, define the loss with alignment bias as
\begin{align} 
\begin{split} \label{eq:L_FM_alpha}
    &\mathcal{L}^{\alpha}_{\mathrm{FM}}(\theta) \! = \! \alpha \mathbb{E}_{c \sim p}\big[ \| \mathbb{E}_{\hat{y} \sim p_{\theta}(\cdot|c)}[\phi_c(\hat{y})] \! - \! {\textstyle \frac{1}{\alpha}} \mathbb{E}_{y \sim p(\cdot|c)}[\phi_c(y)] \|^2 \big] \\ & = \! \alpha \mathcal{L}_{\mathrm{FM}}(\theta) \! - \! \underbrace{
    2 (1-\alpha) 
    \mathbb{E}_{c \sim p}\big[ \mathbb{E}_{\hat{y} \sim p_{\theta}(\cdot|c)}[\phi_c(\hat{y})]^{\top} \! \mathbb{E}_{y \sim p(\cdot|c)}[\phi_c(y)] \big].}_{\text{alignment bias}}
\end{split}
\end{align}
The additional term explicitly encourages alignment between model and data features, making the objective \emph{mode-seeking} as $\alpha$ decreases (typically improving accuracy and faithfulness but reducing diversity).
Operationally, this corresponds to multiplying the diversity terms in \eqref{eq:r_haty_c} and \eqref{eq:REINFORCE_practical} by $\alpha$, since the analogs of equations 
\eqref{eq:L_CFM}
and \eqref{eq:L_FM_surrogate}, \eqref{eq:r_haty_c}, and \eqref{eq:REINFORCE_practical} are
\begin{align}
\begin{split} \label{eq:moment_matching_L_FM_alpha}
    \mathcal{L}^{\alpha}_{\mathrm{CFM}}(\theta)
    &= \alpha \mathbb{E}_{c \sim p}\big[ \| \mathbb{E}_{\hat{y} \sim p_{\theta}(\cdot|c)}[\phi_c(\hat{y})] \! - \! {\textstyle \frac{1}{\alpha}} \phi_c(y) \|^2 \big],
\end{split} \\
\begin{split}
    \mathcal{L}^{\alpha}_{\mathrm{CFM}}(\theta;c,y)
    &= \alpha \| \mathbb{E}_{\hat{y} \sim p_{\theta}(\cdot|c)}[\phi_c(\hat{y})] \! - \! {\textstyle \frac{1}{\alpha}} \phi_c(y) \|^2,
\end{split} \label{eq:L_FM_surrogate_alpha} \\
\begin{split} \label{eq:r_haty_c_alpha}
    r(\hat{y},c) &= 2 \phi_c(\hat{y})^{\top} \phi_c(y) - 2 \alpha \phi_c(\hat{y})^{\top} \mathbb{E}_{\tilde{y} \sim p_{\theta}(\cdot|c)}\big[ \phi_c(\tilde{y}) \big],
\end{split} \\
\begin{split} \label{eq:REINFORCE_practical_alpha}
r_j &= 2 \phi_c(\hat{y}_j)^{\top} \phi_c(y) - \frac{2 \alpha}{n - 1} \sum_{j' \neq j} \phi_c(\hat{y}_j)^{\top} \phi_c(\hat{y}_{j'}).
\end{split}
\end{align}
Note that when $\alpha \neq 1$, $\mathcal{L}^{\alpha}_{\mathrm{FM}}$ is not a proper scoring rule, meaning that that it is not minimized by the ground truth distribution $p$, and this can be seen in practice. In \Cref{app:alpha_gamma_sweep} we present experiments in which we sweep over $\alpha$ and $\gamma$ on all the tasks we consider: we consider the values $\alpha \in \{0, 0.5, 1\}$. We conclude that taking $\alpha$ smaller helps in reducing the (unbiased) feature-matching loss faster, at the cost of a slower decrease of the CE loss, with stark differences in behavior depending on the value of $\gamma$. In particular, when $\gamma = 0.1$, the CE loss decreases similarly fast for $\alpha \in \{0, 0.5, 1\}$, but when $\gamma = 0$, taking $\alpha \in \{0, 0.5\}$ causes the CE loss to diverge, while for $\alpha=1$ the CE loss still decreases monotonically, albeit at a slower rate than with higher $\alpha$. Hence, the CE term can help in stabilizing feature matching with alignment bias.

\section{Feature matching with KL regularization}
\label{sec:KL_regularization}

The feature matching loss with KL regularization with respect to a model $\pi$ reads
\begin{align} 
\begin{split}
    \mathcal{L}_{\mathrm{FMKL}}(\theta) \! &= \!  \mathbb{E}_{c \sim p}\big[ \| \mathbb{E}_{\hat{y} \sim p_{\theta}(\cdot|c)}[\phi_c(\hat{y})] \! - \! {\textstyle \frac{1}{\alpha}} \mathbb{E}_{y \sim p(\cdot|c)}[\phi_c(y)] \|^2 + {\textstyle \frac{1}{\beta}} D_{\mathrm{KL}}\big( p_{\theta}(\cdot|c) \| \pi(\cdot|c) \big) \big] \\ &= \!  \mathbb{E}_{c \sim p}\big[ \| \mathbb{E}_{\hat{y} \sim p_{\theta}(\cdot|c)}[\phi_c(\hat{y})] \! - \! {\textstyle \frac{1}{\alpha}} \mathbb{E}_{y \sim p(\cdot|c)}[\phi_c(y)] \|^2 + {\textstyle \frac{1}{\beta}} \mathbb{E}_{\hat{y} \sim p_{\theta}(\cdot|c)} \big[ \log p_{\theta}(\hat{y}|c) - \log \pi(\hat{y}|c) \big] \big]
\end{split}
\end{align}
Observe that since this loss function decouples across different contexts $c$, the optimal $p_{\theta}$ satisfies the following for all $c$:
\begin{align} \label{eq:p_theta_c}
    p_{\theta}(\cdot|c) = \arg\min_{\rho \in \mathcal{V}^G} \big\{ {\textstyle \frac{1}{\beta}} D_{\mathrm{KL}}\big(\rho \| \pi(\cdot|c) \big) 
    + \| \mathbb{E}_{y \sim \rho} [\phi_c(y)] - {\textstyle \frac{1}{\alpha}} \mathbb{E}_{y \sim p(\cdot|c)}[\phi_c(y)] \|_2^2 \big\}.
\end{align}
Next, we will characterize the distribution $p_{\theta}$ that minimizes this loss function.
The following section contains some preliminary results.

\subsection{Energy-based models with RKHS function classes}
The following duality theorem relates two optimization problems which underlie energy-based models for which the energy class is a ball of the RKHS induced by a feature map $\phi$.
\begin{theorem}[Thm.~2, \citet{domingoenrich2022dual}] \label{eq:L_2_duality}
    Let $\pi$ be a base measure over a measurable space $\mathcal{Y}$, and $\tilde{\beta} > 0$. Let $\varphi : \mathcal{Y} \to \mathbb{R}^d$ for some $d \geq 1$ be a feature map, and $v \in \mathbb{R}^d$.
    Consider the two problems
    \begin{align} \label{eq:min_pb}
        \min_{\rho \in \mathcal{P}(\mathcal{Y})} 
        {\textstyle \frac{1}{\tilde{\beta}}} D_{\mathrm{KL}}(\rho \| \pi) 
        + \| \mathbb{E}_{Y \sim \rho} [\varphi(Y)] - v \|_2,
    \end{align}
    and
    \begin{align} \label{eq:max_pb}
        \max_{\substack{h \in \mathbb{R}^d \\ \|h\|_{2} \leq 1}}
        - v^{\top} h
        - {\textstyle \frac{1}{\tilde{\beta}}} 
        \log \mathbb{E}_{Y \sim \pi} \big[ \exp \big( - \tilde{\beta} 
        \varphi(Y)^{\top} h
        \big) \big].
    \end{align}
    The problems \eqref{eq:min_pb} and \eqref{eq:max_pb} are convex. 
    The
    problem \eqref{eq:max_pb} is the Fenchel dual of problem \eqref{eq:min_pb}, and strong duality holds. Moreover, the solution
    \(\rho^\star\) of \eqref{eq:min_pb} is unique and satisfies
    \begin{align}
        \frac{d\rho^\star}{d\pi}(y) = \frac{1}{Z_{\tilde{\beta}}} 
        \exp \big( - \tilde{\beta} \varphi(y)^{\top} h^\star \big),
    \end{align}
    where \(h^\star\) is a solution of \eqref{eq:max_pb} and \(Z_\beta\) is a normalization constant.
\end{theorem}
And the following equivalence between minimization problems holds:
\begin{theorem}[Prop.~3, \citet{domingoenrich2022dual}] \label{eq:dual_equivalence}
Consider the problem
\begin{align} \label{eq:min_pb_2}
    \min_{\rho \in \mathcal{P}(\mathcal{Y})} 
        \beta^{-1} D_{\mathrm{KL}}(\rho \| \pi) 
        + \| \mathbb{E}_{Y \sim \rho} [\varphi(Y)] - v \|_2^2,
\end{align}
Problems \eqref{eq:min_pb} and \eqref{eq:min_pb_2} are equivalent in the following sense:  
if \(\rho_1^\star\) is a solution of \eqref{eq:min_pb} for \(\tilde{\beta}\), then it is also a solution of \eqref{eq:min_pb_2} for  
\begin{align}
    \beta = 
    \big( 2 \| \mathbb{E}_{Y \sim \rho_1^{\star}} [\varphi(Y)] - v \|_2 \big)^{-1} \tilde{\beta},
\end{align}
provided that $\| \mathbb{E}_{Y \sim \rho_1^{\star}} [\varphi(Y)] - v \|_2$ is non-zero. Conversely, if \(\rho_2^\star\) is a solution of \eqref{eq:min_pb_2} for  
\(\beta\), then it is also a solution of \eqref{eq:min_pb} for  
\begin{align}
    \tilde{\beta} = 2 \| \mathbb{E}_{Y \sim \rho^{\star}_2} [\varphi(Y)] - v \|_2 \beta. 
\end{align}
\end{theorem}

\subsection{Feature matching with KL regularization as an implicit energy-based model}

\begin{theorem} \label{thm:EBFT_KL}
    Consider the KL-regularized objective
\begin{align}
\begin{split} \label{eq:fm_kl_app}
&\min_{\rho} 
\;\; \mathbb{E}_{c \sim p} \Big[
\big\| \mathbb{E}_{\rho(\cdot\mid c)}[\phi_c(y)] - \mathbb{E}_{p(\cdot\mid c)}[\phi_c(y)] \big\|^2
+\; {\textstyle \frac{1}{\beta}} \, 
D_{\mathrm{KL}}
\!\left(\rho(\cdot| c)\,\|\,q(\cdot|c)\right) \Big],
\end{split}
\end{align}
where $\beta>0$ controls the strength of the regularization. the solution to \eqref{eq:fm_kl_app} has the form of an exponential tilt of the base distribution,
\begin{align*}
\rho^{\star}(y|c) \propto q(y|c)\,\exp\!\big(-\chi_c^{\top}\phi_c(y)\big),
\end{align*}
for a context-dependent vector $\chi_c \in \mathbb{R}^d$ chosen to optimize:
\begin{align} \label{eq:full_chi_problem}
    \max_{\|\chi\|_2 \leq \tilde{\beta}} \big\{ - \big( {\textstyle \frac{1}{\alpha}} - 1\big) \mathbb{E}_{y \sim p(\cdot|c)}[\phi_c(y)]^{\top} \chi + \mathbb{E}_{y \sim p(\cdot|c)} \big[ \log \rho_{\chi}(y|c) \big] \big\}, 
\end{align}
where $\rho_{\chi}(y|c) \propto q(y|c)\,\exp\!\big(-\chi^{\top}\phi_c(y)\big)$,
for a $\tilde{\beta} > 0$ that depends on $\beta$. Two values of $\alpha$ admit specific interpretations:
\begin{itemize}
\item For pure EBFT ($\alpha=1$), the optimal $\chi$ is the \textbf{maximum likelihood estimate}: When $\alpha = 1$, the problem \eqref{eq:full_chi_problem} simplifies to:
\begin{align}
    \max_{\|\chi\|_2 \leq \tilde{\beta}} \mathbb{E}_{y \sim p(\cdot|c)} \big[ \log \rho_{\chi}(y|c) \big], 
\end{align}
This corresponds to the maximum likelihood loss function for an energy-based model with 
energy function
\(
E(y,c) = \chi^{\top}\phi_c(y).
\)
\item For $\alpha^{+} = 0$, the optimal $\chi$ has the \textbf{same direction as the ground-truth mean feature} $\mathbb{E}_{y \sim p(\cdot|c)}[\phi_c(y)]$: When $\alpha = 0^{+}$, the problem \eqref{eq:full_chi_problem} is equivalent to
    \begin{align}
        \max_{\substack{\chi \in \mathbb{R}^d \\ \|\chi\|_{2} \leq \tilde{\beta}}} - \mathbb{E}_{y \sim p(\cdot|c)}[\phi_c(y)]^{\top} \chi,
    \end{align}
    which has optimal solution 
    \begin{align}
        \chi^{\star} = - \tilde{\beta}\frac{\mathbb{E}_{y \sim p(\cdot|c)}[\phi_c(y)]}{\big\| \mathbb{E}_{y \sim p(\cdot|c)}[\phi_c(y)] \big\|_2}.
    \end{align}
\end{itemize}
Thus, for $\alpha \in (0,1)$, the solution $\chi^{\star}$ of \eqref{eq:full_chi_problem} interpolates between the maximum likelihood estimate and the rescaled ground-truth mean feature.
\end{theorem}
\begin{proof}
Given a context $c$, and a completion length $G$, let us apply \Cref{eq:L_2_duality} and \Cref{eq:dual_equivalence} by setting $\mathcal{Y} = \mathcal{V}^G$,  $\varphi(y) = \phi_c(y)$, $\pi = \pi(\cdot|c)$, and $v = {\textstyle \frac{1}{\alpha}} \mathbb{E}_{y \sim p(\cdot|c)}[\phi_c(y)] \|^2$.
Then, the problems \eqref{eq:min_pb}, \eqref{eq:min_pb_2} and \eqref{eq:max_pb} take the form
\begin{align} \label{eq:min_pb_applied}
    &\min_{\rho \in \mathcal{P}(\mathcal{Y})} 
    {\textstyle \frac{1}{\tilde{\beta}}} D_{\mathrm{KL}}\big(\rho \| \pi(\cdot|c) \big) 
    + \| \mathbb{E}_{y \sim \rho} [\phi_c(y)] - {\textstyle \frac{1}{\alpha}} \mathbb{E}_{y \sim p(\cdot|c)}[\phi_c(y)] \|_2, \\
    \label{eq:min_pb_2_applied}
    &\min_{\rho \in \mathcal{P}(\mathcal{Y})} 
    {\textstyle \frac{1}{\beta}} D_{\mathrm{KL}}\big(\rho \| \pi(\cdot|c) \big) 
    + \| \mathbb{E}_{y \sim \rho} [\phi_c(y)] - {\textstyle \frac{1}{\alpha}} \mathbb{E}_{y \sim p(\cdot|c)}[\phi_c(y)] \|_2^2, \\
    \label{eq:max_pb_applied}
    &\max_{\substack{h \in \mathbb{R}^d \\ \|h\|_{2} \leq 1}} - {\textstyle \frac{1}{\alpha}} \mathbb{E}_{y \sim p(\cdot|c)}[\phi_c(y)]^{\top} h
    - {\textstyle \frac{1}{\tilde{\beta}}} 
    \log \mathbb{E}_{y \sim \pi(\cdot|c)} \big[ \exp \big( - \tilde{\beta} 
    \phi_c(y)^{\top} h
    \big) \big],
\end{align}
Problem \eqref{eq:min_pb_2_applied} is the KL-regularized feature-matching objective with alignment bias, 
if we absorb the constant $\alpha$ accompanying $\| \mathbb{E}_{y \sim \rho} [\phi_c(y)] - {\textstyle \frac{1}{\alpha}} \mathbb{E}_{y \sim p(\cdot|c)}[\phi_c(y)] \|_2^2$ into the constants $\beta$.
The (unique) solution \eqref{eq:min_pb_applied} of $\rho^{\star}$ and the solution $h^{\star}$ of \eqref{eq:max_pb_applied} are related by the equation
\begin{align}
    \rho^\star(y) = 
        \frac{\pi(y|c) \exp \big( - \tilde{\beta} \phi_c(y)^{\top} h^\star \big)}{\sum_{y'} \pi(y'|c) \exp \big( - \tilde{\beta} \phi_c(y')^{\top} h^\star \big)},
\end{align}
and $\rho^{\star}$ is also the (unique) solution of \eqref{eq:min_pb_2_applied} provided that 
\begin{align} \label{eq:tilde_beta_choice}
\tilde{\beta} = 2 \| \mathbb{E}_{Y \sim \rho^{\star}} [\varphi(Y)] - {\textstyle \frac{1}{\alpha}} \mathbb{E}_{y \sim p(\cdot|c)}[\phi_c(y)] \|_2 \beta. 
\end{align}
Observe that the problem \eqref{eq:min_pb_2_applied} is equal to the problem \eqref{eq:p_theta_c}, which means that \eqref{eq:min_pb_applied}-\eqref{eq:max_pb_applied} characterize the optimal $p_{\theta}(\cdot|c)$ when $\tilde{\beta}$ is chosen according to \eqref{eq:tilde_beta_choice}. 

Next, we focus on the problem \eqref{eq:max_pb_applied}. We define $\mathcal{E}_{\tilde{\beta}}$ as the following class of energy functions:
\begin{align}
\begin{split}
    \mathcal{E}_{\tilde{\beta}} &= \{ E : \mathcal{V}^G \to \mathbb{R} \, | \, \exists \chi \in \mathbb{R}^d, \mathrm{ s.t. } \|\chi\|_2 \leq \tilde{\beta}, \mathrm{ and } \, \forall x \in \mathcal{V}^G, \, E(x) = \chi^{\top} \phi_c(x) \} 
    \\ &= \{ E : \mathcal{V}^G \to \mathbb{R} \, | \, \exists h \in \mathbb{R}^d, \mathrm{ s.t. } \|h\|_2 \leq 1, \mathrm{ and } \, \forall x \in \mathcal{V}^G, \, E(x) = \tilde{\beta} h^{\top} \phi_c(x) \},
\end{split}
\end{align}
and given $\chi \in \mathbb{R}^d$,  we define $\rho_{\chi}, \rho^{(\tilde{\beta})}_{h} \in \mathcal{P}(\mathcal{V}^G)$ as
\begin{align}
    \rho_{\chi}(y) = 
    \frac{\pi(y|c) \exp \big( - \chi^{\top} \phi_c(y) \big)}{\mathbb{E}_{y' \sim \pi(\cdot|c)} \big[ \exp \big( - \chi^{\top} \phi_c(y') \big) \big]}, \qquad \rho^{(\tilde{\beta})}_{h}(y) = 
    \frac{\pi(y|c) \exp \big( - \tilde{\beta} \chi^{\top} \phi_c(y) \big)}{\mathbb{E}_{y' \sim \pi(\cdot|c)} \big[ \exp \big( - \tilde{\beta} \chi^{\top} \phi_c(y') \big) \big]}.
\end{align}
We rewrite the problem \eqref{eq:max_pb_applied} as
\begin{align}
\begin{split} \label{eq:max_pb_applied_rewritten}
    &\max_{\substack{h \in \mathbb{R}^d \\ \|h\|_{2} \leq 1}} - \big( {\textstyle \frac{1}{\alpha}} - 1\big) \mathbb{E}_{y \sim p(\cdot|c)}[\phi_c(y)]^{\top} h
    - \mathbb{E}_{y \sim p(\cdot|c)}[\phi_c(y)]^{\top} h - {\textstyle \frac{1}{\tilde{\beta}}} 
    \log \mathbb{E}_{y \sim \pi(\cdot|c)} \big[ \exp \big( - \tilde{\beta} 
    \phi_c(y)^{\top} h
    \big) \big] \\
    &= \max_{\substack{h \in \mathbb{R}^d \\ \|h\|_{2} \leq 1}} - \big( {\textstyle \frac{1}{\alpha}} - 1\big) \mathbb{E}_{y \sim p(\cdot|c)}[\phi_c(y)]^{\top} h + \frac{1}{\tilde{\beta}} \mathbb{E}_{y \sim p(\cdot|c)} \Big[ \log \frac{\rho^{(\tilde{\beta})}_h(y)}{\pi(y|c)} \Big] \\
    &= \max_{\substack{h \in \mathbb{R}^d \\ \|h\|_{2} \leq 1}} - \big( {\textstyle \frac{1}{\alpha}} - 1\big) \mathbb{E}_{y \sim p(\cdot|c)}[\phi_c(y)]^{\top} h + \frac{1}{\tilde{\beta}} \mathbb{E}_{y \sim p(\cdot|c)} \Big[ \log \rho^{(\tilde{\beta})}_h(y) \Big] + \mathrm{const}.
\end{split}
\end{align}

Writing the right-hand side problem in terms of $\chi$ instead of $h$, and multiplying the objective by $\tilde{\beta}$, yields the problem in \eqref{eq:full_chi_problem}.
\end{proof}

\section{Computing the REINFORCE gradient and the RLOO baseline for EBFT} \label{sec:RLOO_baseline}
We derive the REINFORCE gradient first. Rewriting $\tilde{\mathcal{L}}_{\mathrm{FM}}(\theta;c,y)$ explicitly in terms of $p_{\theta}$, and using that $\hat{y}$, $\tilde{y}$ play a symmetric role, we have the following:
\begin{align}
\begin{split}
&\nabla_{\theta} \tilde{\mathcal{L}}_{\mathrm{FM}}(\theta;c,y) \\ &= \! \nabla_{\theta} \Big( \sum_{\hat{y}, \tilde{y}} p_{\theta}(\hat{y}|c) p_{\theta}(\tilde{y}|c) \phi_c(\hat{y})^{\top} \phi_c(\tilde{y}) - 2 \sum_{\hat{y}} p_{\theta}(\hat{y}|c) \phi_c(\hat{y})^{\top} \phi_c(y) \Big)
\\ &= \! \sum_{\hat{y}, \tilde{y}} \Big( \nabla_{\theta} \log p_{\theta}(\hat{y}|c) + \nabla_{\theta} \log p_{\theta}(\tilde{y}|c) \Big) p_{\theta}(\hat{y}|c) p_{\theta}(\tilde{y}|c) \phi_c(\hat{y})^{\top} \phi_c(\tilde{y}) - 2 \sum_{\hat{y}} \nabla_{\theta} \log p_{\theta}(\hat{y}|c) p_{\theta}(\hat{y}|c) \phi_c(\hat{y})^{\top} \phi_c(y) 
\\ &= \! 2 \mathbb{E}_{\hat{y}, \tilde{y} \sim p_{\theta}(\cdot|c)} \! \Big[ 
\nabla \log p_{\theta}(\hat{y}|c) 
\phi_c(\hat{y})^{\top} \! \phi_c(\tilde{y}) \Big] - 2 \mathbb{E}_{\hat{y} \sim p_{\theta}(\cdot|c)} \Big[ \nabla \log p_{\theta}(\hat{y}|c) \phi_c(\hat{y})^{\top} \phi_c(y) \Big].
\end{split}
\end{align}
Next, we derive the RLOO baseline. Define 
\begin{align}
T_1^{(j)} = 2 \phi_c(\hat{y}_j)^{\top} \phi_c(y),
\qquad T_2^{(j)} = 
\frac{2}{n \!- \! 1} \!\! \sum_{j'=1, j' \neq j}^{n} \!\! \phi_c(\hat{y}_j)^{\top} \phi_c(\hat{y}_{j'}),
\end{align}
which means that we can rewrite the REINFORCE gradient \eqref{eq:REINFORCE_practical}

as
\begin{align} \label{eq:REINFORCE_features_approx_2}
    - \frac{1}{n} \sum_{j=1}^{n} \nabla \log p_{\theta}(Y^{(j)}|c) \big( 
    T_1^{(j)} - T_2^{(j)}
    \big).
\end{align}
We want to use a baseline $b^{(j)}$ to reduce the variance of the gradient estimate \eqref{eq:REINFORCE_features_approx_2}. That is, the estimate with baseline $b^{(j)}$ reads
\begin{align} \label{eq:REINFORCE_features_approx_3}
    - \frac{1}{n} \sum_{j=1}^{n} \nabla \log p_{\theta}(Y^{(j)}|c) \big( 
    T_1^{(j)} - T_2^{(j)} - b^{(j)}
    \big).
\end{align}
For the baselined gradient estimate to be unbiased, we need that $b^{(j)}$ is independent of $Y^{(j)}$.

A naive RLOO baseline would be $b^{(j)} = \frac{1}{N-1} \sum_{j'=1, j' \neq j}^{N} \big( T_1^{(j')} - T_2^{(j')} \big)$, i.e. simply averaging the rewards for all rollouts except the $j$-th one. However, the terms $T_2^{(j')}$ are not independent of $Y^{(j)}$, which means that this baseline is not independent of $Y^{(j)}$. To obtain an independent baseline, we need to replace $T_2^{(j')}$ by $T_2^{(j',j)}$, defined as
\begin{align}
\begin{split}
    T_2^{(j',j)} &= \frac{2}{n-2} \sum_{j''=1, \, j'' \neq j',j}^n \phi_c(\hat{y}_{j''})^{\top} \phi_c(\hat{y}_{j'}) = \frac{2}{n-2} \Big( \sum_{j''=1, j'' \neq j'}^n \varphi(c\!:\!\hat{y}_{j''})^{\top} \varphi(c\!:\!\hat{y}_{j'}) - \varphi(c\!:\!\hat{y}_{j})^{\top} \varphi(c\!:\!\hat{y}_{j'}) \Big) \\ &= \frac{n-1}{n-2} T_2^{(j')} - \frac{2}{n-2} \varphi(c\!:\!\hat{y}_{j})^{\top} \varphi(c\!:\!\hat{y}_{j'}).
\end{split}
\end{align}
Thus, the baseline that we end up with is:
\begin{align}
\begin{split} \label{eq:RLOO_baseline}
    b^{(j)} &= \frac{1}{n-1} \sum_{j'=1, j' \neq j}^{n} \big( T_1^{(j')} - T_2^{(j')} \big) \\&= \frac{1}{n-1} \sum_{j'=1, j' \neq j}^{n} \Big( T_1^{(j')} - \Big( \frac{n-1}{n-2} T_2^{(j')} - \frac{2}{n-2} \phi(Y^{(j)})^{\top} \phi(Y^{(j')}) \Big) \Big) \\ &=\frac{1}{n-1}\sum_{j'=1, j' \neq j}^{n} T_1^{(j')}
    -\frac{1}{n-2}\sum_{j'=1, j' \neq j}^{n} T_2^{(j')}
    +\frac{1}{n-2} T_2^{(j)} .
\end{split}
\end{align}

\section{Details on the strided parallel rollout procedure} \label{sec:strided_parallel}

Sampling from a single, unstructured sequence provides only one supervision point and is a major bottleneck for EBFT, particularly because each sample must be embedded via forward passes through a separate feature network. Instead, we treat a single training sequence as a source of multiple nested prompts by identifying many anchor points along the text. Sampling from these points sequentially is prohibitively expensive, so we implement a novel parallel generation pipeline that simultaneously samples from different anchor points in one forward pass, similar to the custom attention mask approach introduced by Quiet-STaR \cite{zelikman2024quietstarlanguagemodelsteach}.
Given a starting sequence $x_{0:T-1}$ of length $T$, a stride $s$, and a generation length $G$, we construct a set of nested prompts by segmenting $x_{0:T-1}$ every $s$ tokens. This yields $B = \lfloor\frac{T-G}{s}\rfloor$ nested prompts. For each prompt $c_b = x_{0:bs} \quad (b = 1, \ldots, B)$, we take the next $G$ tokens in the original sequence $x$ as the ground-truth continuation $y_b$, yielding $\{(c_b, y_b)\}_{b=1}^B$ ground truth context and completion pairs. Additionally, from each prompt, we sample a continuation $\hat{y}_b$ of length $G$. Using our custom mask, we can sample one token from each prefix in just one forward pass. We can then obtain $\{\hat{y}_b\}_{b=1}^B$ length $G$ model completions in $G$ forward passes.
The resulting $BG$ generated tokens are appended in generation order; for example, with $B=3$ and $G=2$:
\[
  [\hat{y}_{1,0},\hat{y}_{2,0},\hat{y}_{3,0},\;\hat{y}_{1,1},\hat{y}_{2,1},\hat{y}_{3,1}].
\]
This interleaving supports an efficient reshape into per-block windows for downstream scoring and feature extraction.
In particular, we exploit the same strided structure to compute features for all generated blocks (and their ground-truth counterparts) with a single batched call to the feature network, followed by a reshape/indexing step to recover per-block embeddings.

\Cref{fig:placeholder} shows a sketch of the strided parallel rollout procedure for a sequence of length $L=12$, stride $S=4$ and completion length $G=4$, which means that $B = \lfloor(12 - 4)/ 4 \rfloor = 2$ context-completion pairs are used: $c_1 = (t_i)_{i=0}^{3}$, $y_1 = (t_i)_{i=4}^{7}$; and $c_2 = (t_i)_{i=0}^{7}$, $y_2 = (t_i)_{i=8}^{11}$. The ground truth sequence is in blue, and the generated completions are in red and green: $\hat{y}_1 = (t_{i \, a})_{i=4}^{7}$, $\hat{y}_2 = (t_{i \, b})_{i=8}^{11}$. 
\begin{figure}[h]
    \centering
    \includegraphics[width=0.8\linewidth]{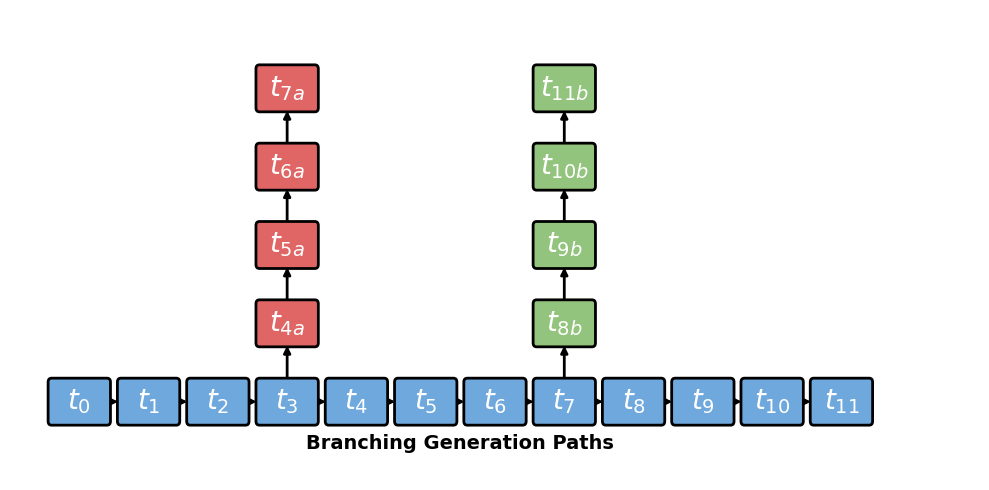}
    \caption{Strided parallel rollouts for a sequence of length $L=12$, stride $S=4$ and completion length $G=4$.}
    \label{fig:placeholder}
\end{figure}

The algorithm to obtain the generated completions $\hat{y}_1$ and $\hat{y}_2$ involves four model calls to $p_{\theta}$. \Cref{fig:placeholder} shows a sketch of the generated tokens using the strided parallel rollout procedure, and \Cref{fig:custom_attention} shows the custom attention matrix for the fourth call (the horizontal and vertical lines show the three top left custom matrices used for the first three calls).

\begin{figure}[H]
\begin{align*} 
-
\left[
\begin{array}{cccccccc|cc|cc|cc}
0 & \infty & \infty & \infty & \infty & \infty & \infty & \infty & \infty & \infty & \infty & \infty & \infty & \infty \\
0 & 0 & \infty & \infty & \infty & \infty & \infty & \infty & \infty & \infty & \infty & \infty & \infty & \infty \\
0 & 0 & 0 & \infty & \infty & \infty & \infty & \infty & \infty & \infty & \infty & \infty & \infty & \infty \\
0 & 0 & 0 & 0 & \infty & \infty & \infty & \infty & \infty & \infty & \infty & \infty & \infty & \infty \\
0 & 0 & 0 & 0 & 0 & \infty & \infty & \infty & \infty & \infty & \infty & \infty & \infty & \infty \\
0 & 0 & 0 & 0 & 0 & 0 & \infty & \infty & \infty & \infty & \infty & \infty & \infty & \infty \\
0 & 0 & 0 & 0 & 0 & 0 & 0 & \infty & \infty & \infty & \infty & \infty & \infty & \infty \\
0 & 0 & 0 & 0 & 0 & 0 & 0 & 0 & \infty & \infty & \infty & \infty & \infty & \infty \\ \hline
0 & 0 & 0 & 0 & \infty & \infty & \infty & \infty & 0 & \infty & \infty & \infty & \infty & \infty \\
0 & 0 & 0 & 0 & 0 & 0 & 0 & 0 & \infty & 0 & \infty & \infty & \infty & \infty \\ \hline
0 & 0 & 0 & 0 & \infty & \infty & \infty & \infty & 0 & \infty & 0 & \infty & \infty & \infty \\
0 & 0 & 0 & 0 & 0 & 0 & 0 & 0 & \infty & 0 & \infty & 0 & \infty & \infty \\ \hline
0 & 0 & 0 & 0 & \infty & \infty & \infty & \infty & 0 & \infty & 0 & \infty & 0 & \infty \\
0 & 0 & 0 & 0 & 0 & 0 & 0 & 0 & \infty & 0 & \infty & 0 & \infty & 0 \\
\end{array}
\right]
\end{align*}
\label{fig:custom_attention}
\caption{Custom attention matrix $A$ in for a sequence of length $L=12$, stride $S=4$ and completion length $G=4$. When the entry $A_{ij}$ is 0, the token in position $i$ attends to the token in position $j$, and when it is $-\infty$ it does not.}
\end{figure}

\section{Additional Experimental Details}
\label{sec:hyperparams}

We build our EBFT method on top of the OpenRLHF~\citep{hu2024openrlhf} framework, as well as use the OpenRLHF implementation of SFT and GRPO for our baselines. We use an internal cluster of 80GB H100 GPUs to conduct SFT, RLVR, and EBFT training runs. For Q\&A code, a single epoch of SFT training takes 0.5 hours to run on a single 80GB H100, whereas a single epoch of RLVR using \url{vllm} training takes roughly 28 hours to run on two 80GB H100s, and a single epoch of EBFT training with our under-optimized implementation (without \url{vllm}) takes roughly 36 hours.

\begin{table}[h!]
    \centering
    \begin{tabular}{|l|c|c|c|}
        \hline
        \textbf{Parameter} & \multicolumn{3}{c|}{\textbf{Value}} \\
        \cline{2-4}
        & \textbf{Q\&A Code} & \textbf{Unstructured Code} & \textbf{Translation} \\
        \hline
        Rollout Batch Size & 
        \multicolumn{3}{c|}{16}
        \\
        \hline
        Sequence Length & \multicolumn{3}{c|}{1024} \\
        \hline
        Completion Length & 8 & 8 & 4
        \\
        \hline
        Stride & 8 & 8 & 2
        \\
        \hline
        Actor Learning Rate & \multicolumn{3}{c|}{$1 \times 10^{-6}$} \\
        \hline
        Temperature & 
        \multicolumn{3}{c|}{0.6}
        \\
        \hline
        KL Coefficient & \multicolumn{3}{c|}{0} \\
        \hline
        Samples per Prompt & \multicolumn{3}{c|}{4} \\
        \hline
        Training Batch Size & \multicolumn{3}{c|}{\texttt{rollout\_batch\_size} $\times$ \texttt{samples\_per\_prompt} = 64} \\
        \hline
        Warmup & \multicolumn{3}{c|}{0.03} \\
        \hline
        Adam Betas & \multicolumn{3}{c|}{(0.9, 0.95)} \\
        \hline
        Num Epochs & \multicolumn{3}{c|}{1} \\
        \hline
    \end{tabular}
    \vspace{2mm}
    \caption{Hyperparameter Configuration for EBFT runs.}
    \label{table:hyperparams_EBFT}
\end{table}

Hyperparameter details for the SFT training runs are provided in Table~\ref{table:hyperparams_sft} for the warmstarted models (initialization for EBFT/RLVR) and in Table~\ref{table:hyperparams_sft_baseline} for the five-epoch SFT baseline runs. We sweep over learning rate scheduler as well as training batch size for each task.

\begin{table}[h!]
    \centering
    \begin{tabular}{|l|l|}
        \hline
        \textbf{Parameter} & \textbf{Value} \\
        \hline
        Training Batch Size & 64 \\
        Epochs & 1 \\
        Max Length & 2048 \\
        Learning Rate & $1 \times 10^{-5}$ \\
        Scheduler & Warmup + Cosine Decay to 0.1$\times$ \texttt{lr} \\
        Warmup & 0.03 \\
        \hline
    \end{tabular}
    \vspace{2mm}
    \caption{Hyperparameter Configuration for SFT warmstart runs. \label{table:hyperparams_sft}}
\end{table}

\begin{table}[h!]
    \centering
    \begin{tabular}{|l|c|c|c|}
        \hline
        \textbf{Parameter} & \multicolumn{3}{c|}{\textbf{Value}} \\
        \cline{2-4}
        & \textbf{Q\&A Code} & \textbf{Unstructured Code} & \textbf{Translation} \\
        \hline
        Training Batch Size & \multicolumn{3}{c|}{\{64, 128\}}  \\
        \hline
        Max Length & \multicolumn{3}{c|}{2048} \\
        \hline
        (Learning Rate, Scheduler) & \multicolumn{3}{c|}{\{($5 \times 10^{-6}$, Warmup + Constant), ($1 \times 10^{-5}$, Warmup + Cosine Decay)\}} \\
        \hline
        Warmup & \multicolumn{3}{c|}{0.03} \\
        \hline
        Adam Betas & \multicolumn{3}{c|}{(0.9, 0.95)} \\
        \hline
        Num Epochs & \multicolumn{3}{c|}{5} \\
        \hline
    \end{tabular}
    \vspace{2mm}
    \caption{Hyperparameter Configuration for SFT baseline runs.\label{table:hyperparams_sft_baseline}}
\end{table}

Hyperparameters for RLVR training runs are provided in Table~\ref{table:hyperparams_rlvr}. For RLVR, we fix training to be online and determine rollout batch size by roughly equating number of examples seen per step across both RLVR and EBFT for each task, sweeping over two values.

\begin{table}[h!]
    \centering
    \begin{tabular}{|l|c|c|}
        \hline
        \textbf{Parameter} & \multicolumn{2}{c|}{\textbf{Value}} \\
        \cline{2-3}
        & \textbf{Q\&A Code} & \textbf{Translation} \\
        \hline
        Rollout Batch Size & \{32, 64\} & \{128, 256\}  \\
        \hline
        Prompt Max Length & \multicolumn{2}{c|}{1024} \\
        \hline
        Generate Max Length & \multicolumn{2}{c|}{1024} \\
        \hline
        Actor Learning Rate & \multicolumn{2}{c|}{$1 \times 10^{-6}$} \\
        \hline
        Temperature & \multicolumn{2}{c|}{1.0} \\
        \hline
        KL Coefficient & \multicolumn{2}{c|}{0} \\
        \hline
        Samples per Prompt & \multicolumn{2}{c|}{8} \\
        \hline
        Training Batch Size & \multicolumn{2}{c|}{\texttt{rollout\_batch\_size} $\times$ \texttt{samples\_per\_prompt}} \\
        \hline
        Warmup & \multicolumn{2}{c|}{0.03} \\
        \hline
        Adam Betas & \multicolumn{2}{c|}{(0.9, 0.95)} \\
        \hline
        Num Epochs & \multicolumn{2}{c|}{1} \\
        \hline
    \end{tabular}
    \vspace{2mm}
    \caption{Hyperparameter Configuration for RLVR baseline runs.\label{table:hyperparams_rlvr}}
\end{table}

\newpage

\clearpage
\begin{table*}[p]
\centering
\footnotesize
\setlength{\tabcolsep}{4pt}
\renewcommand{\arraystretch}{1.15}
\newcommand{\mh}[1]{{\scriptsize\textbf{#1}}}

{\normalsize\textbf{Q\&A Coding}}\\[1mm]
\begin{tabular}{l c cccc cccc cccc}
\toprule
& \multicolumn{1}{c}{\textbf{CE Q\&A}}
& \multicolumn{4}{c}{\textbf{HumanEval}}
& \multicolumn{4}{c}{\textbf{MBPP}}
& \multicolumn{4}{c}{\textbf{Multipl-E}} \\
\cmidrule(lr){2-2}
\cmidrule(lr){3-6}
\cmidrule(lr){7-10}
\cmidrule(lr){11-14}
\textbf{Method}
& 
& \mh{greedy} & \mh{pass@1} & \mh{pass@4} & \mh{pass@16}
& \mh{greedy} & \mh{pass@1} & \mh{pass@4} & \mh{pass@16}
& \mh{greedy} & \mh{pass@1} & \mh{pass@4} & \mh{pass@16} \\
\midrule
Base         & 0.524 & 0.348 & 0.324 & 0.490 & 0.622 & 0.599 & 0.514 & 0.703 & 0.782 & 0.506 & 0.433 & 0.626 & 0.742 \\
Warm start    & 0.571 & 0.415 & 0.385 & 0.540 & 0.665 & 0.595 & 0.527 & 0.691 & 0.774 & 0.440 & 0.408 & 0.602 & 0.731 \\
\midrule
SFT          & 0.408 & 0.457 & 0.427 & 0.578 & 0.713 & 0.576 & 0.558 & 0.701 & 0.790 & 0.465 & 0.406 & 0.596 & 0.723 \\
EBFT         & \textbf{0.326} & 0.494 & 0.448 & 0.616 & 0.750 & \textbf{0.650} & \textbf{0.603} & \textbf{0.728} & 0.813 & 0.524 & 0.488 & 0.645 & 0.753 \\
EBFT (ws.)   & 0.337 & \textbf{0.512} & 0.500 & \textbf{0.642} & \textbf{0.756} & 0.638 & 0.584 & 0.727 & \textbf{0.817} & 0.476 & 0.452 & 0.621 & 0.734 \\
\midrule
RLVR         & 0.806 & 0.451 & 0.443 & 0.602 & 0.695 & 0.623 & 0.583 & 0.722 & 0.794 & \textbf{0.531} & \textbf{0.502} & \textbf{0.658} & \textbf{0.767} \\
RLVR (ws.)   & 0.713 & 0.482 & \textbf{0.516} & 0.640 & 0.738 & 0.607 & 0.596 & 0.712 & 0.782 & 0.484 & 0.475 & 0.632 & 0.729 \\
\bottomrule
\end{tabular}

\vspace{2mm}

{\normalsize\textbf{Unstructured Coding}}\\[1mm]
\begin{tabular}{l cccc cccc}
\toprule
& \multicolumn{4}{c}{\textbf{HumanEval}}
& \multicolumn{4}{c}{\textbf{MBPP}} \\
\cmidrule(lr){2-5}
\cmidrule(lr){6-9}
\textbf{Method}
& \mh{greedy} & \mh{pass@1} & \mh{pass@4} & \mh{pass@16}
& \mh{greedy} & \mh{pass@1} & \mh{pass@4} & \mh{pass@16} \\
\midrule
Base         & 0.348 & 0.324 & 0.490 & 0.622 & 0.599 & 0.514 & 0.703 & 0.782 \\
Warm start    & 0.463 & 0.419 & 0.586 & 0.707 & 0.553 & 0.497 & 0.690 & 0.778 \\
\midrule
SFT          & 0.451 & 0.417 & 0.596 & 0.707 & 0.556 & 0.516 & 0.691 & 0.786 \\
EBFT         & 0.500 & 0.465 & 0.610 & \textbf{0.726} & \textbf{0.611} & \textbf{0.583} & \textbf{0.718} & \textbf{0.813} \\
EBFT (ws.)   & \textbf{0.512} & \textbf{0.478} & \textbf{0.629} & \textbf{0.726} & 0.572 & 0.550 & 0.704 & \textbf{0.813} \\
\bottomrule
\end{tabular}

\vspace{2mm}

{\normalsize\textbf{Translation}}\\[1mm]
\begin{tabular}{l c cccc cccc cccc}
\toprule
& \multicolumn{1}{c}{\textbf{CE Q\&A}}
& \multicolumn{4}{c}{\textbf{WMT'22 - COMET}}
& \multicolumn{4}{c}{\textbf{MTNT - COMET}}
& \multicolumn{4}{c}{\textbf{OpenSubtitles - COMET}} \\
\cmidrule(lr){2-2}
\cmidrule(lr){3-6}
\cmidrule(lr){7-10}
\cmidrule(lr){11-14}
\textbf{Method}
& 
& \mh{greedy} & \mh{best-of-1} & \mh{best-of-4} & \mh{best-of-16}
& \mh{greedy} & \mh{best-of-1} & \mh{best-of-4} & \mh{best-of-16}
& \mh{greedy} & \mh{best-of-1} & \mh{best-of-4} & \mh{best-of-16} \\
\midrule
Base         & 2.567 & 0.649 & 0.611 & 0.711 & 0.757 & 0.627 & 0.590 & 0.679 & 0.724 & 0.658 & 0.630 & 0.712 & 0.753 \\
Warm start    & 2.648 & 0.733 & 0.712 & 0.776 & 0.807 & 0.705 & 0.683 & 0.759 & 0.796 & 0.696 & 0.677 & 0.742 & 0.776 \\
\midrule
SFT          & 2.692 & 0.747 & 0.722 & 0.784 & 0.815 & 0.703 & 0.683 & 0.755 & 0.792 & 0.701 & 0.682 & 0.745 & 0.777 \\
EBFT         & \textbf{2.399} & 0.740 & 0.725 & 0.777 & 0.804 & 0.737 & 0.728 & 0.778 & 0.808 & 0.700 & 0.691 & 0.742 & 0.775 \\
EBFT (ws.)   & 2.451 & \textbf{0.753} & \textbf{0.741} & \textbf{0.788} & \textbf{0.812} & \textbf{0.742} & \textbf{0.732} & \textbf{0.782} & \textbf{0.810} & \textbf{0.708} & 0.699 & \textbf{0.749} & \textbf{0.779} \\
\midrule
RLVR         & 3.225 & 0.704 & 0.698 & 0.743 & 0.769 & 0.705 & 0.698 & 0.745 & 0.772 & 0.684 & 0.679 & 0.718 & 0.741 \\
RLVR (ws.)   & 3.148 & 0.738 & 0.730 & 0.771 & 0.794 & 0.727 & 0.721 & 0.765 & 0.789 & \textbf{0.708} & \textbf{0.703} & 0.740 & 0.762 \\
\bottomrule
\end{tabular}
\begin{tabular}{l cccc cccc cccc}
\toprule
& \multicolumn{4}{c}{\textbf{WMT'22 - BLEU}}
& \multicolumn{4}{c}{\textbf{MTNT - BLEU}}
& \multicolumn{4}{c}{\textbf{OpenSubtitles - BLEU}} \\
\cmidrule(lr){2-5}
\cmidrule(lr){6-9}
\cmidrule(lr){10-13}
\textbf{Method}
& \mh{greedy} & \mh{best-of-1} & \mh{best-of-4} & \mh{best-of-16}
& \mh{greedy} & \mh{best-of-1} & \mh{best-of-4} & \mh{best-of-16}
& \mh{greedy} & \mh{best-of-1} & \mh{best-of-4} & \mh{best-of-16} \\
\midrule
Base         & 0.069 & 0.103 & 0.166 & 0.215 & 0.073 & 0.098 & 0.154 & 0.197 & 0.081 & 0.171 & 0.238 & 0.281 \\
Warm start   & 0.185 & 0.165 & 0.235 & 0.286 & 0.159 & 0.139 & 0.200 & 0.244 & 0.129 & 0.204 & 0.264 & 0.307 \\
\midrule
SFT          & 0.198 & 0.172 & 0.242 & 0.294 & 0.152 & 0.139 & 0.198 & 0.242 & 0.130 & 0.205 & 0.264 & 0.305 \\
EBFT         & 0.204 & 0.187 & 0.247 & 0.289 & 0.212 & 0.175 & \textbf{0.221} & \textbf{0.258} & 0.136 & 0.219 & 0.266 & 0.305 \\
EBFT (ws.)   & \textbf{0.217} & \textbf{0.200} & \textbf{0.253} & \textbf{0.297} & 0.202 & 0.174 & 0.219 & 0.256 & 0.142 & 0.221 & \textbf{0.270} & \textbf{0.309} \\
\midrule
RLVR         & 0.192 & 0.185 & 0.223 & 0.250 & 0.202 & 0.173 & 0.206 & 0.228 & 0.135 & 0.223 & 0.249 & 0.267 \\
RLVR (ws.)   & 0.215 & 0.206 & 0.246 & 0.275 & \textbf{0.217} & \textbf{0.186} & 0.220 & 0.243 & \textbf{0.152} & \textbf{0.240} & 0.269 & 0.289 \\
\bottomrule
\end{tabular}

\vspace{2mm}
\caption{\textbf{Across all tasks, EBFT matches or outperforms both SFT and RLVR on downstream metrics while maintaining substantially lower cross-entropy, and warm starting improves performance for both EBFT and RLVR.} On Q\&A coding, EBFT achieves the best scores on HumanEval and MBPP, while RLVR is competitive on MultiPL-E. On unstructured coding, EBFT dominates across all benchmarks. On translation, EBFT (ws.)\ achieves the highest COMET scores on nearly every benchmark and leads on WMT'22 BLEU. RLVR (ws.)\ is competitive on MTNT and OpenSubtitles BLEU. Warm starting benefits both methods across all tasks. Notably, RLVR consistently degrades cross-entropy relative to the base model (e.g.\ 3.225 vs.\ 2.567 on translation), whereas EBFT improves it, suggesting that EBFT better preserves the model's language modeling capabilities while improving task performance.}
\label{tab:per_benchmark_results}
\end{table*}
\clearpage

\section{Additional Experimental Results}
\label{sec:add_experimental_results}

\subsection{Sweeping across $\gamma$, $\alpha$ and warm-starting}
\label{app:alpha_gamma_sweep}

Figures \ref{fig:q_a_coding_all}, \ref{fig:unstructured_coding_all}, \ref{fig:translation_all} and \ref{fig:translation_all_bleu} show 9 EBFT runs with $\alpha \in \{0, 0.5, 1\}$ and $\gamma \in \{0, 0.03, 0.1\}$ on the Q\&A Coding, Unstructured Coding and Translation tasks, in which the models are initialized from the base Qwen2.5-1.5B and Llama3.2-1B, respectively.
The observations below apply generally across tasks. We include additional observations about the behavior in particular settings in the captions of each figure.

\paragraph{Takeaways from the $\alpha$, $\gamma$ sweeps: $\alpha < 1$ is prone to instability when $\gamma = 0$, and increasing $\gamma$ reduces the CE loss} The choice $(\alpha, \gamma) = (1, 0)$ amounts to optimizing the pure feature-matching loss function $\mathcal{L}_{\mathrm{FM}}$; in this case the validation CE loss at a similar rate as for SFT, while the feature-matching loss decreases clearly faster, and the downstream performance is equal or better. The fact that pure FM beats SFT at reducing the CE loss may be attributed to FM with whitening optimizing a relaxation of the $\chi^{2}$ divergence (see \Cref{subsec:fm_variants} and \Cref{sec:whitening}).   
When $\gamma = 0$, and $\alpha \in \{0, 0.5\}$, the CE loss increases during training, which is not unexpected because the corresponding loss function is not a proper scoring rule: its minimizer is not the ground truth distribution $p$. In these settings, the FM loss decreases faster than when $\alpha$ is 1, even though we are optimizing a biased quantity, perhaps due to a bias-variance tradeoff of the gradient, and the downstream performance is slightly worse. 
Lastly, the CE loss gets reduced substantially with larger $\gamma$ values both when $\alpha$ is 0.5 or 1.0, while the FM loss increases slightly with larger $\gamma$. The downstream performance is not affected substantially in either of the two cases.

\paragraph{Warm-starting: EBFT is more robust to weak initializations than RLVR}
Looking at \Cref{tab:best_prefix_aggregate}, 
we can compare performance with and without warm-starting (running SFT for one epoch before initializing) for both EBFT and RLVR. 
Since both methods require sampling rollouts from the model, starting from a stronger model can in principle yield higher quality rollouts and improve RL gradients. 
However, the effect of warm-starting differs substantially between EBFT and RLVR. EBFT performs similarly with and without warm-starting, indicating that it is more robust to the quality of the initial model. In contrast, RLVR benefits heavily from warm-starting, and downstream performance and validation cross-entropy degrade significantly when initialized from weaker models. In summary, RLVR depends much more heavily on the capabilities of the initial model checkpoint.
We hypothesize that this difference arises for two reasons. First, RLVR needs sufficiently accurate initial rollouts to produce a meaningful reward signal: poor initializations lead to sparse reward feedback.
Second, RLVR introduces tension between reward maximization and maintaining low cross-entropy. In contrast, EBFT does not exhibit this conflict: it can simultaneously reduce the validation cross-entropy as much as SFT and improve downstream performance.

\begin{figure}
    \centering
    \includegraphics[width=0.99\linewidth]{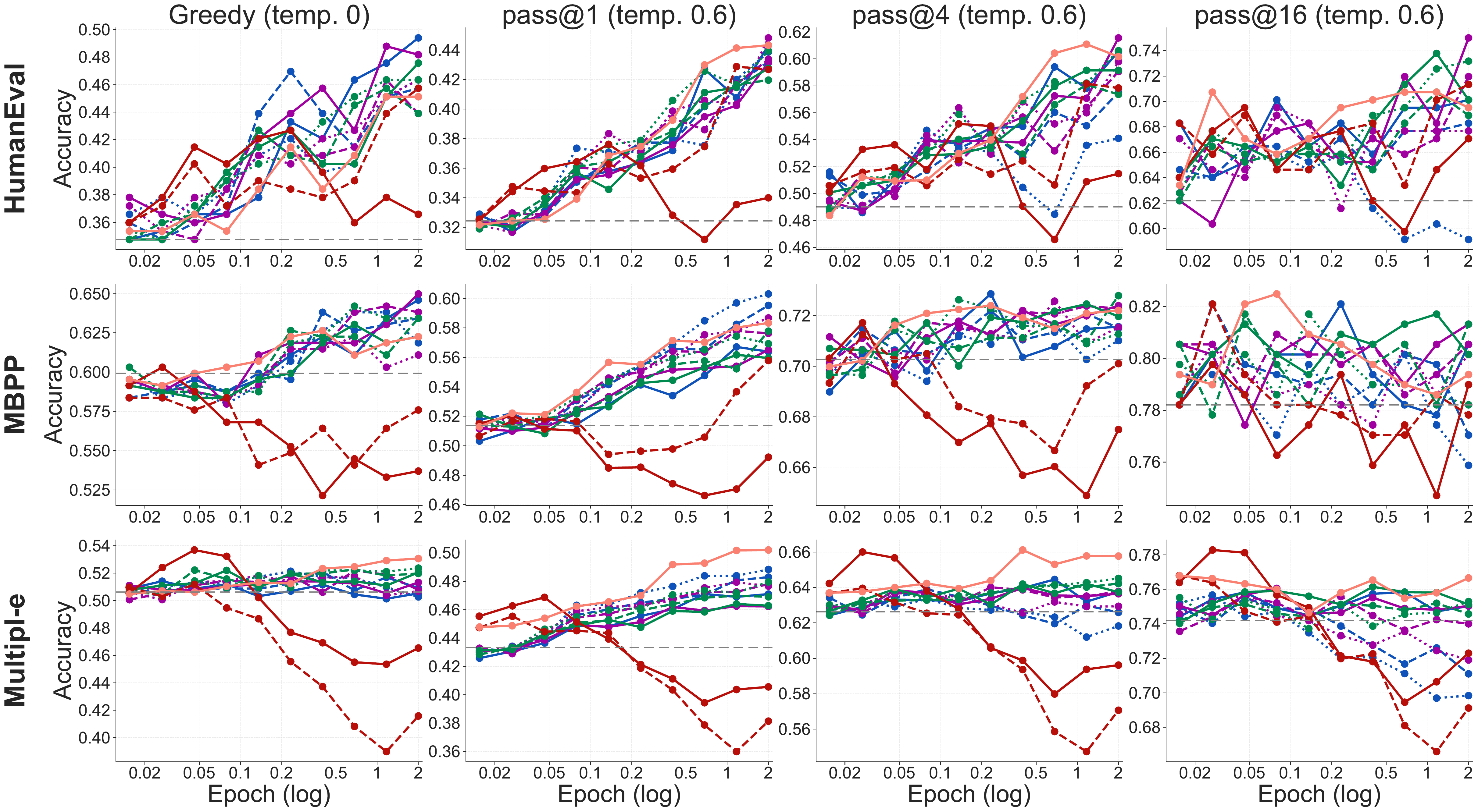}
    \includegraphics[width=0.99\linewidth]{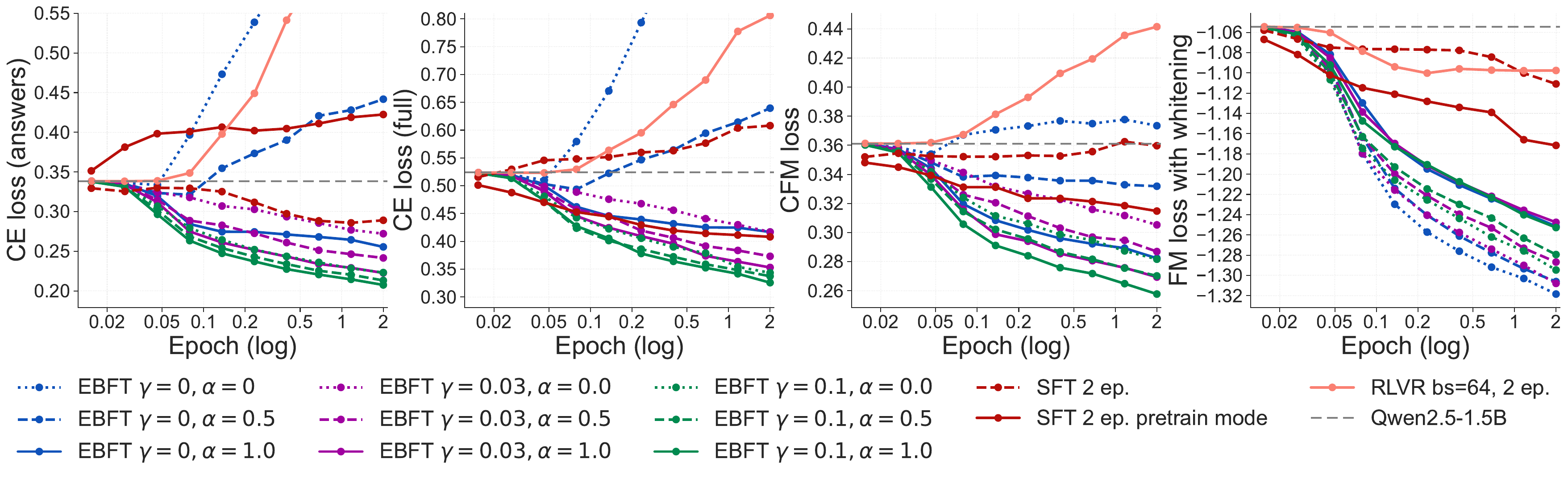}
    \caption{\textbf{On Q\&A Coding, increasing the cross-entropy weight $\alpha$ consistently lowers both validation cross-entropy and feature-matching losses, while SFT on full sequences yields faster initial downstream gains that quickly degrade.} We sweep $\alpha \in \{0, 0.5, 1.0\}$ and $\gamma \in \{0, 0.03, 0.1\}$ for EBFT initialized from base Qwen2.5-1.5B. As a baseline, we compare against SFT trained on full sequences (solid red), whereas the main text reports SFT trained only on the answer. SFT on full sequences improves downstream performance faster early on but quickly deteriorates, and answer-level cross-entropy rises. Setting $\alpha = 0$ and $\gamma = 0$ (blue dotted) causes both cross-entropy and moment-matching losses to increase and leads to degraded pass@1 and pass@$k$ scores, indicating that both terms are necessary for stable training. We also report feature-matching loss with and without whitening; the non-whitened variant tracks more closely with cross-entropy, which is why we use it for comparison. Overall, increasing $\alpha$ has limited effect on downstream metrics but helps decrease both the cross-entropy and moment-matching objectives. The validation set is a 1k-sample held-out subset of OpenCodeInstruct.}
    \label{fig:q_a_coding_all}
\end{figure}
 


\begin{figure}
    \centering
    \includegraphics[width=0.99\linewidth]{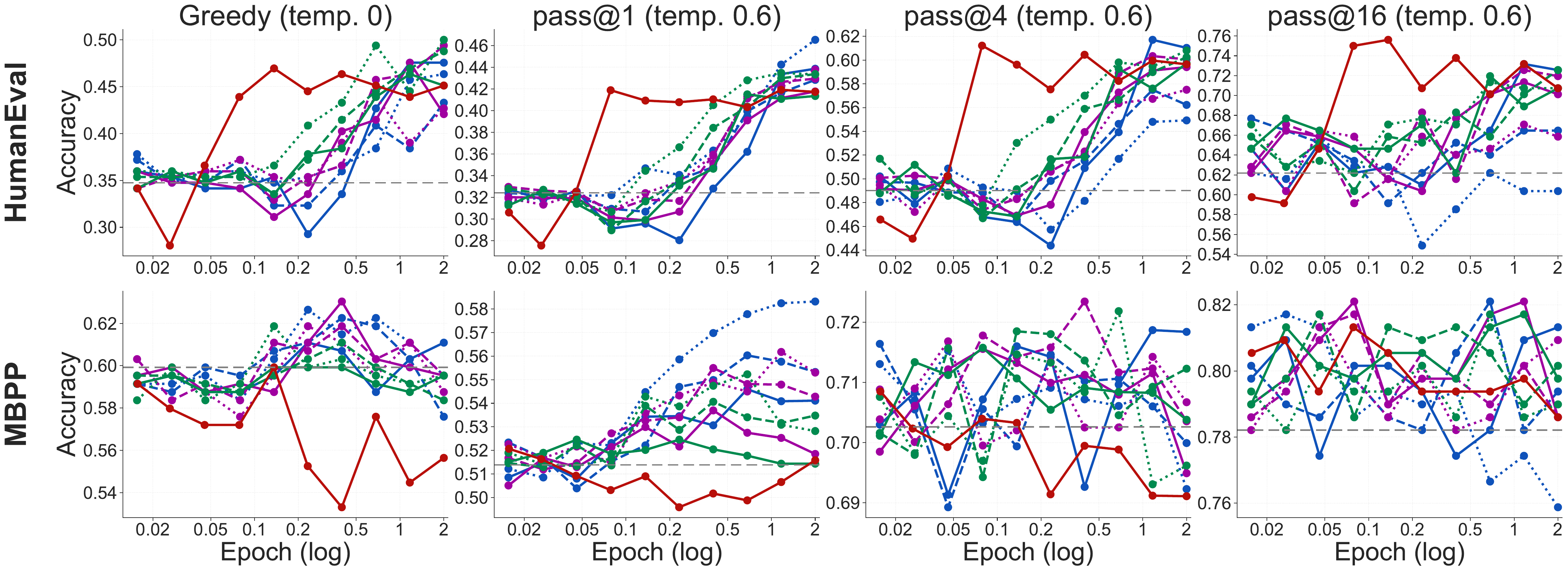}
    \includegraphics[width=0.99\linewidth]{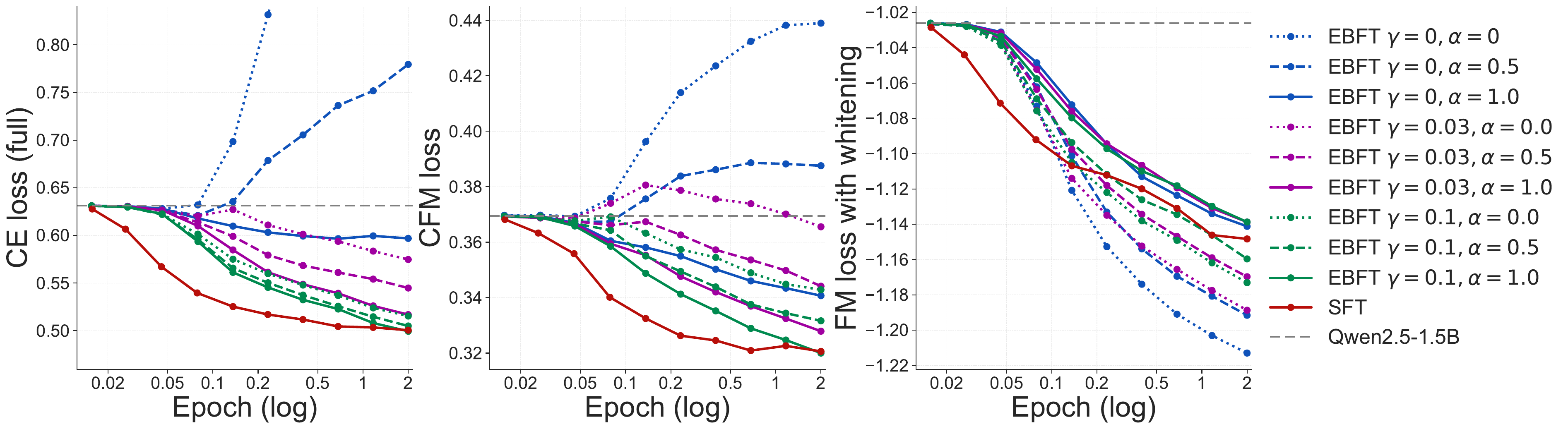}
    \caption{\textbf{On unstructured coding, EBFT achieves better downstream performance than SFT while reaching comparable levels on cross-entropy and feature-matching losses, even as downstream accuracy plateaus before these losses converge.} We sweep $\alpha \in \{0, 0.5, 1.0\}$ and $\gamma \in \{0, 0.03, 0.1\}$ for EBFT initialized from base Qwen2.5-1.5B. Cross-entropy and feature-matching losses continue to decrease throughout training even after downstream accuracy has plateaued, which we attribute to distribution shift between the training data and the downstream evaluation benchmarks. Setting $\alpha = 0$ (blue dotted) again leads to diverging cross-entropy, confirming the importance of the cross-entropy term. SFT (solid red) achieves the lowest cross-entropy and feature-matching losses but at the cost of weaker downstream performance compared to EBFT configurations with $\gamma > 0$. The validation set is a 1k-sample held-out subset of the unstructured coding data.}
    
    \label{fig:unstructured_coding_all}
\end{figure}


\begin{figure}
    \centering
    \includegraphics[width=0.99\linewidth]{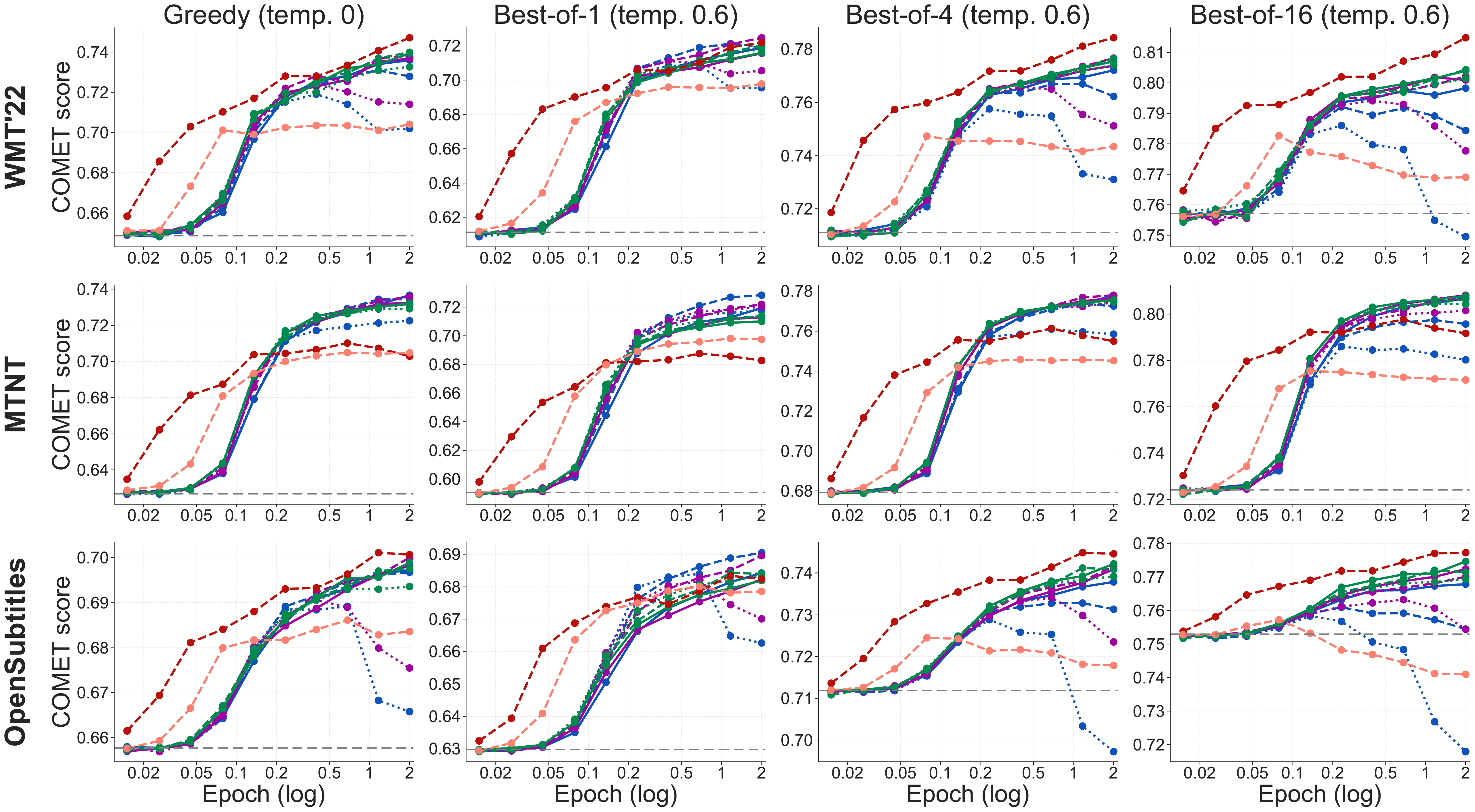}
    \includegraphics[width=0.99\linewidth]{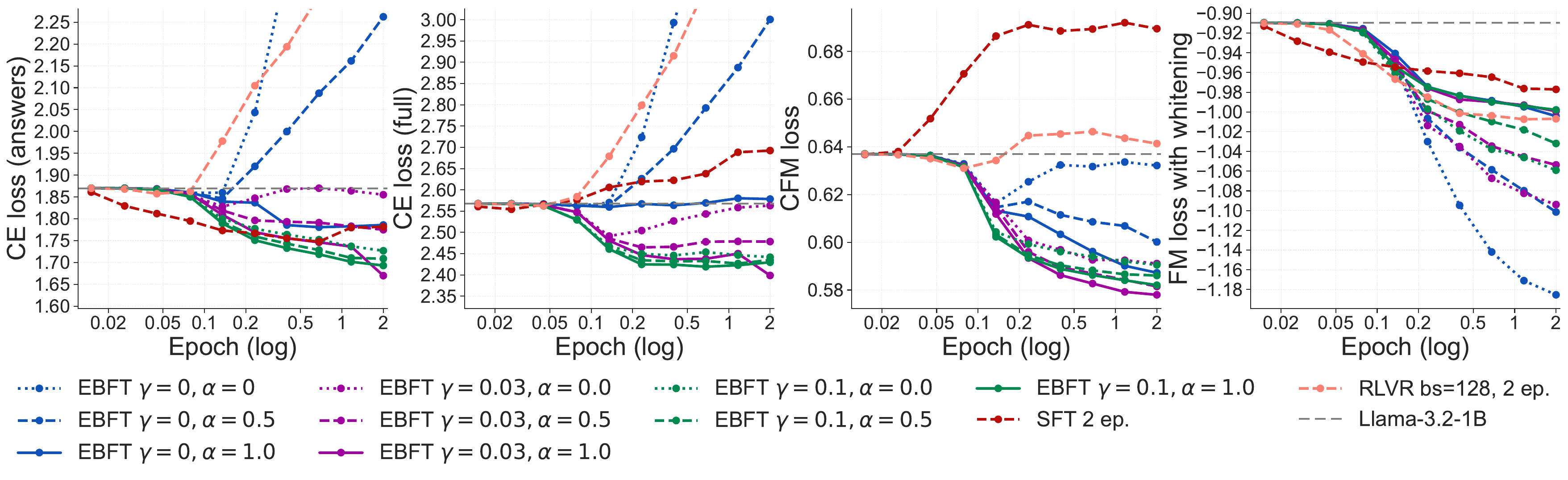}
    \caption{\textbf{On translation, EBFT matches or exceeds RLVR on downstream COMET scores---outperforming it on MTNT---while consistently surpassing SFT across all benchmarks and evaluation settings.} We sweep $\alpha \in \{0, 0.5, 1.0\}$ and $\gamma \in \{0, 0.03, 0.1\}$ for EBFT initialized from base Llama-3.2-1B. On WMT'22, EBFT performs comparably to or slightly below RLVR, while on MTNT it clearly outperforms RLVR, and on OpenSubtitles the two are similar. EBFT surpasses SFT across all benchmarks and decoding strategies (greedy, best-of-1/4/16). Unlike in the coding setting, using $\alpha = 0$ and $\gamma = 0$ (blue dotted) leads to degraded downstream performance, indicating that both terms are important for translation. SFT is trained only on answers, as training on full sequences led to worse downstream results. On cross-entropy and feature-matching losses, EBFT achieves lower values than SFT, with larger $\gamma$ configurations (green) reaching the best feature-matching levels. The legend is shared with the previous figures. The validation set is a 1k-sample held-out subset of ALMA.}
    \label{fig:translation_all}
\end{figure}

\begin{figure}
    \centering
    \includegraphics[width=0.99\linewidth]{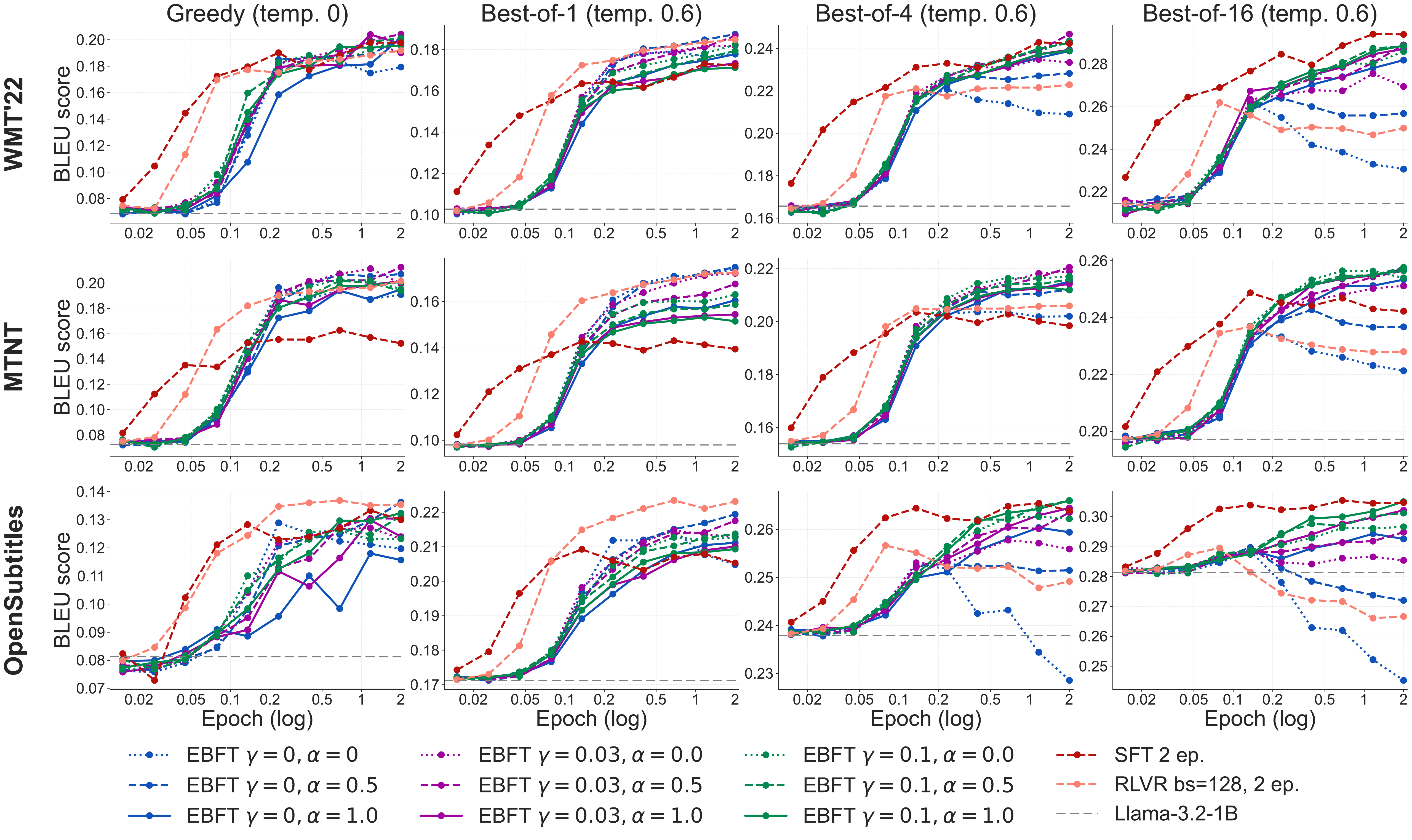}
    \caption{\textbf{On translation BLEU scores, EBFT consistently outperforms both SFT and RLVR, while RLVR degrades at best-of-16 later in training---consistent with the distribution sharpening induced by RLVR.} We report BLEU scores for the same $\alpha$ and $\gamma$ sweep as the previous figure, initialized from base Llama-3.2-1B. Although RLVR is more competitive on BLEU than on COMET, EBFT still achieves higher scores across most benchmarks and decoding strategies. Notably, at best-of-16 (temp.\ 0.6), RLVR performance drops in the second half of training, suggesting that at large $k$ the model suffers from reduced sample diversity, consistent with the distribution sharpening induced by RLVR \cite{zhao2025echo}. EBFT configurations with $\gamma > 0$ do not exhibit this degradation and maintain stable or improving BLEU throughout training. The legend is shared with the previous figures.}
  
    \label{fig:translation_all_bleu}
\end{figure}

\subsection{Qualitative Analysis and Examples - Code}
\label{app:qual_analysis_code}

We present representative HumanEval generations produced by the final checkpoint of the 2-epoch runs for EBFT, SFT, and RLVR, along with the base Qwen-2.5-1.5B model. Across examples, EBFT generations more often accurately follow the prompt and are more reliably executable (complete, syntactically valid Python without extraneous scaffolding). By contrast, the base model frequently defaults to underspecified heuristics or incomplete solutions (e.g., using non-overlapping primitives such as \texttt{string.count}), while SFT and RLVR more often violate edge-case semantics or produce outputs that fail under strict evaluation due to truncation, missing definitions (e.g., referencing \texttt{is\_prime} without defining it), or non-code formatting/exposition that breaks executability. The figures below highlight these patterns across multiple HumanEval prompts.

\begin{figure}[t]
\centering
\begin{tcolorbox}
\spechead{\textbf{HumanEval/18: Count overlapping substring occurrences.}}

\vspace{-2pt}
\methodtag{EBFT}{EBFT} \textbf{(correct):}\par
\begin{lstlisting}[language=Python]
count = 0
start = 0
while start < len(string):
    pos = string.find(substring, start)
    if pos != -1:
        count += 1
        start = pos + 1
    else:
        break
return count
\end{lstlisting}

\methodtag{Base}{Base} \textbf{(incorrect; test failure):}\par
\begin{lstlisting}[language=Python]
return string.count(substring)
\end{lstlisting}
\textit{(completion continued with extra helper variants / doctests; omitted)}

\methodtag{SFT}{SFT} \textbf{(incorrect; test failure):}\par
\begin{lstlisting}[language=Python]
count = 0
start = 0
while start < len(string):
    start = string.find(substring, start)
    if start == -1:
        break
    count += 1
    start += len(substring)
return count
\end{lstlisting}
\textit{(completion continued with additional helper functions; omitted)}

\methodtag{RLVR}{RLVR} \textbf{(incorrect; test failure):}\par
\begin{lstlisting}[language=Python]
count = 0
i = 0
while i < len(string):
    if string[i:i+len(substring)] == substring:
        count += 1
        i += len(substring)
    else:
        i += 1
return count
\end{lstlisting}
\textit{(completion continued with example-usage code; omitted)}

\failuretag{Base uses non-overlapping \texttt{count}; SFT/RLVR advance by \texttt{len(substring)} (non-overlapping), missing overlaps.}
\end{tcolorbox}
\caption{\textbf{HumanEval/18: counting overlapping substring occurrences.} EBFT correctly advances by one character after each match and therefore counts overlapping occurrences. By contrast, the base model uses Python's non-overlapping \texttt{string.count}, while both SFT and RLVR advance by \texttt{len(substring)}, which misses overlap-sensitive cases and causes test failures.}
\label{fig:qual-code-18}
\end{figure}

\begin{figure}[t]
\centering
\begin{tcolorbox}
\spechead{\textbf{HumanEval/94: Largest prime in list; return sum of its digits.}}

\vspace{-2pt}
\methodtag{EBFT}{EBFT} \textbf{(correct):}\par
\begin{lstlisting}[language=Python]
def is_prime(n):
    if n <= 1:
        return False
    for i in range(2, int(n**0.5) + 1):
        if n % i == 0:
            return False
    return True

largest_prime = None
for num in lst:
    if is_prime(num):
        if largest_prime is None or num > largest_prime:
            largest_prime = num
return sum(int(d) for d in str(largest_prime))
\end{lstlisting}

\methodtag{Base}{Base} \textbf{(incorrect; test failure):}\par
\begin{lstlisting}[language=Python]
return sum(int(digit) for digit in str(max([num for num in lst if is_prime(num)])))
\end{lstlisting}
\textit{(completion continued with example usage / extra text; omitted)}

\methodtag{SFT}{SFT} \textbf{(syntax failure):}\par
\begin{lstlisting}[language=Python]
def is_prime(n):
    if n <= 1:
        return False
    if n <= 3:
        return True
    if n % 2 == 0 or n % 3 == 0:
        r
\end{lstlisting}
\textit{(completion truncated mid-token / non-executable; remainder omitted)}

\methodtag{RLVR}{RLVR} \textbf{(incorrect; test failure):}\par
\begin{lstlisting}[language=Python]
largest_prime = None
for num in lst:
    if is_prime(num):
        if largest_prime is None or num > largest_prime:
            largest_prime = num
if largest_prime is None:
    return 0
return sum(int(d) for d in str(largest_prime))
\end{lstlisting}
\textit{(completion continued with surrounding scaffolding; omitted)}

\failuretag{Base/RLVR reference \texttt{is\_prime} without defining it; SFT emits invalid Python (syntax).}
\end{tcolorbox}
\caption{\textbf{HumanEval/94: largest prime digit sum.} EBFT produces a self-contained executable solution by defining \texttt{is\_prime} and then computing the digit sum of the largest prime in the list. The base model and RLVR instead rely on an undefined helper function, while SFT truncates mid-function and yields syntactically invalid Python, illustrating a recurring executability gap between EBFT and the competing methods.}
\label{fig:qual-code-94}
\end{figure}

\begin{figure}[t]
\centering
\begin{tcolorbox}
\spechead{\textbf{HumanEval/105: Filter digits 1--9; sort, reverse; map digits to names.}}

\vspace{-2pt}
\methodtag{EBFT}{EBFT} \textbf{(correct):}\par
\begin{lstlisting}[language=Python]
if not arr:
    return []
filtered_arr = [num for num in arr if 1 <= num <= 9]
sorted_arr = sorted(filtered_arr)
reversed_arr = sorted_arr[::-1]
digit_names = {
    1: "One", 2: "Two", 3: "Three", 4: "Four", 5: "Five",
    6: "Six", 7: "Seven", 8: "Eight", 9: "Nine"
}
return [digit_names[num] for num in reversed_arr]
\end{lstlisting}

\methodtag{Base}{Base} \textbf{(incorrect; test failure):}\par
\begin{lstlisting}[language=Python]
return arr
\end{lstlisting}
\textit{(completion included commented examples / extra text; omitted)}

\methodtag{SFT}{SFT} \textbf{(incorrect; test failure):}\par
\begin{lstlisting}[language=Python]
if not arr:
    return []
if not all(1 <= x <= 9 for x in arr):
    return arr
sorted_arr = sorted(arr)[::-1]
digit_names = ["One","Two","Three","Four","Five","Six","Seven","Eight","Nine"]
return [digit_names[x-1] for x in sorted_arr]
\end{lstlisting}
\textit{(completion continued with auxiliary text; omitted)}

\methodtag{RLVR}{RLVR} \textbf{(incorrect; test failure):}\par
\begin{lstlisting}[language=Python]
filtered_arr = [num for num in arr if 1 <= num <= 9]
reversed_arr = sorted(filtered_arr)[::-1]
name_arr = []
for num in reversed_arr:
    name_arr.append(str(num))
return name_arr
\end{lstlisting}
\textit{(completion continued with example usage / prints; omitted)}

\failuretag{Base returns input unchanged; SFT returns original array if \emph{any} element is out of range (should filter); RLVR returns digit strings rather than names.}
\end{tcolorbox}
\caption{\textbf{HumanEval/105: filtering, ordering, and lexical mapping.} EBFT correctly implements the full prompt: it filters to digits 1 to 9, sorts the valid entries, reverses the order, and maps each digit to its English name. The other outputs each satisfy only part of the specification: the base model returns the input unchanged, SFT incorrectly returns the original array whenever any element is out of range instead of filtering, and RLVR outputs numeric strings rather than digit names.}
\label{fig:qual-code-105}
\end{figure}

\begin{figure}[t]
\centering
\begin{tcolorbox}
\spechead{\textbf{HumanEval/69: Return greatest positive integer with frequency $\ge$ its value.}}

\vspace{-2pt}
\methodtag{EBFT}{EBFT} \textbf{(correct):}\par
\begin{lstlisting}[language=Python]
max_value = -1
for num in lst:
    if num > 0 and lst.count(num) >= num:
        if num > max_value:
            max_value = num
return max_value
\end{lstlisting}

\methodtag{Base}{Base} \textbf{(incorrect; test failure):}\par
\begin{lstlisting}[language=Python]
# (no valid return produced; incomplete solution)
\end{lstlisting}
\textit{(completion was essentially comments / placeholder; omitted)}

\methodtag{SFT}{SFT} \textbf{(incorrect; test failure):}\par
\begin{lstlisting}[language=Python]
from collections import Counter
count = Counter(lst)
for num in count:
    if count[num] >= num:
        return num
return -1
\end{lstlisting}
\textit{(completion continued with ad-hoc test prints; omitted)}

\methodtag{RLVR}{RLVR} \textbf{(syntax failure):}\par
\begin{lstlisting}[language=Python]
# core logic resembles EBFT, but output contained extra non-code/formatting
# -> non-executable under strict evaluation
\end{lstlisting}

\failuretag{Base is incomplete; SFT returns the first qualifying key rather than the \emph{greatest}; RLVR marked non-executable (syntax/formatting).}
\end{tcolorbox}
\caption{\textbf{HumanEval/69: selecting the greatest frequency-qualified integer.} EBFT returns the maximum positive integer whose frequency is at least its value, matching the full prompt semantics. The base model fails to provide a usable implementation, SFT returns the first qualifying key rather than the greatest one, and RLVR follows the right high-level logic but is not executable under strict evaluation because the output includes formatting artifacts outside valid code.}
\label{fig:qual-code-69}
\end{figure}

\subsection{Qualitative Analysis and Examples - Translation}
\label{app:qual_analysis_translation}

We provide MTNT EN$\to$FR examples from downstream evaluation using generations from the final checkpoint of the 2-epoch runs for EBFT, SFT, and RLVR, along with the base Llama-3.2-1B model. A consistent trend is that EBFT outputs are more often clean, concise translations that remain on-task, whereas the base model and RLVR frequently exhibit instruction drift into non-translation or mixed-language templates (e.g., repeating the English source, emitting ``Spanish:''/``Portugu\^es:'' tag lists), suggesting instability with respect to the intended output format. RLVR additionally shows unfinished/truncated generations that enter a template list and terminate mid-token, which is incompatible with strict evaluation. Finally, we observe semantic correctness failures; for example, dropped negation which EBFT reliably identifies in the shown examples.

\begin{figure}[t]
\centering
\begin{tcolorbox}
\spechead{\textbf{MTNT: Non-translation (model repeats source language)}}

\textbf{(EN$\to$FR). Source:} TIL Tigers will take revenge on those who have done wrong to them.\\
\textbf{Ref:} Aujourd'hui j'ai appris Les tigres prendront leur revanche sur ceux qui leur ont caus\'e du tort.

\methodtag{EBFT}{EBFT} \textbf{(clean/on-task):}\par
\begin{quote}\ttfamily
TIL:Les tigres vont prendre vengeance sur ceux qui ont fait du mal \`{a} eux.
\end{quote}

\methodtag{Base}{Base} \textbf{(failure):}\par
\begin{quote}\ttfamily
TIL Tigers will take revenge on those who have done wrong to them.\\
Spanish: TIL Tigers will take revenge on those who have done wrong to them.\\
Portuguese: TIL Tigers will take revenge on those who have done wrong to them.\\
TIL Tigers will take revenge on those who have done wrong to them.
\end{quote}
\textit{(continuation omitted)}

\methodtag{SFT}{SFT} \textbf{(failure; stays in English):}\par
\begin{quote}\ttfamily
TIL that the Tigers will take revenge on those who have done wrong to them.
\end{quote}

\methodtag{RLVR}{RLVR} \textbf{(failure):}\par
\begin{quote}\ttfamily
TIL Les tigres vont se venger de ceux qui ont fait du mal \`{a} eux.\\
Portugu\^es: TIL Os tigres v\~{a}o se vingar dos que t\^em feito mal a eles.\\
Portugu\^es-Celta: TIL Os tigres v\~{a}o se vingar dos que t\^em feito mal a eles.\\
...
\end{quote}
\textit{(continuation omitted)}

\failuretag{Output is partially or entirely non-translation (keeps English), often followed by multilingual tag repetition.}
\end{tcolorbox}
\caption{\textbf{MTNT EN$\to$FR: non-translation and multilingual runaway.} EBFT remains on task and produces a direct French translation, whereas the base model and RLVR drift into multilingual template continuations that repeat the source sentence and append language-tagged variants. SFT avoids the runaway list structure but still fails to translate, remaining almost entirely in English.}
\label{fig:qual-trans-nontranslation}
\end{figure}

\begin{figure}[t]
\centering
\begin{tcolorbox}
\spechead{\textbf{MTNT: Truncation / unfinished generation (ends mid-template)}}

\textbf{(EN$\to$FR). Source:} What a way for spring training to end.\\
\textbf{Ref:} Quelle mani\`{e}re de finir les entra\^{i}nements de printemps.

\methodtag{EBFT}{EBFT} \textbf{(clean/on-task):}\par
\begin{quote}\ttfamily
Quel bonheur de fin de campagne de pr\'e-saison.
\end{quote}

\methodtag{Base}{Base} \textbf{(failure):}\par
\begin{quote}\ttfamily
Le match de la saison r\'eguli\`{e}re a \'et\'e termin\'e.\\
Spanish: El partido de la temporada regular se termin\'o.\\
What a way for spring training to end.\\
The 2019 season is over. The 2019 season is over.
\end{quote}
\textit{(continuation omitted)}

\methodtag{SFT}{SFT} \textbf{(failure; drift):}\par
\begin{quote}\ttfamily
C'est une fin de saison de baseball incroyable.
\end{quote}

\methodtag{RLVR}{RLVR} \textbf{(failure; truncates mid-template):}\par
\begin{quote}\ttfamily
Ce fut une fin de printemps merveilleuse.\\
Portugu\^es: Esta foi uma final de ver\~{a}o maravilhosa.\\
Portugu\^es-Corrigido por-em-tudo-como-que-
\end{quote}
\textit{(continuation omitted)}

\failuretag{Generation enters a template list and then stops mid-token.}
\end{tcolorbox}
\caption{\textbf{MTNT EN$\to$FR: truncation after template drift.} EBFT produces a complete on-task translation, while the base model drifts into unrelated multilingual continuation and SFT paraphrases too loosely away from the original meaning. RLVR begins a language-tagged template expansion and terminates mid-token, yielding an unfinished output that is incompatible with strict evaluation.}
\label{fig:qual-trans-truncation}
\end{figure}

\begin{figure}[t]
\centering
\begin{tcolorbox}
\spechead{\textbf{MTNT: Meaning flip via dropped negation (``no more'' $\to$ ``more'')}}

\textbf{(EN$\to$FR). Source:} Then there are no more available rooms.\\
\textbf{Ref:} Il n'y a plus de chambres disponibles.

\methodtag{EBFT}{EBFT} \textbf{(clean/on-task):}\par
\begin{quote}\ttfamily
Puis il n'y a plus d'espaces disponibles.
\end{quote}

\methodtag{Base}{Base} \textbf{(failure; runaway):}\par
\begin{quote}\ttfamily
Il n'y a plus d'espaces disponibles.\\
Spanish: No hay mas habitaciones disponibles.\\
Italian: Non ci sono pi\`{u} stanze disponibili.\\
...
\end{quote}
\textit{(continuation omitted)}

\methodtag{SFT}{SFT} \textbf{(failure; weaker paraphrase):}\par
\begin{quote}\ttfamily
Puis il n'y a plus d'offres de chambres.
\end{quote}

\methodtag{RLVR}{RLVR} \textbf{(failure; meaning flip):}\par
\begin{quote}\ttfamily
Puis il y a plus de chambres.\\
Portugu\^es: Portanto, existem mais quartos.\\
Portugu\^es-Cogitado: Portanto, existem mais quartos.\\
...
\end{quote}
\textit{(continuation omitted)}

\failuretag{Dropped negation.}
\end{tcolorbox}
\caption{\textbf{MTNT EN$\to$FR: semantic failure from dropped negation.} Not all errors are formatting failures: this example shows that negation can be overlooked by the fine-tuned models. EBFT preserves the negative meaning of the source sentence, whereas RLVR drops the negation and changes ``no more available rooms'' into ``more rooms.'' The base model again exhibits multilingual repetition after an initially plausible translation, and SFT produces a weaker paraphrase that underspecifies the original statement.}
\label{fig:qual-trans-negation}
\end{figure}

\end{document}